%% file: main_arXiv.tex
\documentclass[twoside,11pt]{article}

\usepackage[abbrvbib, preprint]{jmlr2e}
\setcitestyle{square,comma, open={[}, close={]}}
\setcitestyle{numbers}
\usepackage{hyperref}
\usepackage{tabularx}
\usepackage{amsmath, amssymb}
\usepackage{hyperref}
\hypersetup{hidelinks}
\usepackage[ruled,vlined]{algorithm2e}
\usepackage{algorithmic}
\usepackage{mathrsfs}
\usepackage{pgfplots}
\usepackage{enumitem}
\usepackage[skip=2pt,font=scriptsize]{caption}
\usepackage{float}
\usepackage{bbm}
\usepackage{bm}
\usepackage{booktabs}
\usepackage{lipsum}
\usepackage [english]{babel}
\usepackage [autostyle, english = american]{csquotes}
\usepackage{graphicx}
\usepackage{subcaption} % provides the subfigure environment

\newtheorem{prob}{Problem}

\newtheorem{thm}{Theorem}[section]{\bfseries}{\itshape}
\newtheorem{lem}{Lemma}[section]{\bfseries}{\itshape}
\newtheorem{prop}{Proposition}[section]{\bfseries}{\itshape}
\newtheorem{rema}{Remark}[section]{\bfseries}{\itshape}
{\bfseries}{\itshape}
\newtheorem{defi}{Definition}[section]{\bfseries}{\itshape}
{\bfseries}{\itshape}

\DeclareMathOperator*{\argmax}{arg\,max}
\DeclareMathOperator*{\argmin}{arg\,min}

\DeclareMathOperator{\margin}{margin}
\DeclareMathOperator{\KL}{D_{KL}}
\DeclareMathOperator{\JS}{JSD}

\DeclareMathOperator{\TV}{TV}

\renewcommand{\epsilon}{{\varepsilon}}

\newcolumntype{b}{X}
\newcolumntype{s}{>{\hsize=.5\hsize}X}
\newcolumntype{t}{>{\hsize=.3\hsize}X}

\usepackage{xcolor}
\usepackage[normalem]{ulem}

\jmlrheading{}{}{}{}{}{}{Shizhou Xu}

\ShortHeadings{Machine Unlearning}{Xu, Strohmer}
\firstpageno{1}

\begin{document}

\title{Machine Unlearning via Information Theoretic Regularization}

\author{\name Shizhou Xu \email shzxu@ucdavis.edu \\
       \addr Department of Mathematics\\
       University of California Davis\\
       Davis, CA 95616-5270, USA
       \AND
       \name Thomas Strohmer \email strohmer@math.ucdavis.edu \\
       \addr Department of Mathematics\\
       Center of Data Science and Artificial Intelligence Research\\
       University of California Davis\\
       Davis, CA 95616-5270, USA}

\maketitle

\begin{abstract}%   <- trailing '%' for backward compatibility of .sty file
How can we effectively remove or “unlearn” undesirable information, such as specific features or the influence of individual data points, from a learning outcome while minimizing utility loss and ensuring rigorous guarantees? We introduce a unified mathematical framework based on information-theoretic regularization to address both data-point unlearning and feature unlearning. For data-point unlearning, we introduce the \emph{Marginal Unlearning Principle}, an auditable and provable framework. Moreover, we provide an information-theoretic unlearning definition based on the proposed principle and provable guarantees on sufficiency and necessity of marginal unlearning. We then show that the proposed framework provides a natural solution to the marginal unlearning problem and yields auditable high-probability marginal-unlearning guarantees. For feature unlearning, the framework applies to deep learning with flexible training objectives. By combining flexibility in learning objectives with simplicity in regularization design, our approach is highly adaptable and practical for a wide range of machine learning and AI applications. From a mathematical perspective, we provide a unified analytic solution to the optimal feature unlearning problem with a variety of information-theoretic training objectives. Our theoretical analysis reveals intriguing connections between machine unlearning, information theory, optimal transport, and extremal sigma algebras. Numerical simulations support our theoretical findings.
\end{abstract}


\vfill
\newpage

\tableofcontents
\newpage

\section{Introduction}\label{S:Intro}
\input{Section_1}

\section{Unlearning Principles, Definitions, \& Justifications}\label{S:Definition}
\input{Section_2}

\section{Information Theoretic Unlearning Framework}\label{S:Into_Framework}
\input{Section_3}

\section{Theoretical Guarantees}\label{S:Theory}
\input{Section_4}

\section{Algorithm Design}\label{S:Algorithm}
\input{Section_5}

\section{Numerical Experiments}\label{S:Experiment}
\input{Section_6}

\section*{Conclusion}
\input{Conclusion}

\section*{Acknowledgement}

The authors acknowledge support from NSF DMS-2208356,  NIH R01HL16351,  
P41EB032840, and DE-SC0023490.

\newpage
%\clearpage

\iffalse
\section*{Impact Statement}
 We have introduced a flexible mathematical framework for unlearning undesirable information in machine learning.  As our approach combines rigorous theoretical guarantees with practical solutions,  we expect our work to have significant impact on both theory and practice of machine unlearning. Since machine unlearning plays a crucial ethical role in AI by enabling data privacy, fairness, accuracy, and responsible development, the broader impact of our paper is that it contributes positively to building trustworthy and ethical AI. We do not see any negative ethical or societal impact of our work.
\fi
 
\clearpage
\bibliography{References_Unlearning}

\newpage
\appendix
\onecolumn

\section{Appendix: Supplementary material for Section 1}\label{A:Categorty}
\input{Appendix_1}

\section{Appendix: Supplementary material for Section 2}\label{A:Section2}
\input{Appendix_2}

\section{Appendix: Supplementary material for Section 3}\label{A:Section3}
\input{Appendix_3}

\section{Appendix: Supplementary material for Section 4}\label{A:Section4}
\input{Appendix_4}

\end{document}

%% file: Section_1.tex
As machine learning systems are deployed in sensitive and scientific domains, it becomes essential to remove the influence of designated attributes or individual records from \emph{trained} models while preserving utility and providing rigorous guarantees. Naively deleting attributes or data points from the raw data, or masking them in model outputs, rarely suffices: their effects often persist through correlated proxies, latent representations, and model parameters. Fully retraining on the corrected training data set becomes impractical at scale, especially considering the increasing size of AI models and the potential sequence of unlearning requests. \emph{Machine unlearning} offers a principled alternative, enabling models to ``forget'' specified information for privacy, fairness, legal compliance, or scientific correction. The central challenge is to design provable, auditable procedures that achieve the desired removal with minimal utility loss, while avoiding unnecessary end-to-end retraining when possible.

In this work, we focus on two complementary problem settings, motivated by distinct practical needs:

\begin{itemize}[leftmargin=*, itemsep=2pt]
\item \textbf{Feature unlearning (removing attribute influence).}
Redacting explicit tokens rarely suffices when proxies remain. For example, a recruiting model favored male applicants because the training data were predominantly male; removing explicit gender markers failed due to correlated features (e.g., all-women colleges)~\cite{dastin2022amazon,christianalignment}. Feature unlearning aims to neutralize the \emph{influence} of a designated attribute \(Z\) on what the model releases, denoted generically by \(S\), even through proxies.

\item \textbf{Marginal data-point unlearning (removing a record’s footprint).} Under GDPR’s “right to be forgotten’’~\cite{GDPR2016a} and California’s CCPA~\cite{bonta2022california}, individuals may request deletion of their data. Naively deleting records and fine-tuning can leave a detectable footprint in the parameters. Retraining from scratch eliminates a record's \emph{unique contribution} while preserving performance attributable to retained data~\cite{bourtoule2021machine}, but is costly. Marginal data-point unlearning seeks efficient, auditable post-hoc procedures that approximate the retrain-on-retain behavior while maintaining utility.
\end{itemize}

Although originally motivated by privacy and fairness, machine unlearning now plays a broader role in various domains, such as AI-assisted scientific discovery and autonomous experimentation. In these settings, unlearning is used to remove the influence of corrupted or drifting measurements (e.g., broken sensors), instrument-specific artifacts, and other unwanted dependencies (e.g., sensor or batch effects), thereby correcting models and pipelines without costly retraining and enabling reliable, on-the-fly scientific discovery.

For intuition, it is helpful (but not necessary) to think of $X$ as a data matrix (e.g., an EHR table or consumer database), where rows are users and columns are attributes. Feature unlearning removes the information or influence of the columns $X_{[:,j\in z]}$\footnote{If the feature $Z$ to be unlearned is not explicitly present in $X$, one can still remove information in $X$ that is \emph{correlated} with $Z$ by working with the joint pair $(X,Z)$ as the effective input.} corresponding to $Z$, whereas data-point unlearning removes the information contributed by rows $X_u$ to the remainder $X_{r}:=X_{[i\notin u,:]}$.

\noindent\textbf{Problem statement.}
We introduce a generic information-theoretic framework that covers both unlearning settings:
\begin{quote}
    \emph{Given data \(X\), an unlearning variable \(Z\), and a target variable \(Y\), how can we construct a released outcome \(S\) that exposes minimal information about \(Z\) while preserving as much task-relevant information about \(Y\) as possible?}
\end{quote}
Here, \(S\) may live in the data space, an intermediate representation space, or the output space. In feature unlearning, \(Z\) denotes the feature, source, or intervention label to be neutralized, and the released object takes the form \(S=f(X,Z)\). In marginal data-point unlearning, \(Z\) denotes the latent
source label distinguishing the retain distribution from the retain-plus-unlearn mixture, and the released object takes the form \(S=f(X)\). Section~\ref{S:Definition} formalizes these two cases.

\subsection{Paper organization and contribution}\label{s1:contribution}

This paper makes the following contributions:

\begin{itemize}[itemsep=2pt, topsep=2pt, leftmargin=*]

\item \textbf{Marginal unlearning principle.}
We introduce the \emph{Marginal Unlearning Principle} and a corresponding formal unlearning definition (Definition~\ref{def:epsilon_MU}) in Section \ref{s2:data_unlearning}. Inspired by ``blur + reinforce'' memory suppression mechanisms in cognitive neuroscience, this principle offers an inference-based, post-hoc guarantee for data point unlearning. We show that the proposed principle and the resulting definition guarantee (i) \emph{practically checkable}, since it depends only on the released model and the original dataset in Section \ref{s4:data_guarantee}, (ii) \emph{sufficient} for a ``good-utility'' model retrained from scratch on the retain data set, under utility requirements in Section \ref{s2:sufficiency}, and (iii) \emph{necessary} for a ``robust'' model retrained from scratch on the retain dataset unless the retrained model violates smoothness/Lipschitz regularity, which would harm generalization (Lemma~\ref{l:DU_as_condition_for_good_model}) in Section \ref{s2:necessity}.

\item \textbf{Unified information-theoretic framework and practical algorithms.}
We formulate unlearning as an information-theoretic trade-off, analogous to a rate--distortion objective (Section~\ref{s3:problem_setting}):
\begin{equation}\label{eq:heuristic_objective}
    \inf_{f} \quad
    (1-\lambda)\underbrace{\mathcal{C}(Y;S,Z)}_{\text{utility: memory reinforcement}}
    \;+\;
    \lambda\,\underbrace{I(S';Z)}_{\text{unlearning: memory blur}},
    \qquad \lambda\in[0,1],
\end{equation}
where \(S\) denotes the released learning outcome and \(S'\) denotes the
auditing object on which unlearning is certified. The term \(I(S';Z)\)
quantifies the information about the unlearning variable \(Z\) exposed by the
audited object. In feature unlearning, \(Z\) is an observed attribute, source,
or intervention label, and typically \(S'=S=f(X,Z)\). The utility
\(\mathcal{C}(Y;S,Z)\) measures task-relevant information preserved
within each \(Z\)-slice, while the exposure penalty \(I(S;Z)\) controls what
\(S\) reveals about \(Z\). In marginal data-point unlearning, \(Z\) is instead
a latent source label distinguishing the retain-only reference from the
retain-plus-unlearn mixture. It is used only for indistinguishability
certification, and utility is evaluated from the released model output alone. In this case, the utility term specializes to a non-sliced cost \(\mathcal{C}(Y;S)\).
We provide practical regularization algorithms for both settings
(Algorithms~\ref{alg:feature_reg} and~\ref{alg:point_reg},
Section~\ref{S:Algorithm}) and evaluate their empirical behavior on tabular
and image data.

\item \textbf{Analytic solution and theoretical guarantees.}
We first show that bounding the ``compression rate'' \(I(S';Z)\) directly
implies high-probability guarantees for feature unlearning
(Theorem~\ref{thm:mi_tv_general}) and marginal data-point unlearning
(Lemma~\ref{l:epsilon_DU_bound_mutual_info},
Theorem~\ref{thm:DU-Guarantee-via-Compression-Rate}). This provides
actionable guidance for tuning \(\lambda\) to meet a target leakage level while
preserving learning utility. Furthermore, in the hard-independence limit
\(S'\perp Z\), we identify an analytic \(\mathcal W_2\)-barycenter solution
for feature unlearning in the data-space setting under a class of utility objectives. The barycentric representation generates a finest lifted
sigma-algebra among admissible outcomes
(Lemma~\ref{l:Finest_Sigma_Algebra}), yielding maximal lifted utility while
enforcing \(S'\!\perp\! Z\) (Theorem~\ref{th:optimal_solution}). We also
provide a practical implementation of the analytic solution as
Algorithm~\ref{alg:bary} in Section~\ref{s5:analytic_feature_algorithm}.

\item \textbf{Exemplary numerical experiments.}
We evaluate both feature unlearning and data-point unlearning across tabular and image modalities. Section~\ref{s6:data_unlearning_experiment} reports data-unlearning results on tabular data and MNIST; Sections~\ref{s6:feature_tabular_experiment} and~\ref{s6:analytic_feature_experiment} report feature-unlearning results on tabular benchmarks and CelebA, respectively.

\item \textbf{Intriguing mathematical connections.}
Our proposed framework uncovers compelling connections between machine unlearning, information theory, optimal transport, and extremal sigma algebras.
\end{itemize}

\paragraph{Organization.}
The remainder of the paper is organized as follows.
Section~\ref{S:Intro} reviews related work and establishes notation.
Section~\ref{S:Definition} formulates the principles for feature and
data-point unlearning and provides definition-related theoretical guarantees.
Section~\ref{S:Into_Framework} presents the unified information-theoretic
formulation that underpins both settings. Section~\ref{S:Theory} develops the
main framework-related theoretical guarantees and the analytic barycenter
solution for optimal feature unlearning. Section~\ref{S:Experiment} provides
empirical evidence across different data modalities and datasets.

\subsection{Related work}\label{s1:related_work}

Research on unlearning comprises two main threads: \emph{data unlearning} (removing the influence of specific training points) and \emph{feature unlearning} (removing the influence of sensitive attributes). We briefly connect to privacy and fairness where relevant, but the focus of this paper is unlearning.

\paragraph{Data unlearning (machine unlearning).}
The canonical \emph{exact} unlearning standard requires the released model to match the distribution of a model retrained from scratch on the retain data set (the ``anchor'' or ``retain set''). Systems-oriented approaches realize this exactly via sharded pipelines or statistical caching~\citep{cao2015unlearning,bourtoule2021machine,ginart2019making}, though often at significant infrastructure cost. To avoid full retraining, recent literature adopts \emph{approximate} guarantees, relaxing the requirement to distributional proximity (e.g., in total variation or KL-divergence) between the unlearned model and the retrained anchor. A diverse array of algorithmic solutions has emerged to meet this criterion, including influence-based updates~\citep{guo2019certified,sekhari2021remember}, performance-preserving objectives like PUMA~\citep{wu2022puma}, and parameter scrubbing~\citep{golatkar2020eternal}. While these methods offer novel optimization perspectives, framing unlearning as error minimization, targeted repair, or information erasure, they ultimately operate under the shadow of the anchor-based definition. Consequently, verifying their guarantees requires comparing the unlearned model against a stochastic, \emph{unknown} retrained model. This dependency renders these methods difficult to audit or certify~\citep{thudi2022auditable}.

In contrast, our \emph{marginal unlearning} framework breaks this dependency. We propose an inference-based criterion that is checkable solely via observable outputs and the original dataset, providing a practical path to auditable unlearning while maintaining rigorous population-level guarantees. In parallel, a companion work Forgetting-MarI~\cite{xu2025forgetting} applies the marginal unlearning principle to LLMs (Large Language Models), providing large-scale empirical support. This work focuses on providing the mathematical foundations of the principle, such as the sufficiency and necessity guarantees.

\paragraph{Feature unlearning.}
Feature unlearning removes a sensitive attribute’s information from predictions or representations, via adversarial or information-theoretic training~\citep{warnecke2021machine,han2024unlearning}. This connects to outcome-independence in fair ML, such as \emph{statistical parity}~\citep{dwork2012fairness}. Optimal-transport formulations characterize (under regularity) optimally fair predictors/representations via Wasserstein barycenters~\citep{chzhen2020fair,gouic2020projection,xu2023fair}. In computer vision, related \emph{concept removal} methods aim to erase targeted concepts from generative or recognition models~\citep{gandikota2023erasing}.

\paragraph{Information-theoretic formulations.}
Classical information-bottleneck-style compression and rate-distortion frameworks target \emph{all} information about a point or attribute (e.g., minimizing $I(X;\hat Y)$)~\citep{wang2023machine,han2024unlearning}. Our formulation instead targets the \emph{marginal} effect of adding/removing the point via an information-theoretic regularizer on an auxiliary pair $(S',Z)$ defined in Section \ref{s3:problem_setting}. Because our probabilistic guarantee (Definition~\ref{def:epsilon_MU}) is inference-based and does not rely on an unknown retraining oracle, it yields practical, auditable conditions while remaining compatible with diverse utilities.

\paragraph{Data \& AI privacy}
Data-point unlearning is a post hoc complement to preventative privacy techniques such as anonymization and \emph{differential privacy} (DP)~\citep{dwork2006differential,dwork2014dpbook}. Privacy protection techniques, such as anonymization (e.g., $k$-anonymity) and DP, offer training-time protection by bounding the effect of any single record on learned parameters (e.g., DP-SGD)~\citep{abadi2016dpsgd} or learned outcomes. In contrast, unlearning provides \emph{remediation} after training, when sensitive data were already incorporated or when tail-event leakages are discovered via auditing (e.g., membership inference and data extraction~\citep{shokri2017membership,yeom2018mia,carlini2021extracting}). Thus, unlearning \emph{complements} DP by enabling targeted removal after deployment, rather than replacing it.

\paragraph{ML \& Algorithmic fairness.}
Feature unlearning aligns with \emph{statistical parity} (outcome independence from a protected attribute)~\citep{dwork2012fairness}. Related notions include \emph{equalized odds} (conditional independence given the true label)~\citep{hardt2016equality} and counterfactual fairness, but here we focus on \emph{unconditional} feature unlearning; conditional variants are left to future work.

\paragraph{Fine-tuning \& robustness}
Unlearning complements fine-tuning in improving model reliability and longevity: fine-tuning \emph{adds/adapts} knowledge (full or parameter-efficient updates such as adapters, prompt tuning, LoRA/QLoRA)~\citep{houlsby2019adapter,pfeiffer2020adapter,lester2021prompt,hu2022lora,dettmers2023qlora}, while unlearning \emph{removes} misinformation, outdated facts, or undesired influences without full retraining. In practice, combining fine-tuning with unlearning allows AI models to stay current and responsible without rebuilding from scratch. This perspective is complementary (orthogonal) to continual-learning methods that mitigate \emph{catastrophic forgetting} during \emph{addition} of tasks~\citep{kirkpatrick2017ewc}. Here, the goal is principled \emph{removal}.

Moreover, the proposed marginal-information idea can guide robust fine-tuning: update only on the \emph{unique} signal in new data (i.e., the information contributed beyond what is already captured by the retain set). By focusing on this marginal information, the fine-tuned model is less prone to over-training and catastrophic interference with existing knowledge. Conceptually, our marginal approach parallels the residual learning in ResNet \citep{he2016deep}: just as residuals allow a model to learn only ``what is new,'' rather than destabilizing what has already been learned.

\paragraph{Memory suppression studies in cognitive psychology and neuroscience.}
The proposed marginal-information unlearning principle aligns with established mechanisms in cognitive psychology and neuroscience that implement a two-phase “blur + reinforce’’ process to reduce the diagnostic influence of specific memories while preserving alternatives. In the \emph{Think/No-Think} paradigm, repeated suppression of cue-elicited retrieval engages inhibitory control and reliably diminishes later accessibility, effectively lowering the discriminability of responses that would otherwise reveal the suppressed trace \citep{AndersonGreen2001,AndersonHanslmayr2014}. \emph{Directed forgetting} (list/item methods) achieves a comparable reduction via context change and selective rehearsal, making earlier material less accessible at test and thus functionally “blurring’’ its contribution to observed behavior \citep{SahakyanKelley2002}. \emph{Reconsolidation-update} procedures (retrieval–extinction) transiently destabilize reactivated memories and incorporate non-reinforcing information, attenuating the unique predictive impact of the original association \citep{Schiller2010,NaderSchafeLeDoux2000,Xue2012}. Complementarily, \emph{retrieval practice} strengthens desired knowledge and improves later performance on targeted content \citep{KarpickeBlunt2011}. Viewed through our lens, these blur mechanisms reduce the ``membership signal'', analogous to lowering $I(S';Z)$, while reinforcement preserves task utility, providing intuitive and empirical support for our auditable, information-theoretic formulation of unlearning.

\subsection{Tools and notation}\label{s1:notation}

We collect the basic notation used throughout the paper.

\paragraph{Probability.}
Let \((\Omega,\mathcal{F},\mathbb{P})\) be a probability space and \((\mathcal{X},\mathcal{B}_{\mathcal{X}})\) a measurable state space.  
A (measurable) random variable \(X:\Omega\to\mathcal{X}\) has \emph{law} (distribution) \(\mathcal{L}(X):=\mathbb{P}\circ X^{-1}\), i.e., \(\mathcal{L}(X)(A)=\mathbb{P}(X\in A)\) for all \(A\in\mathcal{B}_{\mathcal{X}}\).  
For random variables \(X\) and \(Y\) taking values in \(\mathcal{X}\) and \(\mathcal{Y}\), we write \(X\perp Y\) if
\[
\mathbb{P}(X\in A,\,Y\in B)=\mathbb{P}(X\in A)\,\mathbb{P}(Y\in B)\quad\text{for all }A\in\mathcal{B}_{\mathcal{X}},\ B\in\mathcal{B}_{\mathcal{Y}}.
\]
For information-theoretic quantities, \(H(\cdot)\) denotes (Shannon) entropy, \(I(\cdot;\cdot)\) mutual information (MI), and \(D_{\mathrm{KL}}(\cdot\Vert\cdot)\) Kullback--Leibler (KL) divergence.  
If \(X\) is discrete with pmf \(p_X\), then \(H(X)=-\sum_x p_X(x)\log p_X(x)\); for absolutely continuous \(X\) with pdf \(p_X\) (w.r.t.\ a reference measure \(\mu\)), the differential entropy is \(h(X)=-\int p_X(x)\log p_X(x)\,{\rm d}\mu(x)\).  
For probability measures \(P,Q\) on \((\mathcal{X},\mathcal{B}_{\mathcal{X}})\),
\[
D_{\mathrm{KL}}(P\Vert Q)=
\begin{cases}
\displaystyle \int \log\!\Big(\frac{{\rm d}P}{{\rm d}Q}\Big)\,{\rm d}P, & P\ll Q,\\[1ex]
+\infty, & \text{otherwise}.
\end{cases}
\]
Given a pair \((X,Y)\) with joint law \(\mathbb{P}_{XY}\) and marginals \(\mathbb{P}_X,\mathbb{P}_Y\),
\[
I(X;Y)=D_{\mathrm{KL}}\!\big(\mathbb{P}_{XY}\,\Vert\,\mathbb{P}_X\otimes\mathbb{P}_Y\big),
\]
and, in the discrete case, \(I(X;Y)=H(X)+H(Y)-H(X,Y)\).  
Unless stated otherwise, logarithms are natural, so information quantities are measured in nats. The total variation distance is defined as
\begin{equation*}
    \mathrm{TV}(P,Q):=\sup_{A\in\mathcal B_{\mathcal X}} |P(A)-Q(A)| = \tfrac12 \int |{\rm d}P-{\rm d}Q|.
\end{equation*}

\paragraph{Optimal transport.}
Let \((\mathcal{X},d_{\mathcal{X}})\) be a Polish metric space and \(\mathcal{P}(\mathcal{X})\) the set of Borel probability measures on \(\mathcal{X}\).  
For \(\mu,\nu\in\mathcal{P}(\mathcal{X})\), their \(p\)-Wasserstein distance is~\cite{villani2021topics}
\[
\mathcal{W}_{d_{\mathcal{X}},p}(\mu,\nu)
:=\Bigg(\inf_{\lambda\in\Pi(\mu,\nu)}\ \int_{\mathcal{X}\times\mathcal{X}} d_{\mathcal{X}}(x,x')^{p}\,{\rm d}\lambda(x,x')\Bigg)^{\!1/p},
\]
where \(\Pi(\mu,\nu)\) is the set of couplings with \(\lambda(\cdot,\mathcal{X})=\mu\) and \(\lambda(\mathcal{X},\cdot)=\nu\).  
We write \(\mathcal P_2(\mathcal X)\) for Borel probability measures with
finite second moment, and \(\mathcal{P}_{2,ac}(\mathcal{X})\subset
\mathcal P_2(\mathcal X)\) for measures that are absolutely continuous
(w.r.t.\ Lebesgue when \(\mathcal{X}\subset\mathbb{R}^d\)).
For random variables \(X_1,X_2\) taking values in \(\mathcal{X}\), we use the shorthand
\(\mathcal{W}_{d_{\mathcal{X}},p}(X_1,X_2):=\mathcal{W}_{d_{\mathcal{X}},p}(\mathcal{L}(X_1),\mathcal{L}(X_2))\).  
When the metric is clear, we write \(\mathcal{W}_p\); in particular, \(\mathcal{W}_2:=\mathcal{W}_{d_{\mathcal{X}},2}\).

\paragraph{Wasserstein barycenters.}
Given a family \(\{\mu_z\}_{z\in\mathcal{Z}}\subset\mathcal{P}_{2,ac}(\mathcal{X})\) and weights \(\lambda\in\mathcal{P}(\mathcal{Z})\), a (squared) \(\mathcal{W}_2\) barycenter is any solution~\cite{villani2021topics}
\[
\bar{\mu}\ \in\ \argmin_{\mu\in\mathcal{P}_2(\mathcal{X})}\ \int_{\mathcal{Z}} \mathcal{W}_2^2(\mu_z,\mu)\,{\rm d}\lambda(z).
\]
When Brenier maps \(T_z:\mathcal{X}\to\mathcal{X}\) from \(\mu_z\) to \(\bar{\mu}\) exist (e.g., under absolute continuity and quadratic cost), one may represent a barycentric random variable \(\bar{X}\sim\bar{\mu}\) as \(\bar{X}=T_z(X_z)\) with \(X_z\sim\mu_z\).  
Additional details and sufficient conditions (existence/uniqueness, regularity) are deferred to Appendix~\ref{a:OT_intro}.
%\newpage

%% file: Section_2.tex
In this section, we present the unlearning principles, formal definitions, and provable guarantees for both feature and data unlearning. Section~\ref{s2:feature_unlearning} states the feature-unlearning principle, gives its formal definition, and motivates its information-theoretic regularization.
Section~\ref{s2:data_unlearning} introduces the proposed marginal-unlearning principle for data-point unlearning, together with its formal and information-theoretic definitions.
Section~\ref{s2:sufficiency} establishes \emph{sufficiency}: when the retrain-on-retain model obtains small utility loss, marginal unlearning implies existing approximate unlearning guarantees. Section~\ref{s2:necessity} establishes a complementary \emph{necessity} statement: a well-regularized retrain-on-retain model cannot exhibit detectable inclusion leakage while satisfying marginal unlearning.

For clarity and intuition, we state the unlearning principles, definitions,
and justifications in the canonical prediction setting where the outcome space
coincides with the target space, i.e., \(\mathcal S=\mathcal Y\) and
\(S=\hat Y\). The same principles apply to the generic setting
\(\mathcal S\in\{\mathcal X,\mathcal V,\mathcal Y\}\), where
\(\mathcal X\) is the input data space, as in synthetic data generation or
self-supervised learning with \(\mathcal Y=\mathcal X\); \(\mathcal V\) is a
latent representation space in the representation learning; and \(\mathcal Y\) is the target space in
supervised learning.

\subsection{Feature unlearning}\label{s2:feature_unlearning}

Feature unlearning aims to erase the influence of a designated \emph{feature} \(Z\) (e.g., a protected attribute, a spurious concept, or a source/domain indicator) from what the model \emph{exposes}, while preserving task-relevant information on the retain distribution \(p^r\). In this setting, \(Z\) is observed and available to the unlearning procedure or auditor, but the released output should not reveal \(Z\). Because exposure occurs through the released predictor's output, a natural target is the \emph{output law} of \(\hat Y=f_\theta(X,Z)\). This suggests the following
{\bf feature unlearning principle:}

\begin{quote}
{\em No observer should be able to infer anything about $Z$ from the model’s outputs beyond prior knowledge.}
\end{quote}

We now formalize this principle. Let $(X,Y,Z)\sim p^r$ on $\mathcal{X}\times\mathcal{Y}\times\mathcal{Z}$, with released output \(\hat Y=f_\theta(X,Z)\in\hat{\mathcal{Y}}\), where $f_{\theta}$ is the ML/AI model parameterized by model parameter $\theta$. Define the lifted retain loss
\begin{equation*}
    L(f_\theta):=\mathcal{C}(Y;\hat Y,Z),
    \qquad
    L^\star:=\inf_{\theta} L(f_\theta).
\end{equation*}
We quantify exposure of $Z$ via the mutual information \(I(\hat Y;Z)\) under~$p^r$.

Here, we evaluate utility
within the observed \(Z\)-slices through losses of the form
\[
    \mathcal C(Y;\hat Y,Z)
    =
    \mathbb E_Z\!\left[\mathcal C_z(Y;\hat Y)\right].
\]
This lifted loss isolates the utility preservation from the independence constraint to avoid unnecessary conflict between the two, by guiding how the released output \(\hat Y\) is deformed within each \(Z\)-slice, while the exposure penalty \(I(\hat Y;Z)\) controls how much \(\hat Y\) reveals about \(Z\) after the slices are mixed. For information-based utilities, this interpretation becomes exact:
For a general target \(Y\), the non-lifted information utility satisfies
\[
    I(Y;\hat Y)
    =
    I(Y;\hat Y\mid Z)
    +
    I(Z;\hat Y)
    -
    I(Z;\hat Y\mid Y).
\]
Thus, when \(Y\) is statistically dependent on \(Z\), the non-lifted utility can partially reward the same \(Z\)-dependence that the exposure penalty \(I(\hat Y;Z)\) is designed to remove. More specifically, using equation \eqref{eq:heuristic_objective}, we have
\begin{equation*}
    -(1-\lambda)I(Y;\hat{Y}) + \lambda I(\hat{Y};Z) = (2\lambda - 1) I(\hat{Y};Z) + (1-\lambda) I(Z;\hat Y\mid Y) - (1-\lambda) I(Y;\hat Y\mid Z).
\end{equation*}
The first two terms quantify marginal and conditional dependence involving \(Z\), whereas the third term is the within-slice task information. Therefore, using $I(Y;\hat{Y})$ would result in a mixture of utility and (conditional) dependence. In contrast,
\[
    I(Y;(\hat Y,Z))
    =
    I(Y;Z)+I(Y;\hat Y\mid Z),
\]
and the first term is fixed with respect to the learned representation. Hence, the lifted utility isolates the within-\(Z\)-slice information \(I(Y;\hat Y\mid Z)\), which serves as the genuine utility guidance for deforming \(\hat Y\) under the exposure constraint \(I(\hat Y;Z)\).

Now, we provide a formal definition of feature unlearning.

\begin{defi}[Feature unlearning (independence)]
\label{def:fu-independence}
We say that \(\hat Y\) \emph{unlearns} the feature \(Z\) (with respect to \(p^r\)) if \(\hat Y \perp Z\):
\begin{equation*}
    \mathbb{P}\big(\hat Y\in A,\,Z\in B\big)\;=\;\mathbb{P}(\hat Y\in A)\,\mathbb{P}(Z\in B) \quad \text{for all measurable events } A\in\mathcal{B}_{\hat{\mathcal{Y}}},\ B\in\mathcal{B}_{\mathcal{Z}}.
\end{equation*}
Equivalently, \(I(\hat Y;Z)=0\).
\end{defi}

The exact independence (or zero mutual information) requirement can be relaxed by an information-theoretic tolerance while enforcing utility.

\begin{defi}[Mutual-information feature unlearning with utility budget]
\label{def:feature-unlearning-output}
Given tolerances \((\varepsilon,\delta)\in[0,\infty)^2\), we say that \(f_\theta\) achieves \emph{$(\varepsilon,\delta)$-feature unlearning of \(Z\) (w.r.t.\ \(p^r\))} if
\begin{equation}
\label{eq:FU}
I(\hat Y;Z)\ \le\ \varepsilon
\qquad\text{and}\qquad
L(f_\theta)-L^\star\ \le\ \delta.
\end{equation}
When \(\varepsilon=0\) we have \emph{exact feature unlearning}: \(\hat Y\perp Z\).
\end{defi}

The quantity \(I(\hat Y;Z)\) rigorously captures the no-inference feature unlearning principle above. Let \(\pi:=\mathbb{P}(\{Z{=}1\})\) and let \(P_{\mathrm{acc}}\) be the Bayes (optimal) accuracy for predicting \(Z\) from \(\hat Y\).
By the binary Fano inequality (see \citep{cover2006elements}), the Bayes error \(P_e:=1-P_{\mathrm{acc}}\) satisfies
\begin{equation}\label{eq:bayes_acc_bound}
    H_2(P_e)\ \ge\ H_2(\pi)-I(\hat Y;Z), \quad\text{so}\quad P_{\mathrm{acc}}\ \le\ 1-H_2^{-1}\!\big(H_2(\pi)-I(\hat Y;Z)\big),
\end{equation}
with \(H_2\) denoting binary entropy and \(H_2^{-1}\) understood on \([0,\tfrac12]\).
Thus, imposing \(I(\hat Y;Z)\le\varepsilon\) quantitatively limits the best-possible inference of \(Z\) from the released output \(\hat Y\).
Moreover, for any (possibly randomized) measurable post-processing \(T\),
the data-processing inequality yields \(I(T(\hat Y);Z)\le I(\hat Y;Z)\): downstream transformations cannot \emph{increase} $Z$-dependence. Finally, the utility budget \(\delta\) in \eqref{eq:FU} rules out trivial ``unlearn-by-collapse'' solutions (e.g., constant predictors).

\subsection{Data unlearning}\label{s2:data_unlearning}

Data unlearning (often equated with  ``machine unlearning'' in the existing literature) seeks to remove the influence of a designated subset of training data (the \emph{unlearn set}) on a learned predictor, so that an observer cannot detect whether those records participated in training.

A straightforward baseline is the model obtained by fully retraining on the \emph{retain} set alone, which yields the {\bf anchor-based unlearning principle}:

\begin{quote}
\emph{No observer should be able to distinguish the released model from the model that would have been obtained by training \emph{from scratch} on the retain set alone.}
\end{quote}
While intuitive, the ``retrain-on-retain'' heuristic depends on a retrained anchor model that is either unknown or not auditable \citep{thudi2022auditable}. In particular, if the anchor-based requirement implicitly allows fresh internal randomness on each verification attempt, and there is no mechanism that binds the curator to the original randomness and training procedure, then a curator can, at least in principle, synthesize or sample outputs that pass an indistinguishability check without performing the intended deletion. In effect, the definition presupposes an unlearned anchor and merely shifts the burden to verifying that the anchor is unlearned. Moreover, existing verification schemes either rely on auxiliary instrumentation (e.g., seeds, logs, gradient commitments, hash attestations) or are brittle to minor pipeline changes, making them easy to spoof and difficult to deploy at scale \citep{zhang2024fragile,sommer2022athena,weng2022poul,chen2021unlearning_leaks}; see Appendix~\ref{a:data_unlearning_intro}. These limitations motivate an auditable criterion expressed solely in terms of observable outputs.

To address the lack of auditability above, we propose a novel marginal invariance perspective on the unique information of the unlearn dataset in the model output, which is named {\bf marginal unlearning principle}:
\begin{quote}
\emph{No observer should be able to infer anything about the marginal information contributed by the unlearn set (beyond the retain set) from the model’s outputs, beyond prior knowledge.}
\end{quote}
Here, \emph{marginal information} denotes the unique information (or marginal effect) on the model output contributed by the unlearn set, in addition to the output information already contributed by the retain set. Informally, \emph{marginal unlearning} requires that the model output be \emph{invariant} to whether training examples were drawn purely from \(p^r\) (retain) or from a mixture that also includes the unlearn source. This invariance is \emph{auditable} from output. A separate \emph{utility} constraint preserves performance on \(p^r\).

\begin{rema}[Human analogy: direct vs.\ marginal unlearning]
\label{rem:human-analogy}
The principle aligns with evidence from cognitive psychology and neuroscience. Consider helping a person forget a text they once read. Existing \textbf{direct (anchor-based) unlearning} demands an unverifiable counterfactual: the subject should behave as if they  had \emph{never} encountered the text. By contrast, \textbf{marginal unlearning} follows a two-phase \emph{blur + reinforce} mechanism:
\begin{itemize}[leftmargin = *]
    \item \emph{Blur} (auditable invariance): suppress the diagnostic signal of the to-be-forgotten content so responses to prompts/tests from the ``retain'' vs. ``retain+unlearn'' pools become statistically indistinguishable. In our framework, this corresponds to the regularizer in Eq.~\eqref{eq:heuristic_objective} and is formalized by Def.~\ref{def:epsilon_MU}. In the cognitive psychology and neuroscience literature, it corresponds to the memory destabilization step such as Think/No-Think~\citep{AndersonGreen2001} or reconsolidation update~\citep{NaderSchafeLeDoux2000}.
    \item \emph{Reinforce} (utility preservation): strengthen desired knowledge so useful behavior is maintained. In our setup this is the first term in Eq. \eqref{eq:heuristic_objective} (learning on the retain set). Behaviorally, targeted retrieval practice robustly enhances retention and transfer \citep{KarpickeBlunt2011}.
\end{itemize}
\end{rema}

We quantify marginal information by measuring the distinguishability between the pure retain distribution $p^r$ and the mixture distribution $p^d := (1-\alpha)p^r + \alpha p^u$, where $p^u$ is the distribution of the unlearn set and  $\alpha \in (0,1]$ controls the mixing ratio and should reflect the cardinality ratio between the retain and unlearn sets. The marginal effect is substantial only when $p^d$ deviates significantly from $p^r$, indicating that the unlearning source $p^u$ contributes a significant amount of unique information not already represented in the retain set. We formalize this distinction as a binary inference problem by introducing a latent variable $Z \sim \mathrm{Bernoulli}(\pi)$, $\pi \in (0,1)$, that governs the sampling source. Conditional on $Z$, define $X_{\mathrm{margin}}$ as follows:
\[
\mathcal{L}(X_{\mathrm{margin}}\mid\{Z{=}1\}) = p^r,\qquad
\mathcal{L}(X_{\mathrm{margin}}\mid\{Z{=}0\}) = p^d.
\]
Let $\hat Y_{\mathrm{margin}}:=f(X_{\mathrm{margin}})$ denote the model output. If the released mechanism uses internal randomness, then \(\hat Y_{\mathrm{margin}}\) denotes the output law after averaging over that internal randomness; the randomness is not treated as revealed to the observer. The pair $(\hat Y_{\mathrm{margin}},Z)$ thus encodes the model's behavior under the baseline ($Z=1$) versus the shifted ($Z=0$) distribution. Therefore, the natural quantification of marginal information is the ability to infer the inclusion label $Z$ from the observed output $\hat Y_{\mathrm{margin}}$:

\begin{defi}[$\epsilon$-marginal unlearning]
\label{def:epsilon_MU}
Assume the balanced auditing prior $\mathbb P(Z=0)=\mathbb P(Z=1)=\frac12$. A model \(f\) satisfies \emph{\(\epsilon\)-marginal unlearning} if
\begin{equation}
\label{eq:odds-mu}
\sup_{D\in\mathcal{B}_{\hat{\mathcal{Y}}}}
\left|
\log\!\left(
\frac{\mathbb{P}\!\left(Z=0 \,\middle|\, \hat{Y}_{\mathrm{margin}}\in D\right)}
     {\mathbb{P}\!\left(Z=1 \,\middle|\, \hat{Y}_{\mathrm{margin}}\in D\right)}
\right)
\right|
\le \epsilon,
\end{equation}
where the supremum is taken over measurable events \(D\) for which the
posterior probabilities are well-defined.
Equivalently, for every such event \(D\), the posterior odds of ``excluded''
versus ``included'' are bounded by \(e^{\pm\epsilon}\) given
\(\{\hat Y_{\mathrm{margin}}\in D\}\).
\end{defi}

More generally, if the auditing prior is
\(\pi:=\mathbb P(Z=1)\in(0,1)\), one can use the prior-normalized posterior
odds
\[
\frac{
    \mathbb{P}\!\left(Z=0 \,\middle|\, \hat{Y}_{\mathrm{margin}}\in D\right)
    /
    \mathbb P(Z=0)
}{
    \mathbb{P}\!\left(Z=1 \,\middle|\, \hat{Y}_{\mathrm{margin}}\in D\right)
    /
    \mathbb P(Z=1)
}.
\]
The corresponding generalized \(\epsilon\)-marginal-unlearning condition is
\[
\sup_{D\in\mathcal{B}_{\hat{\mathcal{Y}}}}
\left|
\log\!\left(
\frac{
    \mathbb{P}\!\left(Z=0 \,\middle|\, \hat{Y}_{\mathrm{margin}}\in D\right)
    /
    \mathbb P(Z=0)
}{
    \mathbb{P}\!\left(Z=1 \,\middle|\, \hat{Y}_{\mathrm{margin}}\in D\right)
    /
    \mathbb P(Z=1)
}
\right)
\right|
\le \epsilon,
\]
again over events for which the conditional probabilities are well-defined.
This reduces to Definition~\ref{def:epsilon_MU} when \(\pi=\frac12\).

Definition~\ref{def:epsilon_MU} is a worst-case posterior-odds guarantee across
all measurable events in the learning-output space. The balanced prior
\(\pi=\frac12\) reflects the default auditing setting in which the observer has
no prior preference between the retain-only source and the retain-plus-unlearn
source. Under this prior, posterior odds coincide with likelihood ratios, and
\(0\)-marginal unlearning is equivalent to
\[
    \hat Y_{\mathrm{margin}}\perp Z,
\]
or equivalently
\[
    I(\hat Y_{\mathrm{margin}};Z)=0.
\]
Furthermore, inequalities such as \eqref{eq:bayes_acc_bound} show that larger
\(I(\hat Y_{\mathrm{margin}};Z)\) allows more accurate inference of the source
label \(Z\), and hence greater distinguishability between the output laws
induced by the retain-only and retain-plus-unlearn distributions. Therefore,
\(I(\hat Y_{\mathrm{margin}};Z)\) provides a natural quantification of marginal
information and motivates the following information-theoretic relaxation of
Definition~\ref{def:epsilon_MU}.

\begin{defi}[High-probability marginal unlearning]
\label{def:hp_epsilon_MU}
Assume the balanced auditing prior $\mathbb P(Z=0)=\mathbb P(Z=1)=\frac12$. A model \(f\) satisfies \emph{high-probability \((\epsilon,\delta)\)-marginal
unlearning} if there exists a measurable good set
\(G\in\mathcal B_{\hat{\mathcal Y}}\) such that
\[
    \mathbb P(\hat Y_{\mathrm{margin}}\in G)\ge 1-\delta,
\]
and
\begin{equation}
\label{eq:hp-odds-mu}
\sup_{\substack{D\in\mathcal B_{\hat{\mathcal Y}}\\ D\subseteq G}}
\left|
\log\!\left(
\frac{\mathbb{P}\!\left(Z=0 \,\middle|\, \hat{Y}_{\mathrm{margin}}\in D\right)}
     {\mathbb{P}\!\left(Z=1 \,\middle|\, \hat{Y}_{\mathrm{margin}}\in D\right)}
\right)
\right|
\le \epsilon,
\end{equation}
where the supremum is taken over measurable events \(D\subseteq G\) for which the
posterior probabilities are well-defined.
\end{defi}

Definition~\ref{def:hp_epsilon_MU} differs from Definition~\ref{def:epsilon_MU} only by allowing a small exceptional set \(G^c\). On the good set \(G\), the guarantee remains worst-case over all measurable events \(D\subseteq G\). The high-probability definition is naturally controlled by mutual information.
Indeed, if $I(\hat Y_{\mathrm{margin}};Z)\le \mu$, then as shown in Lemma~\ref{l:epsilon_DU_bound_mutual_info} below, for every
\(\epsilon>0\), the good set
\[
    G_\epsilon
    :=
    \left\{
    \hat y:
    \left|
    \log
    \frac{\mathbb P(Z=0\mid \hat Y_{\mathrm{margin}}=\hat y)}
         {\mathbb P(Z=1\mid \hat Y_{\mathrm{margin}}=\hat y)}
    \right|
    \le \epsilon
    \right\}
\]
satisfies
\[
    \mathbb P(\hat Y_{\mathrm{margin}}\in G_\epsilon)
    \ge
    1-\frac{\sqrt{2\mu}}{\tanh(\epsilon/2)}.
\]
Consequently, mutual information provides an implementation-friendly
information-theoretic relaxation of the high-probability marginal-unlearning definition.

\begin{defi}[Mutual-information marginal unlearning with utility budget]
\label{def:MI_marginal_unlearning}
For tolerances \(\mu,\rho\ge0\), with retain risk
\(L(f):=\mathbb{E}_{(X,Y)\sim p^r}[\ell(f(X),Y)]\) and
\(L^\star:=\inf_h L(h)\), we say \(f\) achieves
\emph{MI-regularized \((\mu,\rho)\)-marginal unlearning} if
\begin{equation}
\label{eq:mu}
I\!\big(\hat Y_{\mathrm{margin}};Z\big)\le \mu,
\qquad\text{and}\qquad
L(f)-L^\star\le \rho.
\end{equation}
When \(\mu=\rho=0\), we obtain \emph{exact marginal unlearning with retain
optimality}: \(\hat Y_{\mathrm{margin}}\perp Z\) and \(f\) is retain-optimal.
\end{defi}
The above definition is implementation friendly and scalable in practice, hence particularly applicable in deep learning. See Section \ref{S:Experiment} for its direct implementation on tabular data analysis and image classification. Also, see \cite{xu2025forgetting} for its implementation on LLMs.

We now reflect on the main motivations behind the definition:
\begin{itemize}[leftmargin=*]
\setlength{\itemsep}{-0.5ex}
\item \emph{Interpretability/practicality:} The posterior-odds form
\eqref{eq:odds-mu} gives a worst-case event-wise indistinguishability guarantee, while the high-probability form \eqref{eq:hp-odds-mu} allows a small exceptional set of outputs. The mutual-information criterion \eqref{eq:mu} aggregates leakage across outputs and yields the high-probability guarantee in Lemma~\ref{l:epsilon_DU_bound_mutual_info}.
\item \emph{Auditability:} Both depend only on samples of \((X_{\mathrm{margin}},Z,\hat Y_{\mathrm{margin}})\); no anchor or training transcript is needed.  
\item \emph{Sufficiency:} As shown in Section \ref{s2:sufficiency} below, with a utility guarantee, marginal unlearning implies the usual anchor-based guarantees (Theorem~\ref{thm:mu-to-anchor}). 
\item \emph{Necessity for ``robust'' models:} In Section \ref{s2:necessity}, we show that a well-generalizing retrained model cannot leak inclusion while also satisfying marginal unlearning without sacrificing regularity (Lemma~\ref{l:DU_as_condition_for_good_model}).
\end{itemize}

\subsection{Sufficiency for existing definitions: an auditable criterion}\label{s2:sufficiency}

Classical \emph{anchor-based} unlearning reads: \emph{after removing $X_u$, the released model should behave as if it were retrained on $X_r := X\setminus\{X_u\}$}. Formally, let $\mathcal{H}$ be the model/hypothesis space. A (possibly randomized) training algorithm $A$ maps a dataset $D$ to a random model $\Theta\in\mathcal{H}$ with law $A(D)$. Given a delete set $U\subseteq D$, the \emph{anchor (retrain)} distribution is
\begin{equation*}
    \Theta_r \sim \mathcal{L}\big(A(D\setminus U)\big) \text{ or sometimes} \sim \mathcal{L}\big(M(A(D\setminus U), D\setminus U, \emptyset))\big).
\end{equation*}
An unlearning mechanism $M$ takes $(\Theta,D,U)$ and outputs an updated (random) model
\begin{equation*}
    \Theta_u \sim \mathcal{L}\big(M(A(D),D,U)\big).
\end{equation*}

The existing data (machine) unlearning definitions can be summarized into the following three categories:

\begin{itemize}[leftmargin=*]
    \item \emph{(Anchored) exact unlearning:} $(A,M)$ is said to satisfy (anchored) \emph{exact unlearning} on $(S,U)$ if $\mathcal{L}(\Theta_u)\;=\;\mathcal{L}(\Theta_r)$ as probability measures on $\mathcal{H}$ \cite{cao2015unlearning, bourtoule2021machine};
    \item \emph{Divergence relaxation:} $(A,M)$ is said to satisfy (anchored) approximate unlearning if $D\!\big(\mathcal{L}(\Theta_u),\,\mathcal{L}(\Theta_r)\big)\ \le\ \varepsilon$ for some divergence $D$ (e.g., total variation, KL, JS)\cite{guo2019certified};
    \item \emph{High-probability relaxation:} $(A,M)$ is said to satisfy (anchored) $(\varepsilon,\delta)$–indistinguishable unlearning if $\mathbb{P}\big(\Theta_u\in W\big)\ \le\ e^{\varepsilon}\mathbb{P}\big(\Theta_r\in W\big)+\delta$ and $\mathbb{P}\big(\Theta_r\in W\big)\ \le\ e^{\varepsilon}\mathbb{P}\big(\Theta_u\in W\big)+\delta$
for all measurable $W\subseteq\mathcal{H}$ \cite{sekhari2021remember}.
\end{itemize}

These criteria hinge on the \emph{unknown} stochastic anchor $\Theta_r$ in the model parameter space. Consequently, they are \emph{inauditable} from the released predictor alone. Verification requires privileged access to training artifacts (e.g., checkpoints, witness models) which are fragile, easily spoofed, or unavailable in practical settings~\citep{thudi2022auditable,zhang2024fragile}. To resolve this, we translate these parameter-space definitions into verifiable \emph{output-space} criteria.

Let $p^{\mathrm{test}}$ be a test distribution. We define anchored unlearning on the output level by comparing the predictive laws $\mathcal{L}(f_\Theta(X))$ for $X \sim p^{\mathrm{test}}$, where $f_\Theta(X)$ denotes the (randomized) model output with parameter $\Theta$. Specifically, we adopt the mixture distribution $p^{\mathrm{test}} = p^d$ (the mixture of retain and unlearn sets, which represents the original training data) as the auditing ground.
For a divergence $D$ and $\varepsilon \ge 0$, a predictor $f_{\Theta_u}$ (model parametrized by $\Theta_u$) satisfies \emph{output-level anchored approximate unlearning} if:
\begin{equation}\label{eq:output_approx}
    D\big(\mathcal{L}(f_{\Theta_r}(X)),\, \mathcal{L}(f_{\Theta_u}(X))\big) \le \varepsilon, \quad \text{for } X \sim p^d.
\end{equation}

We now show that our \emph{marginal unlearning} criterion (Def.~\ref{def:epsilon_MU}) is sufficient to guarantee Eq.~\eqref{eq:output_approx}, provided the model maintains high utility. We first establish a lemma bounding the output drift via mutual information.

\begin{lem}[MI controls output drift]\label{lem:mi-tv}
Let $\Theta$ be an $\mathcal{H}$-valued random variable independent of $(X_{\textrm{margin}}, Z)$. Define the marginal output $\hat Y_{\mathrm{margin}} := f_\Theta(X_{\textrm{margin}})$, and the model output distribution tested on $p^d$ and $p^r$ respectively by
\begin{equation*}
    \mu_{\Theta}^{p^d}:= \mathcal{L}(f_\Theta(X_{\textrm{margin}}) \mid Z = 0), \quad
    \mu_{\Theta}^{p^r}:= \mathcal{L}(f_\Theta(X_{\textrm{margin}}) \mid Z = 1).
\end{equation*}
Then,
\[
\TV\big(\mu_{\Theta}^{p^d},\, \mu_{\Theta}^{p^r}\big) \;\le\; \sqrt{\frac{I(\hat Y_{\mathrm{margin}}; Z)}{2\pi(1-\pi)}}.
\]
\end{lem}
See Appendix~\ref{a:proof_lemma_mi_tv} for a proof. 

Using this lemma, we derive the main sufficiency theorem. We utilize the log-loss regret (e.g., cross-entropy in deep learning): for a (possibly randomized) model $f$, define the (population) log-loss regret under $X\sim p^r$ by
\[
\mathrm{reg}_{\log}(f)\ :=\ \mathbb{E}_{X\sim p^r}\Big[\KL\big(\mathcal{L}(Y\mid X)\,\|\,\mathcal{L}(f(X)\mid X)\big)\Big].
\]
For a random $f_\Theta$, write $\overline{\mathrm{reg}}_{\log}(f_\Theta):=\mathbb{E}_{\Theta}\big(\mathrm{reg}_{\log}(f_\Theta)\big)$.

\begin{thm}[Marginal unlearning $+$ utility $\Rightarrow$ approximate unlearning]\label{thm:mu-to-anchor}
Let $\Theta_u \sim \mathcal{L}(M(A(D),D,U))$ be the unlearned model and $\Theta_r \sim \mathcal{L}(A(D\setminus U))$ (or $\mathcal{L}(M(A(D\setminus U),D\setminus U,\varnothing))$) be the retrained anchor.
Assume: $\overline{\mathrm{reg}}_{\log}(f_{\Theta_u})\le \delta$, $\overline{\mathrm{reg}}_{\log}(f_{\Theta_r})\le \delta_g$ under $X\sim p^r$, and $I(\hat Y_{\mathrm{margin}};Z)\le \varepsilon_u$, where $\hat Y_{\mathrm{margin}}:=f_{\Theta_u}(X_{\mathrm{margin}})$ and $\Theta_u\!\perp\!(X_{\mathrm{margin}},Z)$. Then,
\[
\TV\!\Big(\mu_{\Theta_u}^{\,p^d},\,\mu_{\Theta_r}^{\,p^d}\Big)
\ \le\
\underbrace{\sqrt{\tfrac12}\,\big(\sqrt{\delta}+\sqrt{\delta_g}\big)}_{\text{\;Utility alignment}}
\ +\
\underbrace{\sqrt{\tfrac{\varepsilon_u}{2\pi(1-\pi)}}}_{\text{\;Membership independence}}
\ +\
\underbrace{\TV\!\Big(\mu_{\Theta_r}^{\,p^r},\,\mu_{\Theta_r}^{\,p^d}\Big)}_{\text{\emph{Anchor generalization}}}.
\]
\end{thm}
Here, the second term comes from Lemma~\ref{lem:mi-tv}. For the proof see Appendix~\ref{a:proof_theorem_mu-to-anchor}.

Anchor-based formulations capture the intuition ``retrain on the retain set'' but are generally not auditable. Marginal unlearning replaces the unknown anchor with a statistically verifiable independence criterion and, when paired with utility control, \emph{recovers} the anchored guarantees.

What can we say about the choice of loss function in this context? Log-loss regret \emph{is} an expected KL divergence, and Pinsker’s inequality converts KL to TV on general measurable spaces: $\TV(P,Q)\le \sqrt{\tfrac12\,\KL(P\Vert Q)}$. This enables a dimension- and output-space-agnostic bound. Many standard AI/ML objectives instantiate log-loss (classification cross-entropy, GLMs, likelihood-based deep models, CRFs, etc.), making the theorem broadly applicable.

\subsection{Necessity for any ``robust'' retrained model}\label{s2:necessity}
We now establish a complementary \emph{necessity} statement: a model cannot simultaneously exhibit (1) inclusion leakage, (2) good generalizability, and (3) marginal unlearning. Intuitively, a well-regularized retrained model that fits the remaining data cannot leak inclusion without violating the marginal-unlearning criterion.

Assume equal priors $\mathbb{P}(Z{=}0)=\mathbb{P}(Z{=}1)=\tfrac12$ and that $\mathcal{L}(f(X_0))$ and $\mathcal{L}(f(X_1))$ are mutually absolutely continuous with densities w.r.t.\ a common reference. Then, for every measurable \(D\) with positive probability under both output laws, Bayes' rule gives
\[
\log\frac{\mathbb P(Z=0\mid \hat Y_{\mathrm{margin}}\in D)}{\mathbb P(Z=1\mid \hat Y_{\mathrm{margin}}\in D)}
=\log\frac{\mathcal{L}(f(X_0))(D)}{\mathcal{L}(f(X_1))(D)}.
\]
Hence, Definition~\ref{def:epsilon_MU} is equivalent to the pair of set-wise inequalities
\[
\mathcal{L}\big(f(X_0)\big)(D)\ \le\ e^{\epsilon}\,\mathcal{L}\big(f(X_1)\big)(D),
\qquad
\mathcal{L}\big(f(X_1)\big)(D)\ \le\ e^{\epsilon}\,\mathcal{L}\big(f(X_0)\big)(D)
\quad \forall\,D,
\]
which imply mutual absolute continuity. By the Radon–Nikodym theorem, there exists a density ratio $r:=\tfrac{\mathrm d\mathcal{L}(f(X_0))}{\mathrm d\mathcal{L}(f(X_1))}$ with $e^{-\epsilon}\le r\le e^{\epsilon}$ a.s., and conversely this bound implies the set-wise inequalities. Equivalently,
\[
\operatorname*{ess\,sup}_{y}\ 
\Bigl|\log\frac{p_1(y)}{p_0(y)}\Bigr|
\le \epsilon
\quad\Longleftrightarrow\quad
\text{\(\epsilon\)-marginal unlearning,}
\]
where \(p_i\) denotes the density of \(\mathcal L(f(X_i))\), \(i=0,1\), with respect to the common reference measure.

\begin{lem}[Marginal unlearning as a condition for ``good'' retraining]
\label{l:DU_as_condition_for_good_model}
Assume \(f\) is a retrained model, \(P_i:=\mathcal{L}(f(X_i))\),
\(i \in \{0,1\}\), admit continuous densities \(p_i\) on
an open set \(\widehat{\mathcal Y}\subset\mathbb R^m\), \(f\) is \(L\)-Lipschitz from
\((\mathcal{X},d_{\mathcal X})\) to
\((\widehat{\mathcal{Y}},\|\cdot\|)\), and
\(\mathcal W_{d_{\mathcal X}}(X_1,X_0)>0\). Then either
\[
\operatorname*{ess\,sup}_{y}\ 
\Bigl|\log\frac{p_1(y)}{p_0(y)}\Bigr|
\le \epsilon,
\]
or there exists a point \(y^*\in\widehat{\mathcal Y}\), at which the
likelihood-ratio bound fails, and a local radius \(\rho_*(y^*)>0\) such that
\[
    L\ge L^*(\epsilon,y^*),
\]
where
\begin{equation}
\label{eq:Lstar}
L^*(\epsilon,y^*)
    :=
    \frac{
        \gamma(y^*)\,\omega_m\,\rho_*(y^*)^{m+1}
    }{
        (m+1)\,\mathcal W_{d_{\mathcal X}}(X_1,X_0)
    },
\qquad
\gamma(y^*) := \frac12 |p_1(y^*)-p_0(y^*)|>0.
\end{equation}
Here, \(X_1\!\sim\!p^r\), \(X_0\!\sim\!p^d\), \(\omega_m\) is the volume of the unit ball in \(\mathbb R^m\), and \(\mathcal{W}_{d_{\mathcal X}}\) is the \(1\)-Wasserstein distance on \((\mathcal{X},d_{\mathcal X})\).
\end{lem}

The proof can be found in Appendix~\ref{a:proof_DU_as_good}. The lemma makes precise the sense in which pointwise inclusion leakage becomes
detectable at the level of output distributions. The equality \(f(x_u)=y_u\)
alone is a pointwise statement. However, when \(x_u\) is a non-zero-density
point of the unlearn-source distribution \(p^u\), every sufficiently small
neighborhood of \(x_u\) carries positive \(p^u\)-mass. If \(f\) is Lipschitz,
this local input mass is transported into a controlled neighborhood of
\(y_u=f(x_u)\). Hence, pointwise inclusion leakage propagates to a local
output-level signal. When this signal is not already represented by the retain
output law \(f_\#p^r\), it becomes a detectable density-ratio gap between
\(f_\#p^d\) and \(f_\#p^r\), contradicting \(\epsilon\)-marginal unlearning
unless the Lipschitz constant satisfies the lower bound in
Lemma~\ref{l:DU_as_condition_for_good_model}.

Thus, the lemma yields the following incompatibility among:
\begin{enumerate}[itemsep=2pt, topsep=2pt, leftmargin=*]
    \item \textbf{Inclusion leakage}: the retrained model truthfully reveals
    the unlearn record, \(f(x_u)=y_u\), at a positive-density point of the
    unlearn-source distribution;
    \item \textbf{Marginal unlearning}: the output density ratio between
    \(f_\#p^d\) and \(f_\#p^r\) satisfies the \(\epsilon\) log-ratio bound of
    Definition~\ref{def:epsilon_MU};
    \item \textbf{Good regularity/generalizability}: the retrained model has
    small Lipschitz constant, so pointwise leakage propagates only through
    controlled local neighborhoods rather than arbitrary discontinuous spikes.
\end{enumerate}
Indeed, under (1) and (3), inclusion leakage propagates to a local output-level signal
near \(y_u\). If this signal is specific to the unlearn source, then it produces
a density-ratio gap and violates (2). Equivalently, if (1) and (2) both hold,
then the model cannot remain regular in the sense quantified by
Lemma~\ref{l:DU_as_condition_for_good_model}.

Consequently, marginal unlearning does not require uniform degradation of model utility. Rather, under regularity, it rules out output behavior that is locally diagnostic of the unlearn source. Thus, the effect of marginal unlearning is concentrated on information unique to the unlearn set and not already supported by the retain distribution, while utility on the retain distribution is preserved through the separate retain optimization objective.

%% file: Section_3.tex
This section formalizes our unlearning framework and explains the
rate-distortion, or compression, motivation behind it. We first present the
general information-theoretic formulation in Section~\ref{s3:information_lens}. Section~\ref{s3:problem_setting} then state the optimization problems for feature unlearning and marginal data unlearning separately, based on the unlearning definitions introduced in Section~\ref{S:Definition}. Finally,
Section~\ref{s3:rate_distortion} explains the compression intuition behind the
framework through the lens of rate-distortion theory.

\subsection{An information-theoretic lens}\label{s3:information_lens}

We connect unlearning to classical rate--distortion and compression ideas
\cite{shannon1959coding,berger2003rate}. The key quantity is the
\emph{mutual information} \(I(S';Z)\) between a chosen exposure or auditing
variable \(S'\) and the signal to forget \(Z\). We interpret compression as
reducing \(I(S';Z)\), thereby compressing away information about \(Z\), while
retaining task utility through a user-chosen cost functional. In the unified
notation, this gives constrained and Lagrangian formulations of the form
\begin{align*}
&\min_{f} \ \ \mathcal{C}(Y;S,Z)
\quad\text{s.t.}\quad I(S';Z)\le \tau,
\qquad\text{or}\qquad
\min_{f} \ \ (1-\lambda)\,\mathcal{C}(Y;S,Z)+\lambda\,I(S';Z).
\end{align*}
Here \(\mathcal C(Y;S,Z)\) denotes the lifted utility used in feature
unlearning; in marginal data unlearning, where \(Z\) is only an auditing label,
this specializes to the non-sliced utility \(\mathcal C(Y;S)\).

To start, we note that the learning outcome space $\mathcal{S}$ is not necessarily the data space $\mathcal{X}$ or the target space $\mathcal{Y}$. More specifically, the outcome space $\mathcal{S}$ can be the data space $\mathcal{X}$ in synthetic data or self-supervised learning (i.e. $\mathcal{Y} = \mathcal{X}$) setting, some new latent or embedding variable space $\mathcal{V}$ in the representation learning setting, or the target space $\mathcal{Y}$ in the supervised learning setting. Therefore, we use $\mathcal{S} \in \{\mathcal{X}, \mathcal{V}, \mathcal{Y}\}$ here to denote a generic learning outcome space.

\begin{itemize}[leftmargin=*, itemsep=2pt]
\item \textbf{Feature unlearning ($S=S'=f(X,Z)$ for $f:\mathcal{X} \times \mathcal{Z} \rightarrow \mathcal{S}$).}
$Z$ denotes feature(s) to forget and $X$ the remaining features. We construct $S=f(X,Z)$ whose \emph{outputs} reveal little about $Z$ while preserving utility for the task(s) $Y$. As explained in Section \ref{s2:feature_unlearning}, the \emph{feature} influence in model outputs is quantified by
$\ell_{\mathrm{feature}}(f):=I(S;Z)$, utility loss is given by $\ell_{\mathrm{utility}}(f):=\mathcal{C}(Y;S,Z)$, yielding the training objective
\begin{equation*}
    \inf_f (1-\lambda)\,\ell_{\mathrm{utility}}(f)+\lambda\,\ell_{\mathrm{feature}}(f).
\end{equation*}

\item \textbf{Marginal data unlearning ($S = f(X)$, $S'= S_{\mathrm{margin}} := f(X_{\mathrm{margin}})$).}
Let $p^r$ be the retain distribution and $p^u$ the unlearn source distribution, let $\alpha\in(0,1]$ be the cardinality or mixture weight of the unlearn source, define $p^d:=(1-\alpha)p^r+\alpha p^u$.
Define $(X_{\mathrm{margin}}, Z, S_{\mathrm{margin}})$ as the following: $\mathcal{L}(Z) = \mathrm{Bernoulli}(\tfrac12)$,
\[
\mathcal{L}(X_{\mathrm{margin}} \mid Z{=}1) = p^r,
\qquad
\mathcal{L}(X_{\mathrm{margin}} \mid Z{=}0) = p^d,
\]
and $S_{\mathrm{margin}}:=f(X_{\mathrm{margin}})$. As explained in Section \ref{s2:data_unlearning}, the \emph{marginal} information of $p^u$ in model outputs is quantified by
$\ell_{\mathrm{margin}}(f):=I(S_{\mathrm{margin}};Z)$, utility loss is given by $\ell_{\mathrm{utility}}(f):=\mathcal{C}(Y;S)$, yielding the training objective
\begin{equation*}
    \inf_f (1-\lambda)\,\ell_{\mathrm{utility}}(f)+\lambda\,\ell_{\mathrm{margin}}(f).
\end{equation*}
As shown in Section~\ref{S:Definition}, this formulation not only suffices classical (anchor-based) exact/approximate notions, but also avoids their unauditability by \emph{replacing} an unknown anchor with an observable independence criterion.
\end{itemize}

\subsection{Problem statements}\label{s3:problem_setting}

We now state the Pareto-optimal feature-unlearning and marginal-unlearning
problems over a generic outcome space \(\mathcal S\in\{\mathcal X,\mathcal V,\mathcal Y\}\), allowing practitioners to choose the outcome space appropriate for a given application.

\begin{defi}[Pareto optimal feature unlearning]\label{d:optimal_feature_unlearning}
Given $(X,Z)$ with $Z$ the feature(s) to forget and $X$ the remaining features, and target $Y$, solve
\begin{equation}
\label{eq:feature_unlearning}
\inf_{f:\,\mathcal{X}\times\mathcal{Z}\to\mathcal{S}}
\Big\{\, \mathcal{C}(Y;S,Z)\ :\ S\perp Z \,\Big\},
\qquad S:=f(X,Z),
\end{equation}
where $\mathcal{C}$ quantifies the utility loss of using $S$ to predict $Y$ within the observed \(Z\)-slices. To trade off utility and forgetting, relax independence via a \emph{compression-rate} penalty:
\begin{equation}
\label{eq:feature_unlearning_relaxed}
\inf_{f:\,\mathcal{X}\times\mathcal{Z}\to\mathcal{S}}
\ (1-\lambda)\,\mathcal{C}(Y;S,Z)\;+\;\lambda\,I(S;Z),\qquad \lambda \in [0,1].
\end{equation}
\end{defi}

For data unlearning, direct anchor-based definitions are generally not
auditable from the released model output alone; see
Section~\ref{s2:data_unlearning}. Section~\ref{s2:sufficiency} shows that marginal unlearning, together with an appropriate utility guarantee, recovers output-level anchored guarantees. This motivates marginal unlearning as a constructive and auditable approach to data unlearning.

\begin{defi}[Pareto-optimal marginal unlearning]
\label{d:optimal_marginal_unlearning}
Let \((X_{\mathrm{margin}},Z)\) be defined as above, let
\(S_{\mathrm{margin}}:=f(X_{\mathrm{margin}})\) for a measurable map
\(f:\mathcal X\to\mathcal S\), let \(S:=f(X)\) denote the released outcome used
to evaluate utility, and let \(Y\) be a target variable. The optimal marginal
unlearning problem is
\begin{equation}
\label{eq:marginal_unlearning}
\inf_{f:\,\mathcal{X}\to\mathcal{S}}
\Big\{\, \mathcal{C}(Y;S)\ :\ S_{\mathrm{margin}}\perp Z \,\Big\},
\end{equation}
with relaxed form
\begin{equation}
\label{eq:marginal_unlearning_relaxed}
\inf_{f:\,\mathcal{X}\to\mathcal{S}}
\ (1-\lambda)\,\mathcal{C}(Y;S)\;+\;\lambda\,I(S_{\mathrm{margin}};Z),
\qquad \lambda \in [0,1].
\end{equation}
\end{defi}

\medskip
\paragraph{Why mutual information (MI) rather than a one–way KL.}
The reader might wonder why we did not use the KL divergence between the unlearning information and the retaining information which would reflect the asymmetry of unlearning,  rather than mutual information. 
Let us explore this potential alternative.

That is, we might  penalize a directional KL between the ``to-unlearn'' and ``to-retain'' output laws. But instead we chose to  penalize
$I(S_{\mathrm{margin}};Z)$, which equals a \emph{generalized Jensen--Shannon divergence} between the conditional output laws. Writing $P_1:=\mathbb{P}_{S_{\mathrm{margin}}\mid Z=1}$, $P_0:=\mathbb{P}_{S_{\mathrm{margin}}\mid Z=0}$, $\pi:=\Pr(Z{=}1)$, and $P:=\pi P_1+(1-\pi)P_0$, we obtain
\[
I(S_{\mathrm{margin}};Z)
=\mathbb{E}_{Z}\!\big[D_{\mathrm{KL}}(\mathbb{P}_{S_{\mathrm{margin}}\mid Z}\,\|\,\mathbb{P}_{S_{\mathrm{margin}}})\big]
=\pi\,D_{\mathrm{KL}}(P_1\|P)+(1-\pi)\,D_{\mathrm{KL}}(P_0\|P).
\]
Again, $\pi:=\Pr(Z{=}1)$ is the prior probability/belief one has before observing anything, which we set to be $\frac{1}{2}$ by default.
MI offers three advantages:
\begin{itemize}[leftmargin = *]
    \item \emph{flexibility}: the reference $P$ adapts online, avoiding commitment to a fixed anchor;
    \item \emph{stability}: the generalized JS divergence is bounded by $H(Z)\le\log 2$ (binary $Z$), which typically leads to more stable gradients even with partial support mismatch;
    \item \emph{guarantee under post-processing}: for any downstream $\hat Y_{\mathrm{margin}}=g(S_{\mathrm{margin}})$ for $g: \mathcal{S} \rightarrow \mathcal{Y}$, the data-processing inequality gives $I(\hat Y_{\mathrm{margin}};Z)\le I(S_{\mathrm{margin}};Z)$, so suppressing MI at exposure controls leakage after any measurable post-processing.
\end{itemize}
A one-way KL is appropriate only when a \emph{fixed} ideal reference must be matched exactly; it can be unstable when the supports of the distributions are disjoint.

\subsection{Interpreting feature unlearning and marginal data unlearning as a generalization of rate-distortion theory}\label{s3:rate_distortion}

In this subsection, we explain how the proposed unlearning framework in Definitions~\ref{d:optimal_feature_unlearning} and \ref{d:optimal_marginal_unlearning} generalizes the classical rate-distortion framework for data compression.

To start, we first specialize the generic outcome space to the data space itself,
\(\mathcal S=\mathcal X\), as in synthetic-data design or self-supervised
learning. In this setting, we write \(S=\hat X\). We consider unlearning as compressing $(\hat{X},Z)$ (or $(\hat{X}_{\textrm{margin}},Z)$) so that the output preserves task-relevant information in  $\hat{X}$ while removing information about $Z$. A transparent special case recovers classical rate-distortion. Take
\(\mathcal S=\mathcal X\), set the task variable equal to the source,
\(Y=X\), and take the exposure signal to be the source itself, \(Z=X\). Then
\(S=\hat X=f(X,Z)=f(X,X)\) can be written simply as \(\hat X=f(X)\). If the
utility loss is chosen to be an expected distortion,
\[
\mathcal C(X;\hat X,X) := \mathbb E_X\!\left[\mathbb E\big[d(X,\hat X)\mid X\big]\right] = \mathbb E[d(X,\hat X)],
\]
then the relaxed objective becomes
\begin{equation}
\label{eq:feature_unlearning_relaxed_again}
\inf_{f:\mathcal X\to\mathcal X}
\mathbb E[d(X,\hat X)] + \lambda I(X;\hat X),
\qquad \hat X=f(X),
\end{equation}
which is the Lagrangian form of the classical rate-distortion problem.
The classical rate-distortion formulation optimizes over stochastic kernels
\(P_{\hat X\mid X}\). Under standard measurable-space assumptions, such kernels
can be represented by randomized mechanisms \(\hat X=f(X,\Theta)\), where
\(\Theta\) is internal randomness independent of \(X\). Thus the deterministic
map formulation above is the deterministic subclass, while the stochastic
rate-distortion formulation can be recovered by allowing internal randomness and
evaluating the induced output law after averaging over \(\Theta\).

Here, $I(\hat X;X)=H(X)-H(X\mid \hat X)$ measures preserved information. Since $2^{H(X)}$ and $2^{H(X\mid \hat X)}$ are the effective description lengths of $X$ before and after compression, the ratio $2^{I(X;\hat X)}$ is the expected partition cardinality induced by $\hat X$; maximizing it preserves fidelity (Appendix~\ref{a:utility_motivation}). In classic data compression, the original data set $X$ is considered ``unwanted'' in data compression because it is often too redundant for the task variable $Y$.

Therefore, feature unlearning generalizes the classical rate-distortion
viewpoint by allowing the unwanted information to be a designated feature,
source, or attribute \(Z\), rather than the entire source variable \(X\).
The feature \(Z\) may be explicitly present in \(X\), or may be provided as an
auxiliary variable correlated with \(X\). Marginal data unlearning further
generalizes this viewpoint by constructing the auditing pair
\((X_{\mathrm{margin}},Z)\), where \(Z\) labels whether the input is drawn from
the retain-only distribution or the retain-plus-unlearn mixture, and then
compressing away the information about this constructed source label.

\paragraph{Admissibility (no artificial information creation).}
Requiring \(S=f(X,Z)\) in feature unlearning or \(S=f(X)\) in marginal data
unlearning ensures
\[
    \sigma(S)\subset\sigma(X,Z)
    \qquad\text{or}\qquad
    \sigma(S)\subset\sigma(X),
\]
respectively. Thus, the released outcome is a measurable transformation of the
available information and does not create information ex nihilo. See
Appendix~\ref{a:admissibility} for a detailed explanation.

%% file: Section_4.tex
This section establishes theoretical guarantees for both \emph{feature unlearning} and \emph{marginal unlearning} as defined in Definitions~\ref{d:optimal_feature_unlearning} and~\ref{d:optimal_marginal_unlearning}. For clarity, we carry out all proofs in the canonical setting where the outcome space is the same as the input data space, i.e., \(\mathcal S=\mathcal X\) and \(S=\hat X\). This data-space setting also yields leakage guarantees for generic outcome spaces via the data-processing inequality. See Remark~\ref{rema:suffices_generic}.

For \emph{feature unlearning}, we show that driving the mutual information $I(\hat X;Z)$ small forces conditional law $\mathbb{P}(\hat X\mid Z=z)$ to concentrate around the mixture $\mathbb{P}(\hat X)$ in standard statistical distances (KL, TV), thereby suppressing $Z$–signal in the released output. For \emph{marginal unlearning}, we obtain a probabilistic, \emph{auditable} guarantee by penalizing $I(\hat X_{\mathrm{margin}};Z)$, where \(Z\) denotes the latent source label distinguishing the retain-only distribution from the retain-plus-unlearn mixture. Finally, for feature unlearning in the data-space setting, we use optimal transport to derive an analytic lifted solution for particular utility criteria.

\begin{rema}[\(\mathcal S=\mathcal X\) suffices leakage guarantee in generic setting]\label{rema:suffices_generic}
The guarantees we derive for \(S=\hat X\) immediately extend to the generic case \(S\in\{\mathcal X,\mathcal V,\mathcal Y\}\), where \(\mathcal V\) denotes an intermediate representation space (e.g., an encoder output) and \(\mathcal Y\) the prediction/output space. Indeed, let \(T:\mathcal X\to\mathcal S'\) be any measurable map representing downstream post-processing (e.g., \(T=h:\mathcal X\!\to\!\mathcal V\) or \(T=g:\mathcal X\!\to\!\mathcal Y\)), and define \(S' := T(\hat X)\).

\begin{itemize}[leftmargin=*, itemsep=2pt]
\item \emph{Feature unlearning:} If \(I(\hat X;Z)\le \varepsilon\), then by the data-processing inequality,
\[
I\!\big(T(\hat X);Z\big)\ \le\ I(\hat X;Z)\ \le\ \varepsilon,
\]
so any bound proved for \(S=\hat X\) propagates to \(S'=T(\hat X)\in\{\mathcal V,\mathcal Y\}\).

\item \emph{Marginal unlearning:} Writing \(S'_{\mathrm{margin}}:=T(\hat X_{\mathrm{margin}})\), the odds bound in Definition~\ref{def:epsilon_MU} is preserved under measurable pushforwards: for every measurable \(B\subseteq \mathcal S'\),
\[
\frac{\mathbb P\!\big(Z=0\,\big|\,   S'_{\mathrm{margin}}  \in B\big) }
     {\mathbb P\!\big(Z=1\,\big|\,   S'_{\mathrm{margin}}  \in B\big) }
\ =\
\frac{\mathbb P\!\big(Z=0\,\big|\, \hat X_{\mathrm{margin}}\in T^{-1}(B)\big)}
     {\mathbb P\!\big(Z=1\,\big|\, \hat X_{\mathrm{margin}}\in T^{-1}(B)\big)},
\]
so the same \(\epsilon\)-bound holds with \(D:=T^{-1}(B)\). Likewise,
\(I(S'_{\mathrm{margin}};Z)\le I(\hat X_{\mathrm{margin}};Z)\) by the
data-processing inequality.
\end{itemize}

Consequently, leakage bounds proved for the data-space release \(\hat X\) automatically apply to any representation or prediction obtained from \(\hat X\) by measurable post-processing. We adopt the notation \(S=\hat X\) for the remainder of the section.
\end{rema}

\subsection{Feature unlearning guarantee}\label{s4:feature_guarantee}

In practice we optimize the relaxed problem
\[
\inf_{f}\ (1-\lambda)\mathcal{C}(Y;\hat X,Z) \;+\;\lambda\,I(\hat X;Z),
\]
and consider solutions for which $I(\hat X;Z)\le \varepsilon$ for some small $\varepsilon>0$. The next bounds convert this constraint into quantitative control of the deviation between $\mathbb{P}(\hat X\mid Z=z)$ and its center $\mathbb{P}(\hat X)$ in KL and TV. We write $\mathrm{TV}(P,Q):=\tfrac12\int|dP-dQ|=\sup_{A}|P(A)-Q(A)|$.

\begin{prop}[Discrete $Z$: MI controls conditional–marginal drift]\label{prop:mi_tv_multiclass_feature}
Let $(\hat X,Z)$ be jointly distributed with $Z$ supported on a finite set $\mathcal{Z}$ and $p(z):=\mathbb{P}(Z{=}z)>0$ for all $z\in\mathcal{Z}$. Then
\begin{align}
\mathbb{E}_{Z}\!\Big[\,D_{\mathrm{KL}}\big(\mathbb{P}(\hat X\mid Z)\,\Vert\,\mathbb{P}(\hat X)\big)\,\Big]
& = I(\hat X;Z), \label{eq:disc-EKL-equals-I}\\
D_{\mathrm{KL}}\big(\mathbb{P}(\hat X\mid Z{=}z)\,\Vert\,\mathbb{P}(\hat X)\big)
& \le \frac{I(\hat X;Z)}{p(z)},\qquad \forall\,z\in\mathcal{Z}, \label{eq:disc-pt-KL}
\end{align}
and
\begin{align}
\mathbb{E}_{Z}\!\Big[\,\mathrm{TV}\big(\mathbb{P}(\hat X\mid Z),\,\mathbb{P}(\hat X)\big)\,\Big]
&\le \sqrt{\tfrac12\,I(\hat X;Z)}, \label{eq:disc-ETV}\\
\mathrm{TV}\big(\mathbb{P}(\hat X\mid Z{=}z),\,\mathbb{P}(\hat X)\big)
&\le \sqrt{\tfrac{1}{2\,p(z)}\,I(\hat X;Z)},\qquad \forall\,z\in\mathcal{Z}. \label{eq:disc-pt-TV}
\end{align}
\end{prop}

\begin{proof}
Since
\begin{equation*}
    I(\hat X;Z)\;=\;\sum_{z\in\mathcal{Z}} p(z)\,D_{\mathrm{KL}}\!\big(\,\mathbb{P}(\hat X\mid Z{=}z)\,\Vert\,\mathbb{P}(\hat X)\big),
\end{equation*}
and Pinsker's inequality gives, for each $z$,
\[
\mathrm{TV}\!\big(\,\mathbb{P}(\hat X\mid Z{=}z),\,\mathbb{P}(\hat X)\big)
\;\le\;\sqrt{\tfrac12\,D_{\mathrm{KL}}\!\big(\,\mathbb{P}(\hat X\mid Z{=}z)\,\Vert\,\mathbb{P}(\hat X)\big)}.
\]
It then follows from Jensen's inequality and concavity of square root that
\begin{align*}
    \sum_{z}p(z)\,\mathrm{TV}\big(\mathbb{P}(\hat X\!\mid\! z),\mathbb{P}(\hat X)\big) & \;\le\;\sum_{z}p(z)\sqrt{\tfrac12\,D_{\mathrm{KL}}(\mathbb{P}(\hat X\!\mid\! z)\Vert\mathbb{P}(\hat X))}\\
    & \;\le\;\sqrt{\tfrac12\sum_{z}p(z)\,D_{\mathrm{KL}}(\mathbb{P}(\hat X\!\mid\! z)\Vert\mathbb{P}(\hat X))},
\end{align*}
which establishes the bound in~\eqref{eq:disc-ETV}.

For the pointwise bound, note that all terms $D_{\mathrm{KL}}(\mathbb{P}(\hat X\!\mid\! z)\Vert\mathbb{P}(\hat X))$ are nonnegative and satisfy
\[
p(z)\,D_{\mathrm{KL}}(\mathbb{P}(\hat X\!\mid\! z)\Vert\mathbb{P}(\hat X))
\;\le\;\sum_{z'}p(z')\,D_{\mathrm{KL}}(\mathbb{P}(\hat X\!\mid\! z')\Vert\mathbb{P}(\hat X))
\;=\;I(\hat X;Z),
\]
hence
\(
D_{\mathrm{KL}}(\mathbb{P}(\hat X\!\mid\! z)\Vert\mathbb{P}(\hat X))\le I(\hat X;Z)/p(z).
\)
Combining this with Pinsker’s inequality gives \eqref{eq:disc-pt-TV}. That completes the proof.
\end{proof}

The case where $Z$ is not restricted to discrete values is a bit more involved.

\begin{thm}[General \(Z\): mutual information controls KL and TV]
\label{thm:mi_tv_general}
Let \((\hat X,Z)\) be jointly distributed on standard Borel spaces
\((\mathcal X,\mathcal B_{\mathcal X})\) and
\((\mathcal Z,\mathcal B_{\mathcal Z})\). Let \(\mathbb P_{\hat X}\) denote the
marginal law of \(\hat X\), and let
\(\{\mathbb P_{\hat X\mid Z=z}\}_{z\in\mathcal Z}\) be a regular conditional
distribution, defined for \(\mathbb P_Z\)-a.e. \(z\). Assume
\(I(\hat X;Z)<\infty\). Then:

\begin{enumerate}[leftmargin=*, itemsep=2pt]
\item \textit{KL expectation identity:}
\begin{equation}\label{eq:avg-KL}
\mathbb E_{Z}\Big[\,\KL\big(\mathbb P_{\hat X\mid Z}\Vert\,\mathbb P_{\hat X}\big)\,\Big] = I(\hat X;Z).
\end{equation}
\item \textit{KL probability tail bound:} For all \(\tau>0\),
\begin{equation}\label{eq:tail-KL}
\mathbb P\!\Big(\KL\big(\mathbb P_{\hat X\mid Z}\Vert\,\mathbb P_{\hat X}\big)\ge \tau\Big) \le \frac{I(\hat X;Z)}{\tau}.
\end{equation}
\item \textit{TV expectation bound:}
\begin{equation}\label{eq:avg-TV}
\mathbb E_{Z}\Big[\,\mathrm{TV}\big(\mathbb P_{\hat X\mid Z},\,\mathbb P_{\hat X}\big)\,\Big]
\;\le\;\sqrt{\frac{1}{2}\,I(\hat X;Z)}\, .
\end{equation}

\item \textit{TV probability tail bound:} For all $\tau>0$,
\begin{equation}\label{eq:tail-TV}
\mathbb P\;\Big(\,\mathrm{TV}\big(\mathbb P_{\hat X\mid Z},\,\mathbb P_{\hat X}\big)\ge \tau\,\Big)
\;\le\;\frac{I(\hat X;Z)}{2\,\tau^{2}}\, .
\end{equation}
\end{enumerate}
\end{thm}

\begin{proof}
By disintegration of relative entropy,
\[
I(\hat X;Z)
=
D_{\mathrm{KL}}\!\big(
\mathbb P_{\hat X,Z}
\Vert
\mathbb P_{\hat X}\otimes \mathbb P_Z
\big)
=
\int
D_{\mathrm{KL}}\!\big(
\mathbb P_{\hat X\mid Z=z}
\Vert
\mathbb P_{\hat X}
\big)\,\mathbb P_Z(dz).
\]
This proves the KL expectation identity. Since \(I(\hat X;Z)<\infty\), the
integrand is finite for \(\mathbb P_Z\)-a.e. \(z\).

\smallskip\noindent\emph{Step 1 (pointwise Pinsker).}
For $\mathbb P_Z$-a.e.\ $z$, Pinsker's inequality yields
\[
\mathrm{TV}\!\big(\mathbb P_{\hat X\mid Z=z},\,\mathbb P_{\hat X}\big)
\;\le\;\sqrt{\tfrac12\,D_{\mathrm{KL}}\!\big(\mathbb P_{\hat X\mid Z=z}\,\Vert\,\mathbb P_{\hat X}\big)} < \infty.
\]

\smallskip\noindent\emph{Step 2 (average TV bound).}
Taking expectation with respect to $Z$, it then follows from Jensen's inequality and concavity of the square root function that
\[
\mathbb E_Z\,\mathrm{TV}\big(\mathbb P_{\hat X\mid Z},\mathbb P_{\hat X}\big)
\;\le\;\mathbb E_Z\sqrt{\tfrac12\,D_{\mathrm{KL}}(\mathbb P_{\hat X\mid Z}\Vert \mathbb P_{\hat X})}
\;\le\;\sqrt{\tfrac12\,\mathbb E_Z\,D_{\mathrm{KL}}(\mathbb P_{\hat X\mid Z}\Vert \mathbb P_{\hat X})}.
\]

\smallskip\noindent\emph{Step 3 (tail bound).}
The KL tail bound follows directly from Markov's inequality:
\[
\mathbb P\!\left(
\KL(\mathbb P_{\hat X\mid Z}\Vert\mathbb P_{\hat X})\ge \tau
\right)
\le
\frac{\mathbb E_Z\KL(\mathbb P_{\hat X\mid Z}\Vert\mathbb P_{\hat X})}{\tau}
=
\frac{I(\hat X;Z)}{\tau}.
\] From Pinsker inequality we have
\(
\mathrm{TV}(\mathbb P_{\hat X\mid Z},\mathbb P_{\hat X})^{2}\le \tfrac12\,D_{\mathrm{KL}}(\mathbb P_{\hat X\mid Z}\Vert \mathbb P_{\hat X}),
\)
$\mathbb P_Z$-a.s.. It follows from Markov's inequality that, for any $\tau>0$,
\[
\mathbb P\!\left(
\TV(\mathbb P_{\hat X\mid Z},\mathbb P_{\hat X})\ge \tau
\right)
=
\mathbb P\!\left(
\TV(\mathbb P_{\hat X\mid Z},\mathbb P_{\hat X})^2\ge \tau^2
\right)
\le
\frac{I(\hat X;Z)}{2\tau^2},
\]
which is \eqref{eq:tail-TV}. That completes the proof.
\end{proof}

\begin{rema}[On pointwise bounds without discreteness]
In the discrete case with $\mathbb P(Z=z)=p(z)>0$ one has the pointwise estimate
\begin{equation*}
    \mathrm{TV}(\mathbb P_{\hat X\mid z},\mathbb P_{\hat X})\le \sqrt{I(\hat X;Z)/(2p(z))},
\end{equation*}
because $I(\hat X;Z)=\sum_z p(z)\,D_{\mathrm{KL}}(\mathbb P_{\hat X\mid z}\Vert\mathbb P_{\hat X})$ and each summand is nonnegative.
For general (non-discrete) $Z$, no analogous pointwise inequality can hold without additional assumptions (e.g., uniform lower bounds on the density of $Z$ over sets of positive reference measure), since an expectation constraint $\int p_Z(z)\,K(z)\,dz=I$ does not control the essential sup of $K(z)$.
Therefore, the average and tail bounds \eqref{eq:avg-TV}–\eqref{eq:tail-TV} are the natural general theoretical guarantees.
\end{rema}

\subsection{Data unlearning guarantee}\label{s4:data_guarantee}

For data unlearning we operate on the data space and penalize
\[
I\!\big(\hat X_{\mathrm{margin}};Z\big),
\]
with $\hat X_{\mathrm{margin}}:=f(X_{\mathrm{margin}})$, $X_{\mathrm{margin}}\mid(Z{=}1)\sim p^r$, $X_{\mathrm{margin}}\mid(Z{=}0)\sim p^d$, and $Z\sim\mathrm{Bernoulli}(1/2)$, as in Section \ref{s2:data_unlearning}. That is, we penalize the information that \(\hat X_{\mathrm{margin}}\) carries about the
latent source label \(Z\), which distinguishes the retain-only distribution
from the retain-plus-unlearn mixture:
\begin{equation}\label{eq:estimate_point_unlearning}
\inf_{f:\,\mathcal{X}\to\mathcal{X}}
\Big\{\, (1-\lambda) \mathcal{C}(Y;\hat{X})\;+\;\lambda\,I(\hat X_{\mathrm{margin}};Z)\,\Big\}.
\end{equation}
The next lemma shows that controlling $I(\hat X_{\mathrm{margin}};Z)$ directly yields an inference guarantee.

\begin{lem}[MI controls high-probability posterior odds]\label{l:epsilon_DU_bound_mutual_info}
Let $\hat X_{\mathrm{margin}} =f(X_{\mathrm{margin}})$ defined as above with $X_0\sim p^d$, $X_1\sim p^r$, and $Z\in\{0,1\}$ indicating the source. Then for any $\eta\in(0,1)$,
\[
\mathbb{P}\!\left(
\left|\log\!\frac{\mathbb{P}\big(Z{=}0\mid \hat X_{\mathrm{margin}}\big)}
{\mathbb{P}\big(Z{=}1\mid \hat X_{\mathrm{margin}}\big)}\right|
\ \le\ \log\!\frac{1+\eta}{1-\eta}
\right)
\ \ge\ 1\;-\;\frac{2}{\eta}\,\sqrt{\tfrac{1}{2}\,I\!\big(\hat X_{\mathrm{margin}};Z\big)}.
\]
\end{lem}

See Appendix~\ref{a:proof_DU_bound_mutual_info} for the proof. Setting
\[
    \eta_\epsilon:=\frac{e^\epsilon-1}{e^\epsilon+1}
    =
    \tanh(\epsilon/2),
\]
Lemma~\ref{l:epsilon_DU_bound_mutual_info} implies
\[
\mathbb P\!\left(
\left|
\log
\frac{\mathbb P(Z=0\mid \hat X_{\mathrm{margin}})
     }{\mathbb P(Z=1\mid \hat X_{\mathrm{margin}})}
\right|
\le \epsilon
\right)
\ge
1-\frac{\sqrt{2I(\hat X_{\mathrm{margin}};Z)}}{\tanh(\epsilon/2)}.
\]
Thus, if \(I(\hat X_{\mathrm{margin}};Z)=0\), then
\(\hat X_{\mathrm{margin}}\perp Z\). More generally, if
\(I(\hat X_{\mathrm{margin}};Z)\) is small, then the set of outputs with
posterior odds larger than \(e^\epsilon\) has small marginal probability.
Equivalently, except on a small-probability set of outputs, observing
\(\hat X_{\mathrm{margin}}\) does not substantially change the posterior odds
of whether the input came from the retain-only distribution or the
retain-plus-unlearn mixture.

The following guarantee translates a small ``compression rate'' into an auditable inference bound for high-probability marginal unlearning as defined in Definition \ref{def:hp_epsilon_MU}:

\begin{thm}[High-probability marginal unlearning via compression rate]\label{thm:DU-Guarantee-via-Compression-Rate}
Let $Z\sim\mathrm{Bernoulli}(1/2)$. If $I(\hat X_{\mathrm{margin}};Z)\le \mu < \frac{1}{2}(\frac{e^{\epsilon} - 1}{e^{\epsilon} + 1})^2$, then for every \(\epsilon>0\), \(f\) satisfies high-probability
\((\epsilon,\delta_\epsilon)\)-marginal unlearning in the sense of
Definition~\ref{def:hp_epsilon_MU}, with
\[
    \delta_\epsilon
    =
    \frac{\sqrt{2\mu}}{\tanh(\epsilon/2)}
    =
    \frac{e^\epsilon+1}{e^\epsilon-1}\sqrt{2\mu}.
\]
\end{thm}

Hence, one can select the regularization weight $\lambda$ to drive the observable $I(\hat X_{\mathrm{margin}};Z)$ below a target $\mu$ aligned with a desired $(\epsilon,\text{confidence})$ pair ($\epsilon$ is the marginal difference allowed in the $\epsilon$-marginal unlearning definition, confidence is the guarantee probability provided by Theorem \ref{thm:DU-Guarantee-via-Compression-Rate}), directly yielding an \emph{auditable} post hoc certificate from samples of $(X_{\mathrm{margin}},Z,\hat X_{\mathrm{margin}})$.

Finally, if one wants to apply the marginal unlearning guarantee together with utility to provide an \emph{anchor-based unlearning} (also known as approximate unlearning) guarantee, Theorem \ref{thm:DU-Guarantee-via-Compression-Rate} together with Theorem \ref{thm:mu-to-anchor} provide an auditable/tractable guarantee (based on the observed $\mu$) under the assumption of small log-loss for both unlearned model and the anchor (retain from scratch) model.

\subsection{An analytic feature unlearning solution independent of the downstream tasks}\label{s4:feature_analytic_solution}
Here, we show that for specific types of cost functions, we can obtain an analytic feature-unlearning solution using optimal transport and extreme sigma-algebras. In particular, we focus on the following cost functions:

\begin{itemize}[itemsep=2pt, topsep=2pt, leftmargin=*]
\item \emph{Mutual information maximization:} \(\mathcal{C}(Y;\hat X,Z)=-I(Y;\hat X,Z)\) maximizes informativeness about $Y$.
\item \emph{Posterior concentration:}
$\mathcal{C}(Y;\hat X,Z)
=
-\mathbb E\!\left[
D_{\mathrm{KL}}\!\big(
\mathbb P(Y\mid \hat X,Z)\,\Vert\,\mathbb P(Y)
\big)
\right],$ which coincides with \(-I(Y;\hat X,Z)\) provided the KL-divergence is well-defined.
\item \emph{Conditional–probability energy:}
\[
\mathcal{C}(Y;\hat X,Z)=
\begin{cases}
-\mathbb E\!\left[\mathbb P(Y\in A\mid \hat X,Z)^2\right], & \text{for classification events \(A\in\mathcal B_{\mathcal Y}\)},\\[2pt]
-\mathbb \|\mathbb E(Y\mid \hat X,Z)\|^2_{L^2}, & \text{for regression with \(Y\in L^2\)}.
\end{cases}
\]
\end{itemize}

Although these criteria are different, they are all monotone with respect to
the lifted sigma-algebra generated by \((\hat X,Z)\). Under the mild assumptions below, the Wasserstein barycenter construction produces a canonical representation \(\bar X\perp Z\) whose lifted sigma-algebra \(\sigma(\bar X,Z)\) is maximal among admissible lifted representations. Consequently, \(\bar X\) simultaneously optimizes the lifted monotone utility criteria above, subject to the hard feature-unlearning constraint
\(\hat X\perp Z\).

We start with the following result, which establishes that the Wasserstein-2 barycenter generates the finest sigma-algebra among all admissible outcomes:

\begin{lem}[Wasserstein barycenter $\implies$ lifted finest sigma-algebra]
\label{l:Finest_Sigma_Algebra}
Let \(P_z:=\mathcal L(X\mid Z=z)\in\mathcal P_{2,ac}(\mathcal X)\), and let
\(\bar P\) denote their \(\mathcal W_2\)-barycenter. Let \(T_z:P_z\to\bar P\)
be the Brenier map and define the barycentric representation $\bar X:=T_Z(X)$. Under the Brenier a.e.-invertibility guaranteed by the absolute-continuity
assumptions, one has $\sigma(X,Z)=\sigma(\bar X,Z)$ up to null sets. Consequently, for every admissible \(\hat X=f(X,Z)\),
\[
    \sigma(\hat X,Z)\subset \sigma(\bar X,Z).
\]
\end{lem}

See Appendix \ref{a:fine_algebra} for a proof. Now, the importance of the above result lies in the monotonicity of the objective functions listed earlier w.r.t.\ the sigma-algebra generated by random variables. That is, the fineness of the sigma-algebra is equivalent to the amount of information (in a probability-theoretic sense) contained by the random variable:

\begin{lem}[Monotonicity w.r.t.\ sigma–algebra]\label{l:monotonicity}
If $\sigma(X_1)\subset \sigma(X_2)$, then the following hold whenever the quantities are well-defined:
\begin{itemize}[itemsep=0pt, topsep=2pt, leftmargin=*]
    \item \(H(Y\mid X_2)\le H(Y\mid X_1)\) for any $Y: \Omega \to \mathcal{Y}$ with finite-entropy in discrete setting or with the absolutely continuous conditional laws \(\mathcal L(Y\mid X_s)\) and finite conditional differential entropies in continuous setting.
    \item $I(Y;X_1)\le I(Y;X_2)$ for any $Y: \Omega \to \mathcal{Y}$.
    \item $\big\|\mathbb P(Y\in A\mid X_1)\big\|_{L^2}^2 \le \big\|\mathbb P(Y\in A\mid X_2)\big\|_{L^2}^2;$ for any $A \in \mathcal{B}_{\mathcal{Y}}$, any $Y: \Omega \to \mathcal{Y}$.
    \item $\big\|\mathbb E(Y\mid X_1)\big\|_{L^2}^2 \le \big\|\mathbb E(Y\mid X_2)\big\|_{L^2}^2.$ for any $Y\in L^2$.
\end{itemize}
\end{lem}

See a proof in Appendix \ref{a:proof_monotonicity_sigma}. As a result, the information quantifications discussed above are naturally monotone w.r.t.\ the fineness of the generated sigma-algebra: See Lemma~\ref{l:monotonicity} for the formal result. Informally, Lemma~\ref{l:monotonicity} shows that conditioning on a finer sigma-algebra reduces conditional uncertainty and increases information-based
or energy-based measures of predictive informativeness.

\begin{thm}[Unified optimal feature unlearning solution]\label{th:optimal_solution}
Under the assumptions of Lemma~\ref{l:Finest_Sigma_Algebra}, the following are equivalent for any admissible $\hat X=f(X,Z)$:
\begin{itemize}[itemsep=0pt, topsep=2pt, leftmargin=*]
\item $\sigma((\hat{X},Z))=\sigma((\bar{X},Z))$.
\item $\hat{X}\in\argmin\{H(Y\mid \tilde{X},Z):\ \tilde{X}\perp Z\}$ for all $Y: \Omega \to \mathcal{Y}$ with finite-entropy in discrete setting or with the absolutely continuous conditional laws \(\mathcal L(Y\mid X_s)\) and finite conditional differential entropies in continuous setting.
\item $\hat{X}\in\argmax\{I(Y;\tilde{X},Z):\ \tilde{X}\perp Z\}$ for all $Y: \Omega \to \mathcal{Y}$.
\item $\hat{X}\in\argmax\{\|\mathbb{P}(Y\in A\mid \tilde{X},Z)\|_2^2:\ \tilde{X}\perp Z\}$ for all $A\in\mathcal B_{\mathcal Y}$ and $Y: \Omega \to \mathcal{Y}$.
\item $\hat{X}\in\argmax\{\|\mathbb{E}(Y\mid \tilde{X},Z)\|_2^2:\ \tilde{X}\perp Z\}$ for all $Y \in L^2(\Omega;\mathcal Y)$.
\end{itemize}
\end{thm}

\begin{proof}
First, by Lemma~\ref{l:Finest_Sigma_Algebra}, for every admissible
\(\hat X=f(X,Z)\), we have $\sigma(\hat X,Z)\subset \sigma(\bar X,Z)$. Now, suppose that
\[
    \sigma(\hat X,Z)=\sigma(\bar X,Z),
\]
then conditioning on \((\hat X,Z)\) is the same as conditioning on
\((\bar X,Z)\). Since \(\bar X\perp Z\), the barycenter representation is
feasible for the hard feature-unlearning constraint. Moreover, by
Lemma~\ref{l:monotonicity}, for every feasible admissible \(\tilde X\), we have
\begin{itemize}[leftmargin = *]
    \item $H(Y\mid \bar X,Z)\le H(Y\mid \tilde X,Z)$,
    \item $I(Y;\tilde X,Z)\le I(Y;\bar X,Z)$,
    \item $\|\mathbb P(Y\in A\mid \tilde X,Z)\|_{L^2}^2 \le \|\mathbb P(Y\in A\mid \bar X,Z)\|_{L^2}^2$, and
    \item $\|\mathbb E(Y\mid \tilde X,Z)\|_{L^2}^2 \le \|\mathbb E(Y\mid \bar X,Z)\|_{L^2}^2$,
\end{itemize}
whenever the corresponding quantities are well-defined. Since
\(\sigma(\hat X,Z)=\sigma(\bar X,Z)\), the same optimal values are attained by \(\hat X\). Therefore, \(\hat X\) is one optimal solution to all the above objectives under the hard independence constraint.

Conversely, suppose that one of the stated optimality conditions holds for \(\hat X\). We show that $\sigma(\bar X,Z)\subset \sigma(\hat X,Z)$. More specifically, let \(B\in\sigma(\bar X,Z)\) be arbitrary, we show that \(B\in\sigma(\hat X,Z)\).  Indeed, since the above four optimality are true for all measurable $Y$, let $Y=\mathbbm{1}_B$.

If the conditional-entropy optimality holds, then because \(Y=\mathbbm{1}_B\) is \(\sigma(\bar X,Z)\)-measurable, we have $H(Y\mid \bar X,Z)=0$. Since \(\bar X\) is feasible and \(\hat X\) is optimal, we also obtain
\[
    0 \leq H(Y\mid \hat X,Z) \leq H(Y\mid \bar X,Z)=0.
\]
Hence, \(Y=\mathbbm{1}_B\) is \(\sigma(\hat X,Z)\)-measurable, and therefore
\(B\in\sigma(\hat X,Z)\).

If the mutual-information optimality holds, take again
\(Y=\mathbbm 1_B\), where \(B\in\sigma(\bar X,Z)\). Since \(Y\) is
\(\sigma(\bar X,Z)\)-measurable, we have $I(Y;\bar X,Z)=H(Y)$. Because \(\bar X\) is feasible and \(\hat X\) is optimal for this target,
\[
    I(Y;\hat X,Z)\ge I(Y;\bar X,Z)=H(Y).
\]
On the other hand, \(I(Y;\hat X,Z)\le H(Y)\). Hence, $I(Y;\hat X,Z)=H(Y)$,
or equivalently,
\[
    H(Y\mid \hat X,Z)=0.
\]
Therefore \(Y=\mathbbm 1_B\) is \(\sigma(\hat X,Z)\)-measurable, and so
\(B\in\sigma(\hat X,Z)\).

If the conditional-probability-energy optimality holds, take
\(A=\{1\}\). Since \(Y=\mathbbm 1_B\) is \(\sigma(\bar X,Z)\)-measurable, $\mathbb P(Y\in A\mid \bar X,Z) = \mathbbm 1_B$. Hence,
\[
    \|\mathbb P(Y\in A\mid \bar X,Z)\|_{L^2}^2
    =
    \mathbb P(B).
\]
By Lemma~\ref{l:monotonicity},
\[
    \|\mathbb P(Y\in A\mid \hat X,Z)\|_{L^2}^2
    \le
    \mathbb P(B).
\]
On the other hand, since \(\bar X\perp Z\), the barycenter representation is
feasible. Therefore, optimality of \(\hat X\) for the conditional-probability
energy implies
\[
    \|\mathbb P(Y\in A\mid \hat X,Z)\|_{L^2}^2
    \ge
    \|\mathbb P(Y\in A\mid \bar X,Z)\|_{L^2}^2
    =
    \mathbb P(B).
\]
Thus, the equality holds. Let
\[
    q_B:=\mathbb P(B\mid \hat X,Z).
\]
Then $q_B \in [0,1]$ satisfies
\[
    \mathbb E[q_B]=\mathbb P(B),
    \qquad
    \mathbb E[q_B^2]=\mathbb P(B).
\]
Therefore,
\[
    \mathbb E[q_B(1-q_B)]=0.
\]
Since \(q_B(1-q_B)\ge 0\), we have \(q_B\in\{0,1\}\) a.s. Moreover,
because \(q_B=\mathbb E(\mathbbm 1_B\mid \hat X,Z)\),
\[
\begin{aligned}
    \mathbb E[(\mathbbm 1_B-q_B)^2]
    &=
    \mathbb P(B)-2\mathbb E[\mathbbm 1_B q_B]+\mathbb E[q_B^2]  \\
    &=
    \mathbb P(B)-2\mathbb E[q_B^2]+\mathbb E[q_B^2] \\
    &=
    \mathbb P(B)-\mathbb E[q_B^2]
    =
    0.
\end{aligned}
\]
Thus, \(\mathbbm 1_B=q_B\) a.s.. Since \(q_B\) is
\(\sigma(\hat X,Z)\)-measurable, \(B\in\sigma(\hat X,Z)\).

Finally, if the conditional-mean-energy optimality holds, take again
\(Y=\mathbbm{1}_B\in L^2\). Since \(B\in\sigma(\bar X,Z)\),
\[
    \mathbb E(Y\mid \bar X,Z)=\mathbbm{1}_B,
\]
and hence $\|\mathbb E(Y\mid \bar X,Z)\|_{L^2}^2=\mathbb P(B)$. By Lemma~\ref{l:monotonicity},
\[
    \|\mathbb E(Y\mid \hat X,Z)\|_{L^2}^2\le \mathbb P(B).
\]
Optimality forces equality. Writing
\[
    q_B:=\mathbb E(Y\mid \hat X,Z)=\mathbb P(B\mid \hat X,Z),
\]
the same argument as above gives
\[
    \mathbbm{1}_B=q_B\quad\text{a.s.}
\]
Thus, \(B\in\sigma(\hat X,Z)\).

Therefore, under any of the stated optimality conditions, every
\(B\in\sigma(\bar X,Z)\) belongs to \(\sigma(\hat X,Z)\). Hence,
\[
    \sigma(\bar X,Z)\subset \sigma(\hat X,Z).
\]
Combining this with the inclusion $\sigma(\hat X,Z)\subset \sigma(\bar X,Z)$, we obtain
\[
    \sigma(\hat X,Z)=\sigma(\bar X,Z).
\]
This proves the equivalence and completes the proof.
\end{proof}

%% file: Section_5.tex
In this section we describe the practical realization of our framework in terms of numerical algorithms. We first give general marginal information regularization algorithms that apply broadly to \emph{feature unlearning} defined by Definition \ref{d:optimal_feature_unlearning} and \emph{marginal (data–point) unlearning} by Definition~\ref{d:optimal_marginal_unlearning}. We then present the analytic (closed-form) solver available for feature unlearning under the structured utility class in Theorem~\ref{th:optimal_solution}.

\subsection{General cost feature unlearning: solution to Definition \ref{def:feature-unlearning-output}}

For general utility costs, we solve the relaxed feature-unlearning problem in
Definition~\ref{d:optimal_feature_unlearning} by penalizing the mutual
information between the released outcome and the feature to be forgotten. Given
a parametrized measurable map
\[
    f_\theta:\mathcal X\times\mathcal Z\to\mathcal S,
\]
we set
\[
    S=f_\theta(X,Z)
\]
and optimize the regularized objective
\[
    (1-\lambda)\mathcal C(Y;S,Z)+\lambda I(S;Z).
\]
Here \(I(S;Z)\) may be estimated using a differentiable plug-in,
variational, or adversarial mutual-information estimator, depending on the
model class and data modality.

\begin{algorithm}[H]
\caption{Feature Unlearning on Model via Regularization}
\label{alg:feature_reg}
\begin{algorithmic}[1]
\REQUIRE Dataset $D = (X,Y,Z) = \{(x_i, y_i, z_i)\}_{i=1}^{N}$, loss function $\mathcal{C}$, learning rate $\eta$, batch size $B$, number of epochs $T$, regularization parameter $\lambda$.\\
\textbf{(Optional:)} Pre-trained neural network $f_{\theta_{\text{origin}}}$ with parameters $\theta_{\text{origin}}$.
\ENSURE Unlearned model parameters $\theta$.

\STATE \textbf{Initialize:} $\theta \leftarrow$ random initialization
\IF{pre-trained model available}
    \STATE Load $\theta \leftarrow \theta_{\text{origin}}$
\ENDIF

\FOR{$t = 1$ to $T$} 
    \STATE Shuffle dataset $D$
    \FOR{each mini-batch $d \subset D$ with $d = (X_d,Y_d,Z_d)$}
        \STATE Compute predictions: $S_d = f_{\theta}(X_d)$
        \STATE Compute loss: $\mathcal{L}_{\text{reg}} = (1-\lambda)\mathcal{C}(Y_d;S_d,Z_d) + \lambda I(S_d;Z_d)$
        \STATE Compute gradients: $\nabla_{\theta} \mathcal{L}_{\text{reg}}$
        \STATE Update parameters: $\theta \leftarrow \theta - \eta \nabla_{\theta} \mathcal{L}_{\text{reg}}$
    \ENDFOR
\ENDFOR

\STATE \textbf{Return} Unlearned model parameters $\theta$

\end{algorithmic}
\end{algorithm}

\subsection{General cost marginal data unlearning: solution to Definition \ref{def:MI_marginal_unlearning}}\label{s5:data_algorithm}

For marginal unlearning, we operate in prediction space and penalize $I(S_{\mathrm{margin}};Z)$ with the paired construction $X_{\mathrm{margin}}\!\mid_{\{Z{=}1\}}\,\sim p^r$, $X_{\mathrm{margin}}\!\mid_{\{Z{=}0\}}\,\sim p^d$ as in Definition~\ref{d:optimal_marginal_unlearning}. We train by minimizing $$(1-\lambda)\mathcal{C}(Y;S)+\lambda\,I(S_{\mathrm{margin}};Z).$$ The population guarantee follows from Lemma~\ref{l:epsilon_DU_bound_mutual_info} and Theorem~\ref{thm:DU-Guarantee-via-Compression-Rate}.

\begin{algorithm}[H]
\caption{Marginal Data Unlearning via Regularization}
\label{alg:point_reg}
\begin{algorithmic}[1]
\REQUIRE Remaining dataset $R = \{(x_i, y_i)\}_{i=1}^{N}$, Unlearning dataset $U = \{(x_i, y_i)\}_{i=N+1}^{N+K}$, trained neural network $f_{\theta_{\text{origin}}}$ with parameters $\theta_{\text{origin}}$, loss function $\mathcal{C}$, learning rate $\eta$, batch size $B$, number of epochs $T$, regularization parameter $\lambda$.
\ENSURE Unlearned model parameters \(\theta\) and empirical leakage estimate
\(\widehat{\mu}\approx I(S_{\mathrm{margin}};Z)\). If
\(I(S_{\mathrm{margin}};Z)\le \mu\), then
Theorem~\ref{thm:DU-Guarantee-via-Compression-Rate} certifies
high-probability \((\epsilon,\delta_\epsilon)\)-marginal unlearning with
\[
    \delta_\epsilon
    =
    \frac{e^\epsilon+1}{e^\epsilon-1}\sqrt{2\mu}.
\]
\STATE \textbf{Initialize:} Load pre-trained parameters $\theta \leftarrow \theta_{\text{origin}}$.
\FOR{$t = 1$ to $T$}
    \STATE Shuffle datasets $R$ and $U$.
    \FOR{each mini-batch $r \subset R$ and $u \subset U$, where $r = (X_r, Y_r)$ and $u = (X_u, Y_u)$}
        \STATE Compute predictions: $S_r = f_{\theta}(X_r)$ and $S_u = f_{\theta}(X_u)$.
        \STATE Construct $S_0 = \text{concat}(S_r, S_u, \dim = 0)$\footnote{$\dim = 0$ here means row-wise concatenation.} and $S_1 = S_r$.
        \STATE Define joint distribution $(S_{\textrm{margin}}, Z)$ where: $S_{\textrm{margin}} |_{Z=0} = S_0$ and $S_{\textrm{margin}} |_{Z=1} = S_1$.
        \STATE Compute regularized loss: $\mathcal{L}_{\text{reg}} = (1-\lambda)\mathcal{C}(Y_r;S_r) + \lambda I(S_{\textrm{margin}}; Z)$.
        \STATE Compute gradients: $\nabla_{\theta} \mathcal{L}_{\text{reg}}$.
        \STATE \textbf{Update} model parameters: $\theta \leftarrow \theta - \eta \nabla_{\theta} \mathcal{L}_{\text{reg}}$.
    \ENDFOR
\ENDFOR

\STATE Estimate final leakage \(\widehat\mu\approx I(S_{\mathrm{margin}};Z)\).
\STATE \textbf{Return} unlearned model parameters \(\theta\) and certificate estimate \(\widehat\mu\).
\end{algorithmic}
\end{algorithm}

\subsection{Analytic optimal feature unlearning for specific utilities via Theorem~\ref{th:optimal_solution}} \label{s5:analytic_feature_algorithm}

When specializing to the case where $\mathcal{S} = \mathcal{X}$ (hence $S = \hat{X}$) and the utility is one of the monotone information criteria in Theorem~\ref{th:optimal_solution}. Examples of such utilities include $-I(Y;\hat X,Z)$,
the equivalent posterior-concentration objective $-\mathbb E \![
    D_{\mathrm{KL}} \!(
    \mathbb P(Y\mid \hat X,Z)\,\Vert\,\mathbb P(Y)
    )
    ]$,
and the classification or regression energy criteria in
Theorem~\ref{th:optimal_solution}. The barycentric representation $\bar X=T_Z(X)$ pushes every conditional law \(P_z\) to a common barycenter law \(\bar P\),
thereby enforcing $\bar X\perp Z$. In such cases, feature unlearning admits a closed-form/analytic optimal solution: the $\mathcal{W}_2$ barycenter of $\{X\mid Z=z\}_z$, which produces the finest admissible $\sigma$-algebra and enforces $\hat X\!\perp\! Z$.\footnote{This analytic solution applies to feature unlearning only; marginal (data–point) unlearning does not admit a comparable barycenter form in general.} More specifically, the Algorithm~\ref{alg:bary} below contains a pseudocode for the analytic solution of the feature unlearning problems by calculating the Wasserstein-2 barycenter of marginal distributions w.r.t. the feature to remove $Z$.

\begin{algorithm}[H]
\caption{Analytic Unlearning Solution for Feature Unlearning via Wasserstein Barycenter}
\label{alg:bary}
\begin{algorithmic}[1]
\REQUIRE Dataset $D = (X,Z) = \{(x_i, z_i)\}_{i=1}^{N}$, Maximum number of iterations $T$, Convergence threshold $\epsilon$.
\ENSURE Estimated Wasserstein barycenter $\bar{X}$.

\STATE \textbf{Initialize:} $\bar{X} \leftarrow \bar{X}_0$ \COMMENT{Random initialization of barycenter}

\FOR{$t = 1$ to $T$} 
    \FOR{each unique value of $Z$: $z \in \text{unique}(Z)$}
        \STATE Compute the optimal transport map $T_z$ that maps $\bar{X}$ to $X_z := X |_{Z=z}$.
    \ENDFOR
    \STATE Compute updated barycenter: $\bar{X}_{\text{new}} = \sum_{z \in \text{unique}(Z)} \frac{|X_z|}{|X|} T_z(\bar{X})$
    \STATE Compute convergence criterion: $\epsilon_t = \mathcal{W}_2(\bar{X},\bar{X}_{\text{new}})$.
    \STATE Update barycenter: $\bar{X} \leftarrow \bar{X}_{\text{new}}$.
    \IF{$\epsilon_t < \epsilon$}
        \STATE \textbf{break} \COMMENT{Terminate loop if convergence threshold is met}
    \ENDIF
\ENDFOR

\STATE \textbf{Return} Estimated Wasserstein barycenter $\bar{X}$.

\end{algorithmic}
\end{algorithm}

The computation of the Wasserstein-2 barycenter is known to be NP-hard in general~\cite{altschuler2022wasserstein}. Our proposed algorithm applies the fixed point estimation mechanism which guarantees convergence to the true barycenter \cite{agueh2011barycenters}, when provided true optimal transport maps. Therefore, one can apply different computational methods to estimate the optimal transport maps $T_z$'s, e.g. linear maps or neural networks, and then follow the iterative method in Algorithm \ref{alg:bary} to find an estimation of the true barycenter. In the experiment \ref{s6:analytic_feature_experiment} below, we use neural optimal transport \cite{korotin2022neural}, which formulates the Kantorovich Duality as an adversarial generative network (GAN) structure to estimate the true optimal transport maps, together with the proposed iterative method (which becomes McCann interpolation in binary marginals setting) to estimate the Wasserstein-2 barycenter. We defer a detailed study of different estimation methods to future work.

We note here again that the above analytic solution only applies when the cost function is among the listed ones. Otherwise, the assumption of Theorem \ref{th:optimal_solution} is not satisfied and one should apply information-theoretic relaxation to approach the optimal solution for each of the unlearning problems with general cost functions.

%% file: Section_6.tex
We conduct numerical experiments on synthetic and real-world datasets to
validate the proposed framework, with emphasis on its theoretical guarantees,
diagnostic behavior, and interpretability rather than exhaustive benchmarking
against all existing unlearning methods. A more comprehensive evaluation,
including improved barycenter estimation for feature unlearning, principled
regularization selection, and refined mutual-information estimation for
marginal data-point unlearning, is left for future work. The code is available
at \url{https://github.com/xushizhou/Machine_Unlearn_via_Info_Reg}.

\subsection{Feature Unlearning via Regularization: Algorithm \ref{alg:feature_reg}}\label{s6:feature_tabular_experiment}

\paragraph{Setting}
We adopt a fixed $5$-fold validation. In particular, for each fold $f\in\{1,\dots,5\}$, the data are split into train/validation/test as follows: we first take the provided train/test split for the fold, then partition the training split into (train, validation) using a stratified 80/20 split with a fixed seed. Models are trained only on the train portion, all thresholds and tradeoff knobs are selected using the validation portion if needed, and final metrics are reported on the test portion. We aggregate results as mean $\pm$ standard deviation over the five folds. For comparability purpose, all methods share the same classifier backbone: a two-layer MLP (width 128, ReLU, dropout 0.2), and the same training: Adam with learning rate $10^{-3}$, weight decay $10^{-4}$, batch size 256, and 60 epochs.

\paragraph{Datasets}
For indirect comparison purpose, we evaluate four benchmark datasets that are popular choices in feature unlearning:

\begin{table}[h]
\centering
\small
\setlength{\tabcolsep}{7pt}
\begin{tabular}{l l l l l l}
\toprule
\textbf{Dataset} & \textbf{Target $Y$} & \textbf{Sensitive $Z$} & \textbf{Dim($X$)} & \textbf{\#Samples}\\
\midrule
UCI Adult
& Income ($>\$50$k) 
& Gender
& 16
& 48842\\
COMPAS 
& Two-year recidivism
& Race
& 8
& 5300\\
UCI Bank
& Subscription
& Marital
& 51
& 45211\\
CelebA 
& Smiling
& Gender
& 512
& 202599\\
\bottomrule
\end{tabular}
\caption{Benchmarks used in the feature unlearning experiments. For each fold, we train on the fold's training partition (with an internal $20\%$ validation split, select thresholds/tradeoff knobs on validation, and report metrics on the fold's test partition. \emph{Feature dimension Dim($X$)} equals the post-encoding input dimensionality used by the classifier backbone (excluding $Z$; some methods concatenate $Z$ during training/evaluation as specified in the method section).}
\label{tab:data-sp}
\end{table}

\paragraph{Methods compared}
We include six representative approaches spanning pre-processing, post-processing, in-processing, representation learning, and two ERM baselines (include and exclude $Z$). Each exposes a single monotone knob (denoted $\lambda\!\in[0,1]$), so that we can trace a consistent Pareto frontier.

\begin{table}[h]
\centering
\small
\setlength{\tabcolsep}{6.5pt}
\begin{tabular}{l l l l l}
\toprule
\textbf{Method} & \textbf{Type} & \textbf{Tradeoff knob} & \textbf{Uses $Z$ (train/test)}\\
\midrule
Barycenter & Post-proc & $\lambda$ (toward barycenter) & no / \textbf{yes} \\
DIR (quantile repair) & Pre-proc & $\lambda$ (repair strength) & \textbf{yes} / no\\
MI (ours) & In-proc & $\lambda$ in $(1{-}\lambda)\,\mathrm{CE}+\lambda\,I(\hat Y;Z)$ & \textbf{yes} / \textbf{yes}\\
Zemel LFR & Representation & $\lambda$ (monotone map to $\lambda$) & \textbf{yes} / no \\
ERM\_deep$(X)$ & Baseline & N/A & no / no \\
ERM\_deep$([Z\mid X])$ & Baseline & N/A & \textbf{yes} / \textbf{yes} \\
\bottomrule
\end{tabular}
\caption{Comparison methods included in the feature unlearning experiment.}
\label{tab:methods-sp}
\end{table}

\paragraph{Policies and metrics}
Each method yields either a score $s_i\in[0,1]$ or a per-sample probability mass function $\pi_i{=}(p_{i,0},p_{i,1})$. To minimize the influence an additional step of generating hard label prediction, we directly evaluate via randomized policy which treats $\pi_i$ as a randomized classifier and compute:
\begin{equation*}
    \mathrm{acc}_\mathrm{rand} \;=\; \frac{1}{n}\sum_{i=1}^n p_{i,Y_i}, 
    \qquad
    \mathrm{dp\_gap} \;=\; \left|\,\mathbb{E}[p_{i,1}\!\mid Z{=}1]-\mathbb{E}[p_{i,1}\!\mid Z{=}0]\,\right|,
\end{equation*}
where $p_{i,1}$ is the probability of predicting class 1. AUROC is computed from the underlying scalar score $s_i$ or $p_{i,1}$ (e.g., softmax logit for ERM/MI/LFR, OT-mapped score for Barycenter).

\begin{figure}[t]
  \centering
  % --- (a) GPT-2 on HP ---
  \begin{subfigure}{0.98\linewidth}
    \centering
    \includegraphics[width=0.45\linewidth]{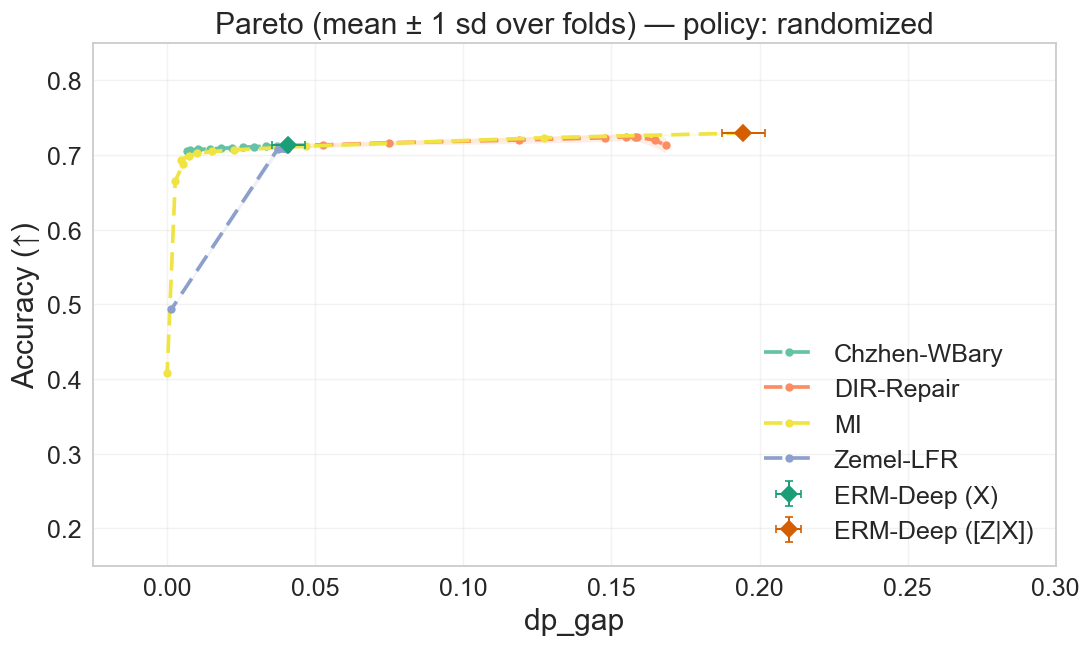}
    \includegraphics[width=0.45\linewidth]{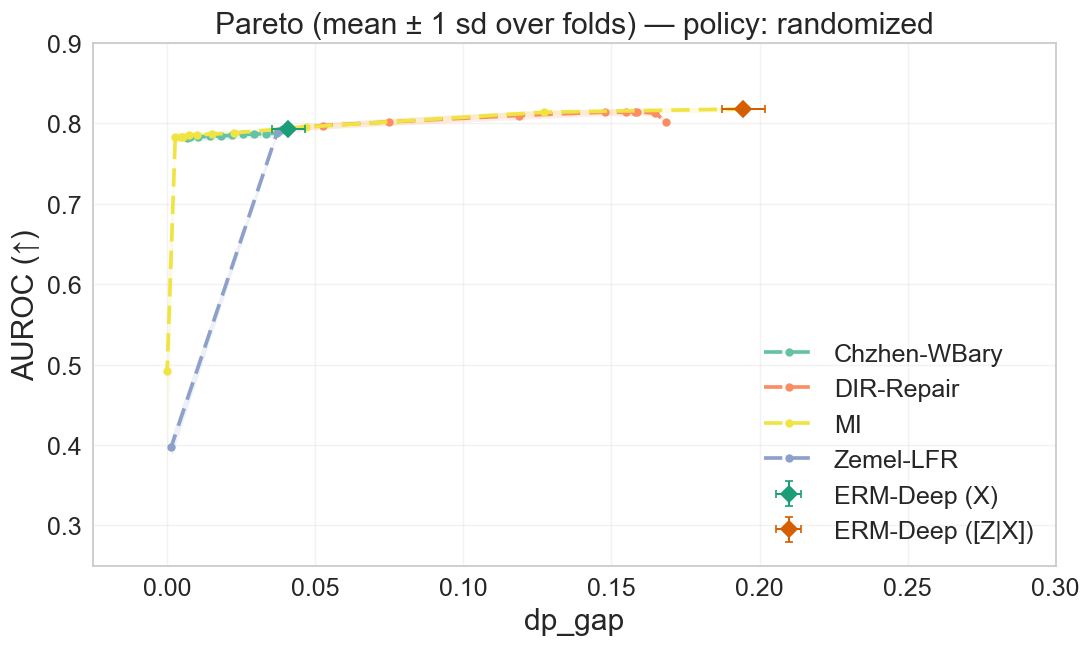}
    \caption{Results on UCI Adult}
    \label{fig:adult}
  \end{subfigure}
  \vspace{3pt}
  \begin{subfigure}{0.98\linewidth}
    \centering
    \includegraphics[width=0.45\linewidth]{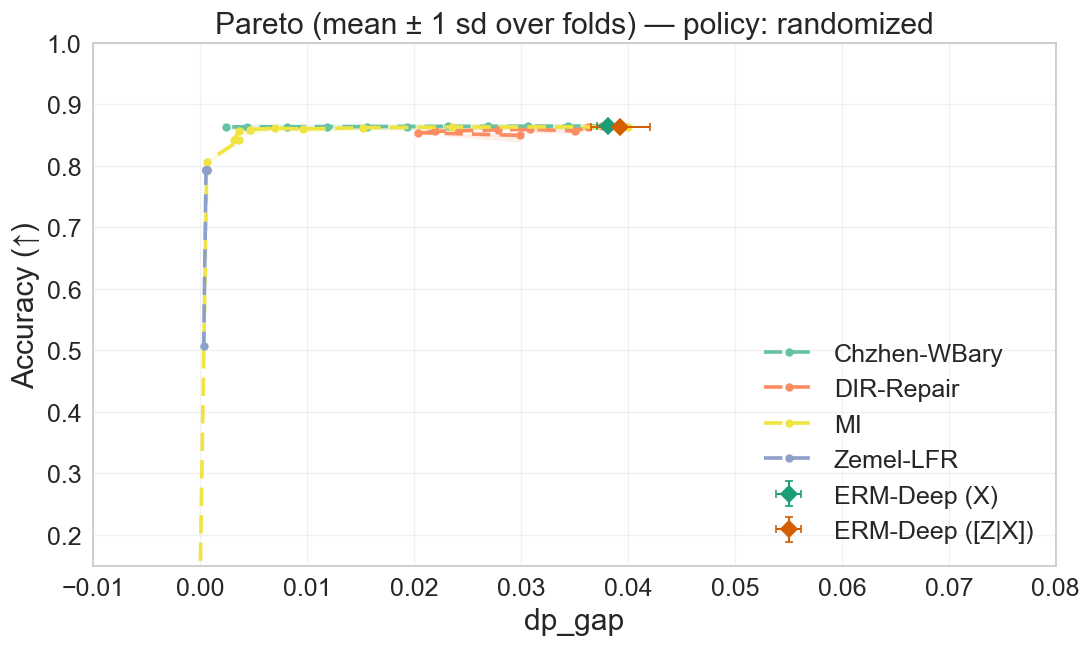}
    \includegraphics[width=0.45\linewidth]{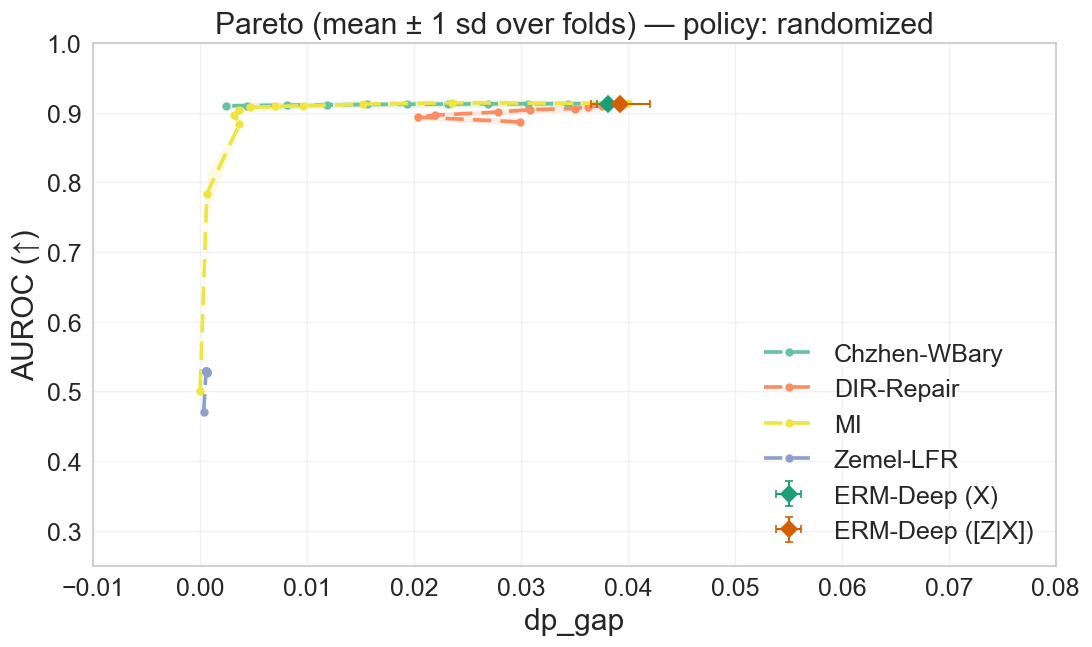}
    \caption{Comparison result on UCI Bank}
    \label{fig:adult}
  \end{subfigure}
  \begin{subfigure}{0.98\linewidth}
    \centering
    \includegraphics[width=0.45\linewidth]{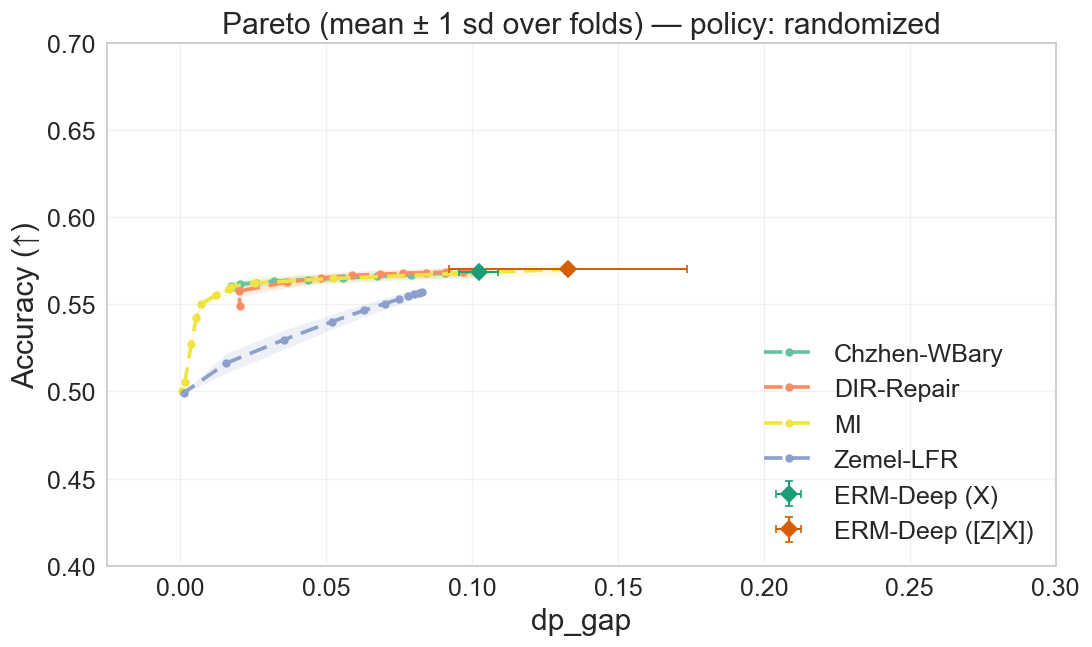}
    \includegraphics[width=0.45\linewidth]{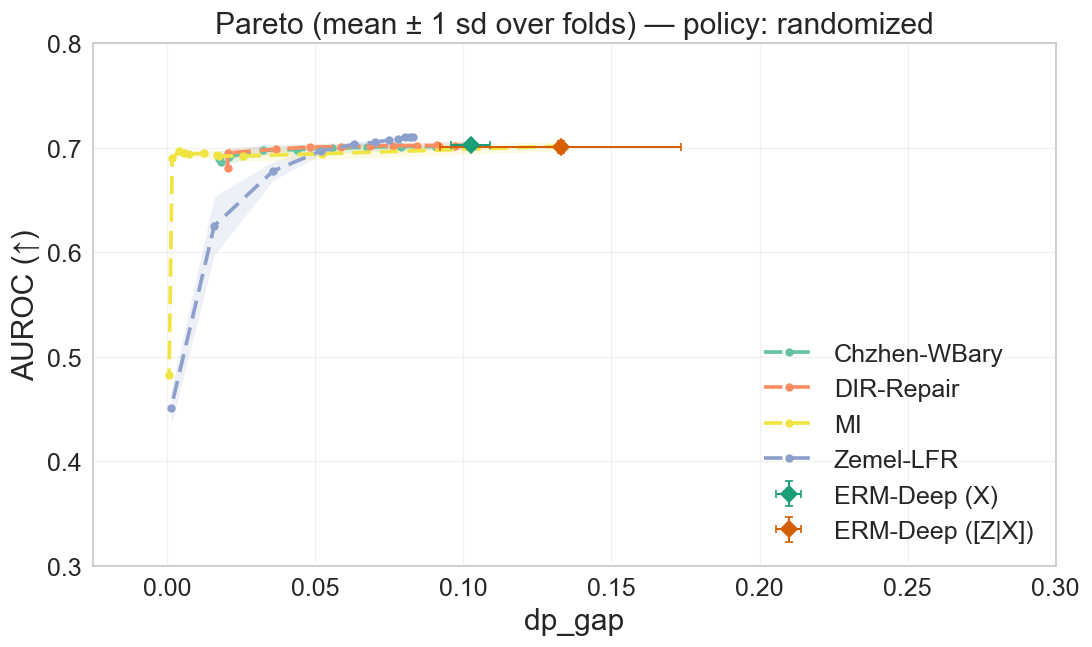}
    \caption{Comparison result on COMPAS}
    \label{fig:adult}
  \end{subfigure}
  \begin{subfigure}{0.98\linewidth}
    \centering
    \includegraphics[width=0.45\linewidth]{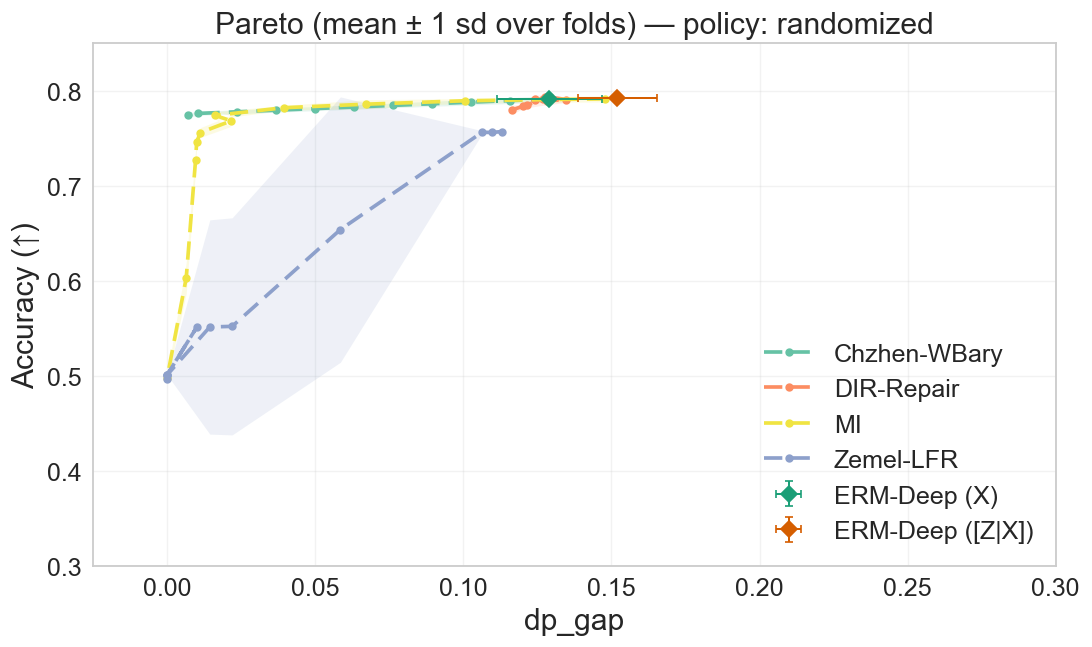}
    \includegraphics[width=0.45\linewidth]{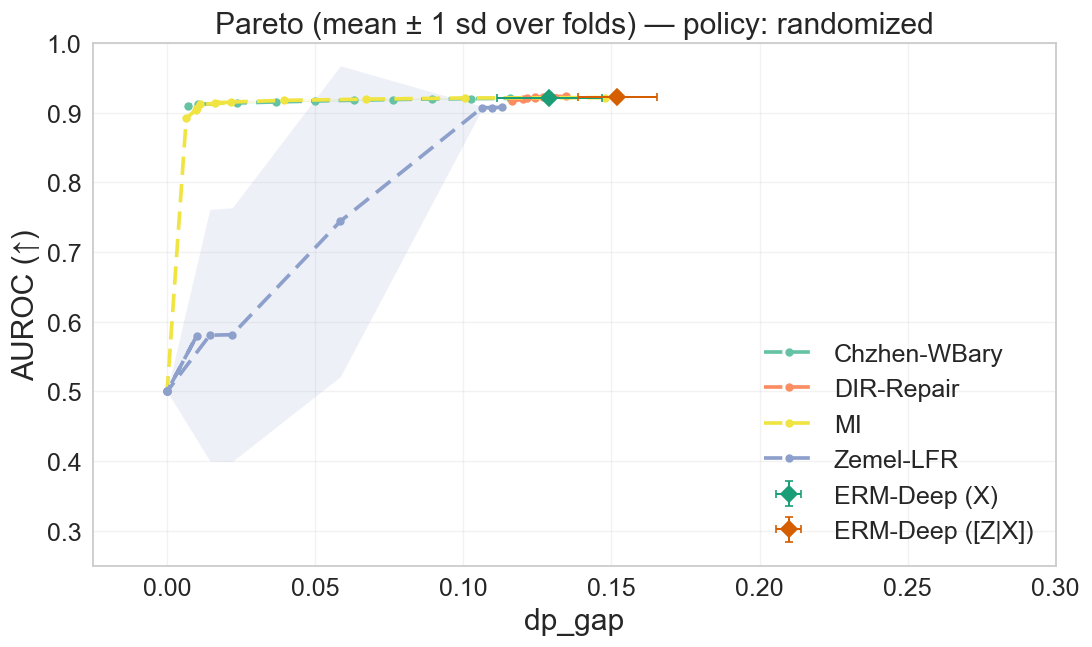}
    \caption{Comparison result on CelebA}
    \label{fig:adult}
  \end{subfigure}

  \caption{Feature–unlearning frontiers: Each row reports utility versus feature influence for a dataset: the left panel shows accuracy (↑) and the right panel shows AUROC (↑) against the demographic-parity gap (DP-gap, ↓), defined as $|\,\mathbb{P}[\hat Y{=}1\mid Z{=}1]-\mathbb{P}[\hat Y{=}1\mid Z{=}0]\,|$. Curves trace each method’s trade-off as its trade-off parameter varies. Points denote the mean and bands denote $\pm$1 s.d. over $5$ folds. Lower DP-gap at comparable or higher utility indicates a better frontier.}
  \label{fig:feature_unlearning_compare}
\end{figure}

\paragraph{Results.}
Figure~\ref{fig:feature_unlearning_compare} shows that the proposed mutual information regularization method (MI) is the best method in reducing the influence of $Z$ on learning outcome (quantified by dp-gap) while preserving utility accuracy and the area under curve on Adult and Compas datasets, and only slightly behind Chzhen's Wasserstein barycenter method (which only works for binary variable with significantly worse computational cost) on Bank and CelebA datasets. Furthermore, it offers the best smooth parameterization of the Pareto frontier, equivalently the optimal trade-off, between utility (quantified by accuracy and AUROC) and feature influence (quantified by dp-gap).

\subsection{Marginal Data-Point Unlearning via Regularization: Algorithm~\ref{alg:point_reg}}
\label{s6:data_unlearning_experiment}

\subsubsection{Forget Gaussian Experiment}
\label{ss6:data_gaussian}

This synthetic experiment illustrates the mechanism of marginal data-point
unlearning in a one-dimensional setting where the retain and unlearn
distributions have visibly different structure. The retain distribution is
nearly flat on an interval, while the unlearn distribution is concentrated
near the origin. Thus, the information to be forgotten is the Gaussian
concentration around \(0\), while the information to be preserved is the
approximately uniform retain density on \([-L,L]\).

\paragraph{Data and notation.}
We consider two distributions supported on the same interval \([-L,L]\), with
\(L=3\):
\[
    X_r\sim \mathrm{Unif}([-L,L]),
    \qquad
    X_u\sim \mathcal N(\mu,\sigma^2)
    \quad\text{truncated to }[-L,L].
\]
We draw \(N_R\) i.i.d.\ retain samples
\(\{x_i^{(r)}\}_{i=1}^{N_R}\) from \(X_r\), and \(N_U\) i.i.d.\ unlearn
samples \(\{x_j^{(u)}\}_{j=1}^{N_U}\) from \(X_u\). Let
\[
    \rho:=\frac{N_R}{N_R+N_U}
\]
denote the empirical retain fraction. The retain-plus-unlearn mixture is
therefore $p^d=\rho p^r+(1-\rho)p^u$. Equivalently, if the theoretical notation in Section~\ref{s3:problem_setting} uses \(p^d=(1-\alpha)p^r+\alpha p^u\), then here \(\rho=1-\alpha\).

We use a residual network \(f_\theta:\mathbb R\to\mathbb R\), initialized near
the identity, and first pretrain it so that the retain output density mimics
the retain-plus-unlearn mixture:
\[
    \widehat p_{f_{\theta_0}(X_r)}
    \approx
    \rho\,\widehat p_{X_r}
    +
    (1-\rho)\,\widehat p_{X_u}.
\]
This initialization models the situation in which the original model has
absorbed the unlearn-set signal into its output behavior.

All densities are estimated on a uniform grid
\(\mathcal G=\{x_k\}_{k=1}^{N_X}\) with spacing \(\Delta x\), using Gaussian
KDE with bandwidths \(h_x\) for inputs and \(h_y\) for outputs:
\[
\widehat p_{S,h}(x_k)
:=
\frac{1}{m}
\sum_{\ell=1}^m
\frac{
\exp\!\big(-\frac{(x_k-s_\ell)^2}{2h^2}\big)
}{
\sum_{j=1}^{N_X}
\exp\!\big(-\frac{(x_j-s_\ell)^2}{2h^2}\big)\Delta x
}.
\]
For grid densities \(P,Q\), we approximate
\[
    H(P\Vert Q):=-\sum_k P(x_k)\log Q(x_k)\,\Delta x,
\]
and
\[
    \KL(P\Vert Q)
    :=
    \sum_k P(x_k)\log\frac{P(x_k)}{Q(x_k)}\,\Delta x.
\]
For a binary auditing label \(Z\in\{0,1\}\) with priors \(p_0,p_1\) and
conditional output densities \(P_0,P_1\), the mutual information estimator is
\[
    I(S_{\mathrm{margin}};Z)
    =
    p_0\,\KL(P_0\Vert P)+p_1\,\KL(P_1\Vert P),
    \qquad
    P:=p_0P_0+p_1P_1.
\]

\paragraph{Compared objectives.}
We compare three objectives. Unless otherwise stated, the auditing prior is
balanced, \(p_0=p_1=\frac12\).

\begin{itemize}[leftmargin=*]
\item \textbf{Marginal-MI (ours).}
The marginal objective compares the retain output law against the
retain-plus-unlearn mixture output law. Define
\[
    P_1:=\widehat p_{f_\theta(X_r)},
    \qquad
    P_0:=
    \rho\,\widehat p_{f_\theta(X_r)}
    +(1-\rho)\,\widehat p_{f_\theta(X_u)} .
\]
The loss is
\[
\mathcal L_{\textsc{marginal}}(\theta)
=
(1-\lambda)\,
H\!\big(
\widehat p_{X_r}\Vert \widehat p_{f_\theta(X_r)}
\big)
+
\lambda\, I(S_{\mathrm{margin}};Z).
\]
Minimizing the MI term drives \(P_0=P_1\). Since $P_0=\rho P_1+(1-\rho)\widehat p_{f_\theta(X_u)}$, this forces
\[
    \widehat p_{f_\theta(X_u)}
    =
    \widehat p_{f_\theta(X_r)}
\]
whenever \(1-\rho>0\). The utility term simultaneously anchors
\(\widehat p_{f_\theta(X_r)}\) to the retain density \(\widehat p_{X_r}\).

\item \textbf{Grad-Diff.}
This baseline replaces the MI term by a negative cross-entropy term that pushes
the unlearn output away from the original unlearn density:
\[
\mathcal L_{\textsc{grad\_diff}}(\theta)
=
(1-\lambda)\,
H\!\big(
\widehat p_{X_r}\Vert \widehat p_{f_\theta(X_r)}
\big)
-
\lambda\,
H\!\big(
\widehat p_{X_u}\Vert \widehat p_{f_\theta(X_u)}
\big).
\]
This discourages resemblance to the unlearn density, but does not prescribe
where the unlearn mass should move.

\item \textbf{KL+utility.}
This baseline anchors the retain output to the pretrained model while applying
the same negative unlearn-density term:
\[
\mathcal L_{\textsc{KL}}(\theta)
=
(1-\lambda)\,
\KL\!\big(
\widehat p_{f_{\theta_0}(X_r)}
\Vert
\widehat p_{f_\theta(X_r)}
\big)
-
\lambda\,
H\!\big(
\widehat p_{X_u}\Vert \widehat p_{f_\theta(X_u)}
\big).
\]
The first term keeps \(f_\theta(X_r)\) close to the pretrained retain output,
whereas the second term discourages resemblance to \(X_u\).
\end{itemize}

\begin{figure*}[t]
  \centering
  \begin{minipage}{0.4\textwidth}
    \includegraphics[width=\linewidth]{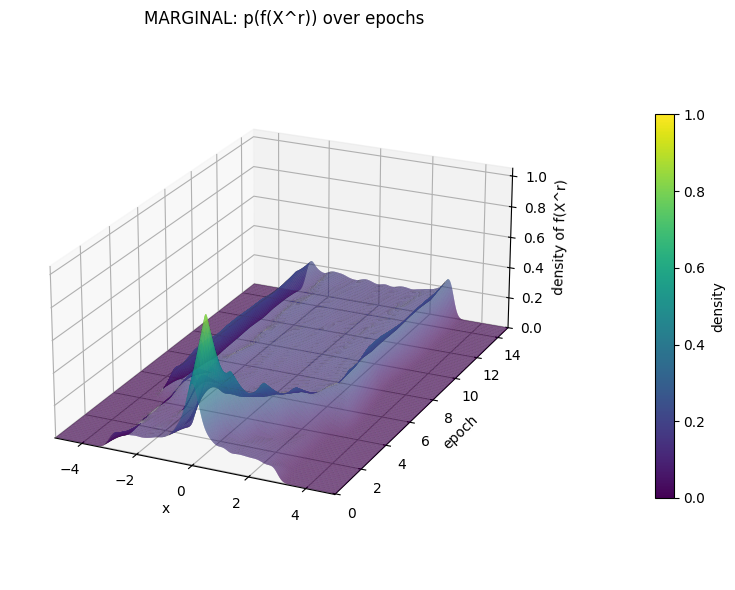}
  \end{minipage}
  \begin{minipage}{0.4\textwidth}
    \includegraphics[width=\linewidth]{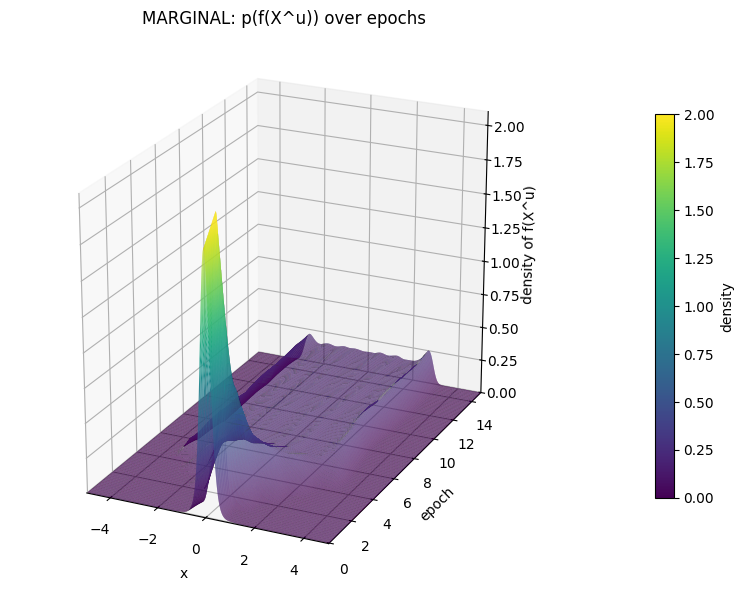}
  \end{minipage}\\
  \begin{minipage}{0.4\textwidth}
    \includegraphics[width=\linewidth]{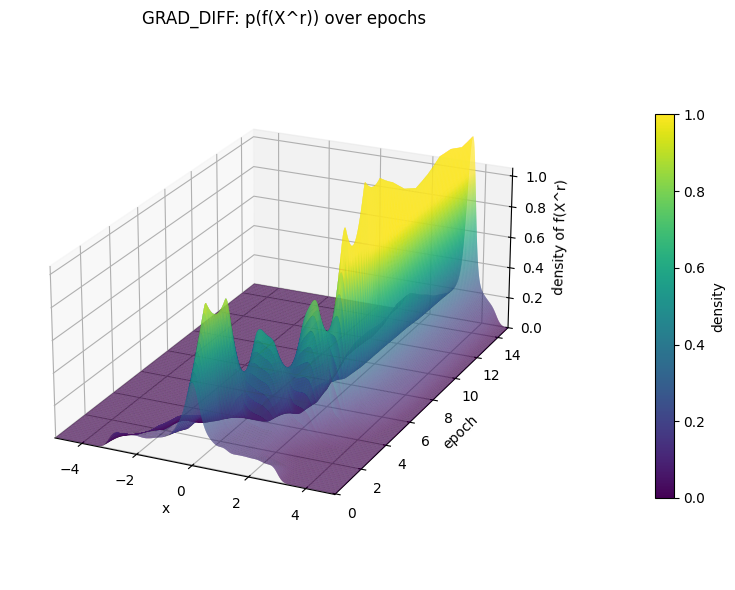}
  \end{minipage}
  \begin{minipage}{0.4\textwidth}
    \includegraphics[width=\linewidth]{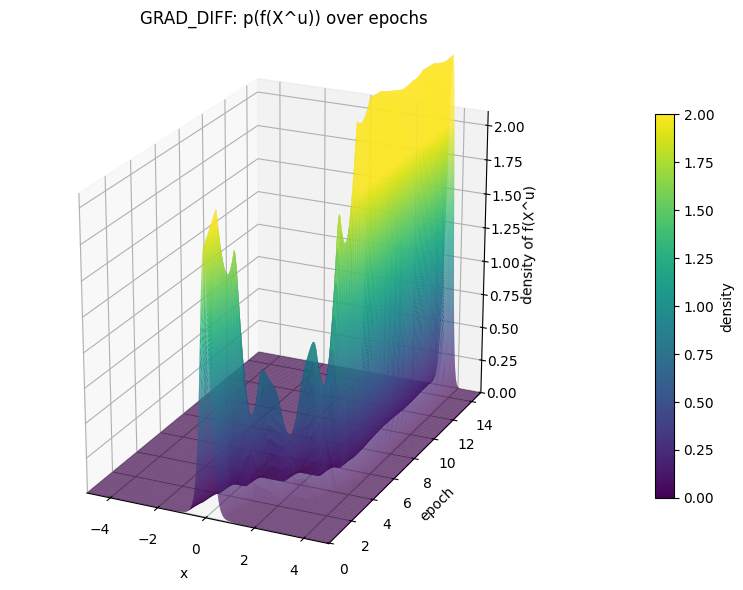}
  \end{minipage}\\
  \begin{minipage}{0.4\textwidth}
    \includegraphics[width=\linewidth]{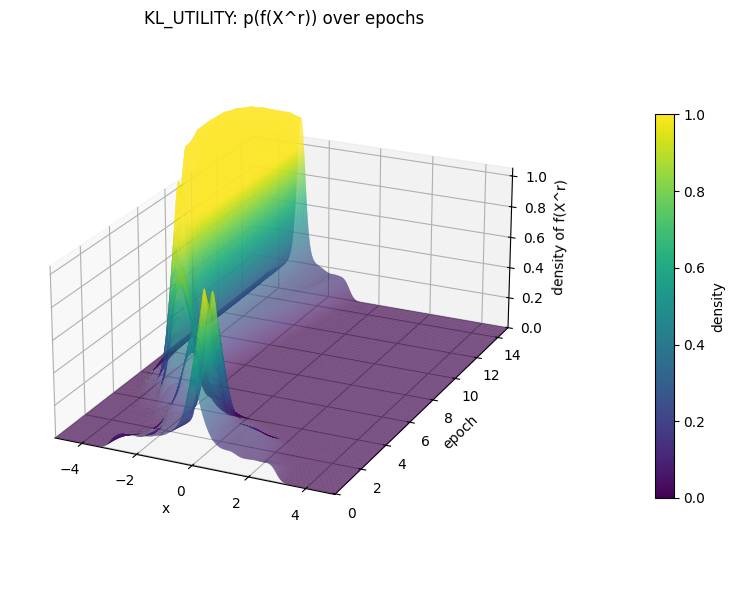}
  \end{minipage}
  \begin{minipage}{0.4\textwidth}
    \includegraphics[width=\linewidth]{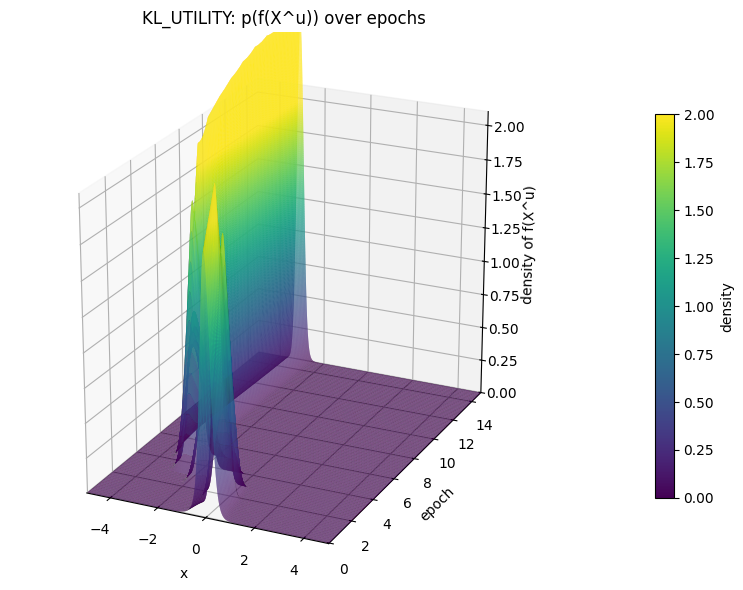}
  \end{minipage}
  \caption{Evolution of the output densities \(p(f_\theta(X_r))\) and
  \(p(f_\theta(X_u))\) over epochs for the three objectives. Left column:
  retain outputs. Right column: unlearn outputs. Marginal-MI suppresses the
  unlearn-specific concentration near zero while keeping the retain output
  close to the uniform retain density. In contrast, the gradient-ascent-based
  baselines tend to relocate the unlearn mass toward other regions, such as
  the boundaries, rather than matching it to the retain distribution.}
  \label{fig:forget-gaussian}
\end{figure*}

\paragraph{Outcomes.}
Figure~\ref{fig:forget-gaussian} shows how the retain and unlearn output
densities evolve from the pretrained model. Under \textsc{marginal}, the MI
penalty equalizes the retain-only and retain-plus-unlearn output conditionals.
Algebraically, this drives
\[
    p(f_\theta(X_u))\approx p(f_\theta(X_r)),
\]
while the utility term keeps
\[
    p(f_\theta(X_r))\approx p(X_r).
\]
Thus, the unlearn-specific Gaussian concentration is removed by matching it to
the retain output distribution, rather than by pushing it to an arbitrary
location.

In contrast, \textsc{grad\_diff} lacks a target distribution for the unlearn
mass. Its negative cross-entropy term can move the unlearn density toward the
boundaries or other regions, and shared parameters may also distort the retain
output. The \textsc{KL} baseline stabilizes the retain trajectory by anchoring
to the pretrained output, but its unlearning term still does not specify where
the unlearn mass should go. Consequently, it can also push the forget mass
toward the edges.

\paragraph{Summary.}
The Gaussian experiment illustrates the qualitative difference between
marginal unlearning and gradient-ascent-style forgetting. Marginal-MI removes
the unlearn-specific signal by matching the unlearn output law to the retain
output law, while retaining the uniform structure through the utility term.
The baselines can reduce resemblance to the unlearn density, but they do so
without a principled replacement distribution for the removed mass.

\subsubsection{Forget MNIST}\label{ss6:data_MNIST}

\paragraph{Dataset and retain/unlearn split.}
We use \textsc{MNIST} with the standard \(60{,}000/10{,}000\) train/test split
and normalize images using mean \(0.1307\) and standard deviation \(0.3081\).
Let \(U_{\mathrm{all}}\) denote the set of training images whose label is digit
\(3\). We randomly select \(99.5\%\) of \(U_{\mathrm{all}}\) as the unlearn set
\(U\), and define the retain set as
\[
    R:=\mathrm{train}\setminus U .
\]
Thus, the retain set contains almost no digit-\(3\) examples, while the
FineTune-All model is trained on the full original training set.

For evaluation and trajectory aggregation, we form five random folds
independently on \(R\) and on \(U\). For each fold \(i\), this gives
\[
    (R_i^{\mathrm{tr}},R_i^{\mathrm{val}})
    \qquad\text{and}\qquad
    (U_i^{\mathrm{tr}},U_i^{\mathrm{val}}).
\]
These folds are used to train and validate the unlearning trajectories. Test
accuracy is computed on the original MNIST test set. Unless otherwise stated,
the nominal retain and unlearn batch sizes are \(128\), and the test batch size
is \(512\); on MPS devices, the implementation reduces the retain and unlearn
batch sizes to \(96\) for stability.

\paragraph{Shared backbone and optimization for comparability.}
We use a lightweight CNN for MNIST ($28{\times}28$ grayscale) to ensure strict cross-method comparability. All convolutions use $3{\times}3$ kernels, stride $1$, and padding $1$ (“same”), so spatial resolution is preserved between pooling layers. The architecture is:
\begin{align*}
& \underbrace{\mathrm{Conv}(32)\!\to\!\mathrm{ReLU}\!\to\!\mathrm{Conv}(64)\!\to\!\mathrm{ReLU}\!\to\!\mathrm{MaxPool}(2)\!\to\!\mathrm{Dropout}(0.25)}_{\textbf{Block 1}}\\
\;\to\;& 
\underbrace{\mathrm{Conv}(128)\!\to\!\mathrm{ReLU}\!\to\!\mathrm{MaxPool}(2)\!\to\!\mathrm{Dropout}(0.25)}_{\textbf{Block 2}}\\
\;\to\;&
\underbrace{\mathrm{Flatten}(128{\times}7{\times}7)\!\to\!\mathrm{FC}(256)\!\to\!\mathrm{ReLU}\!\to\!\mathrm{Dropout}(0.5)\!\to\!\mathrm{FC}(10)}_{\textbf{Head}}.
\end{align*}
Max-pooling layers are $2{\times}2$ (stride $2$), yielding spatial sizes $28{\times}28 \!\to\! 14{\times}14 \!\to\! 7{\times}7$. ReLU follows every convolution and the first fully connected layer; dropout is applied after each pooling block and before the classifier head.

We use Adam with learning rate \(10^{-3}\) and weight decay \(10^{-4}\), and
fix the random seed to \(1337\). FineTune-All and Retrain-on-Retain each run
for \(10\) epochs. Unlearning runs for at most \(30\) epochs and may stop early
according to the rules below.

\paragraph{Baselines and compared methods.}
We compare five models.

\begin{itemize}[leftmargin=*]
    \item \textbf{FineTune-All (FT).} Train the backbone on the full MNIST
    training set for \(10\) epochs. This is the initial model that has seen the
    unlearn data.

    \item \textbf{Retrain-on-Retain (RT).} Train the same backbone from scratch
    on \(R\) only for \(10\) epochs. This is the anchor model that never trains
    on the selected unlearn set \(U\).

    \item \textbf{Marginal-MI (ours).} Starting from the FT weights, minimize
    \[
    \mathcal L_{\mathrm{MI}}
    =
    (1-\gamma)\,
    \mathrm{CE}\big(f_\theta(x),y;\ x\in R_i^{\mathrm{tr}}\big)
    +
    \gamma\,\widehat{\mathrm{MI}}_{\mathrm{soft}}.
    \]
    The estimator \(\widehat{\mathrm{MI}}_{\mathrm{soft}}\) is a
    softmax-marginal plug-in proxy for \(I(S_{\mathrm{margin}};Z)\). Let
    \[
        P_r
        :=
        \mathbb E_{x\in R_i^{\mathrm{tr}}}
        [\hat p_\theta(\cdot\mid x)],
        \qquad
        P_u
        :=
        \mathbb E_{x\in U_i^{\mathrm{tr}}}
        [\hat p_\theta(\cdot\mid x)].
    \]
    With $\rho
        :=
        \frac{|R_i^{\mathrm{tr}}|}
        {|R_i^{\mathrm{tr}}|+|U_i^{\mathrm{tr}}|}$, the retain-plus-unlearn mixture output is $P_d:=\rho P_r+(1-\rho)P_u$. Using balanced auditing priors, the estimator is
    \[
        \widehat{\mathrm{MI}}_{\mathrm{soft}}
        =
        \frac12\KL(P_r\Vert P)
        +
        \frac12\KL(P_d\Vert P),
        \qquad
        P:=\frac12P_r+\frac12P_d.
    \]
    This is not the full empirical output-law mutual information. Instead, it
    is a stable low-dimensional proxy based on average class-probability
    vectors.

    \item \textbf{Grad-Diff.} Starting from the FT weights, optimize
    \[
    \mathcal L_{\mathrm{GD}}
    =
    (1-\gamma)\,
    \mathrm{CE}\big(f_\theta(x),y;\ x\in R_i^{\mathrm{tr}}\big)
    -
    \gamma\,
    \mathrm{CE}\big(f_\theta(x),y;\ x\in U_i^{\mathrm{tr}}\big).
    \]

    \item \textbf{KL+CE.} With the FT model frozen as teacher
    \(f_{\theta^\star}\), optimize from the FT weights
    \[
    \mathcal L_{\mathrm{KL}}
    =
    (1-\gamma)
    \mathbb E_{x\in R_i^{\mathrm{tr}}}
    \!\left[
    \KL\big(
    \hat p_{\theta^\star}(\cdot\mid x)
    \Vert
    \hat p_\theta(\cdot\mid x)
    \big)
    \right]
    -
    \gamma\,
    \mathrm{CE}\big(f_\theta(x),y;\ x\in U_i^{\mathrm{tr}}\big).
    \]
\end{itemize}

\paragraph{Trade-off grids and initialization.}
Every unlearning method is initialized from the same FT checkpoint. We sweep
method-specific \(\gamma\)-grids:
\[
\gamma_{\mathrm{MI}}\in\{0.0020,\,0.0055,\,0.0100\},
\]
\[
\gamma_{\mathrm{KL}}\in\{0.0003,\,0.0006,\,0.0012\},
\qquad
\gamma_{\mathrm{GD}}\in\{0.0002,\,0.00035,\,0.0007\}.
\]
The grids are chosen so that each method ranges from mild unlearning to a
substantial decrease in unlearn-set accuracy.

\paragraph{Early stopping.}
We use early stopping only for Marginal-MI and Grad-Diff. For KL+CE, the
student is initialized from the same FT checkpoint as the frozen teacher, so
the initial teacher--student retain-validation KL is numerically zero. A
relative KL stopping rule based on this initial value is therefore inactive in
this experiment, and KL+CE is run for the full unlearning budget of \(30\)
epochs.

\begin{itemize}[leftmargin=*]
\item \textbf{Marginal-MI:} compute
\(\mathrm{MI}^{\mathrm{val}}_0\) before unlearning on
\((R_i^{\mathrm{val}},U_i^{\mathrm{val}})\), and stop when
\[
    \mathrm{MI}^{\mathrm{val}}
    \le
    0.85\,\mathrm{MI}^{\mathrm{val}}_0
\]
with minimum epoch \(1\) and patience \(1\).

\item \textbf{Grad-Diff:} infer the number of classes \(C\) from the logits
(\(C=10\) for MNIST), and stop when the unlearn-validation accuracy satisfies
\[
    \mathrm{Acc}_{U}^{\mathrm{val}}
    \le
    \frac1C+0.02
\]
for two consecutive checks, with minimum epoch \(1\) and patience \(2\).
\end{itemize}

\paragraph{Results.}

\begin{figure}[t]
  \centering
  % --- (a) GPT-2 on HP ---
  \begin{subfigure}{0.98\linewidth}
    \centering
    \includegraphics[width=0.85\linewidth]{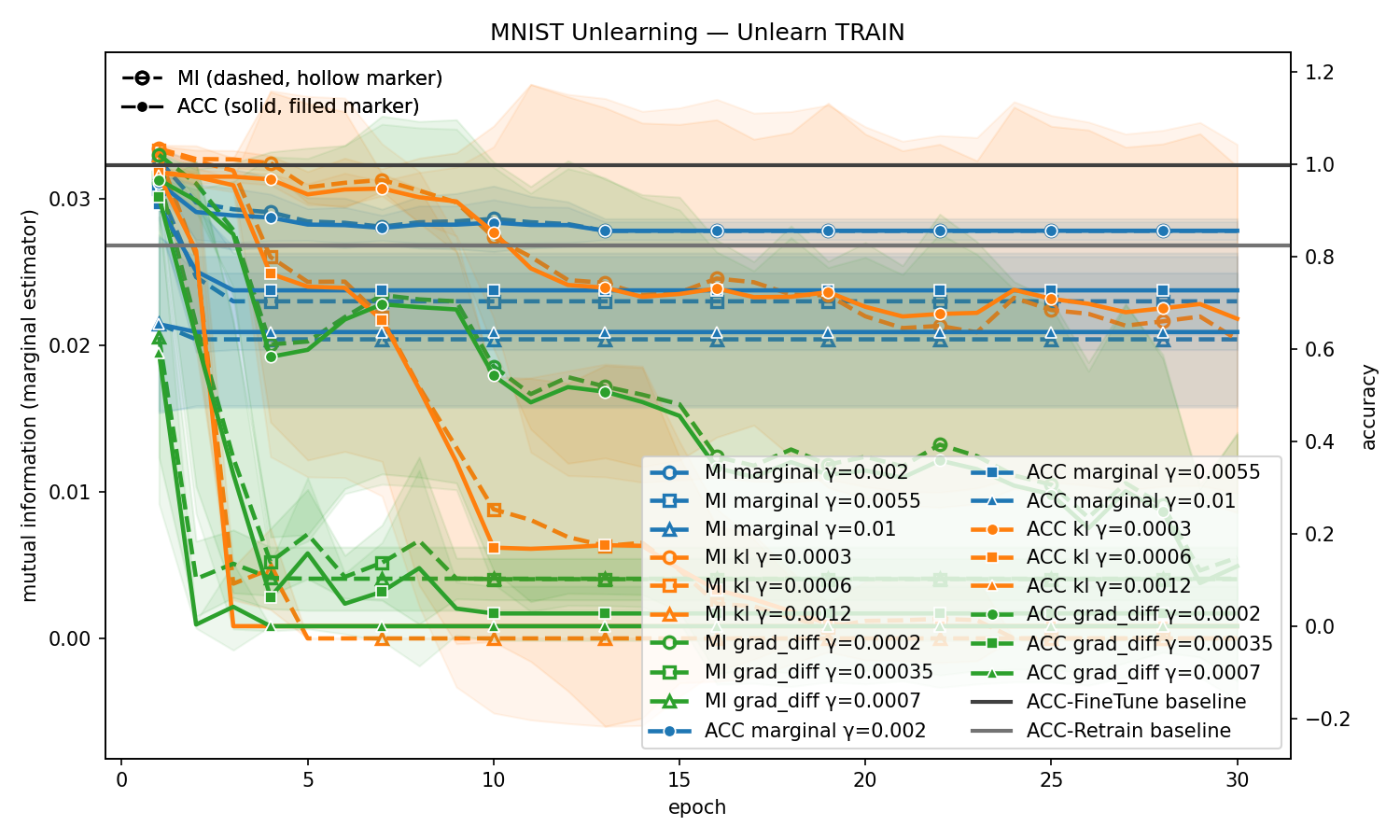}
    \caption{Unlearn set performance trajectories}
    \label{fig:MNIST_unlearn}
  \end{subfigure}
  \vspace{3pt}
  \begin{subfigure}{0.98\linewidth}
    \centering
    \includegraphics[width=0.85\linewidth]{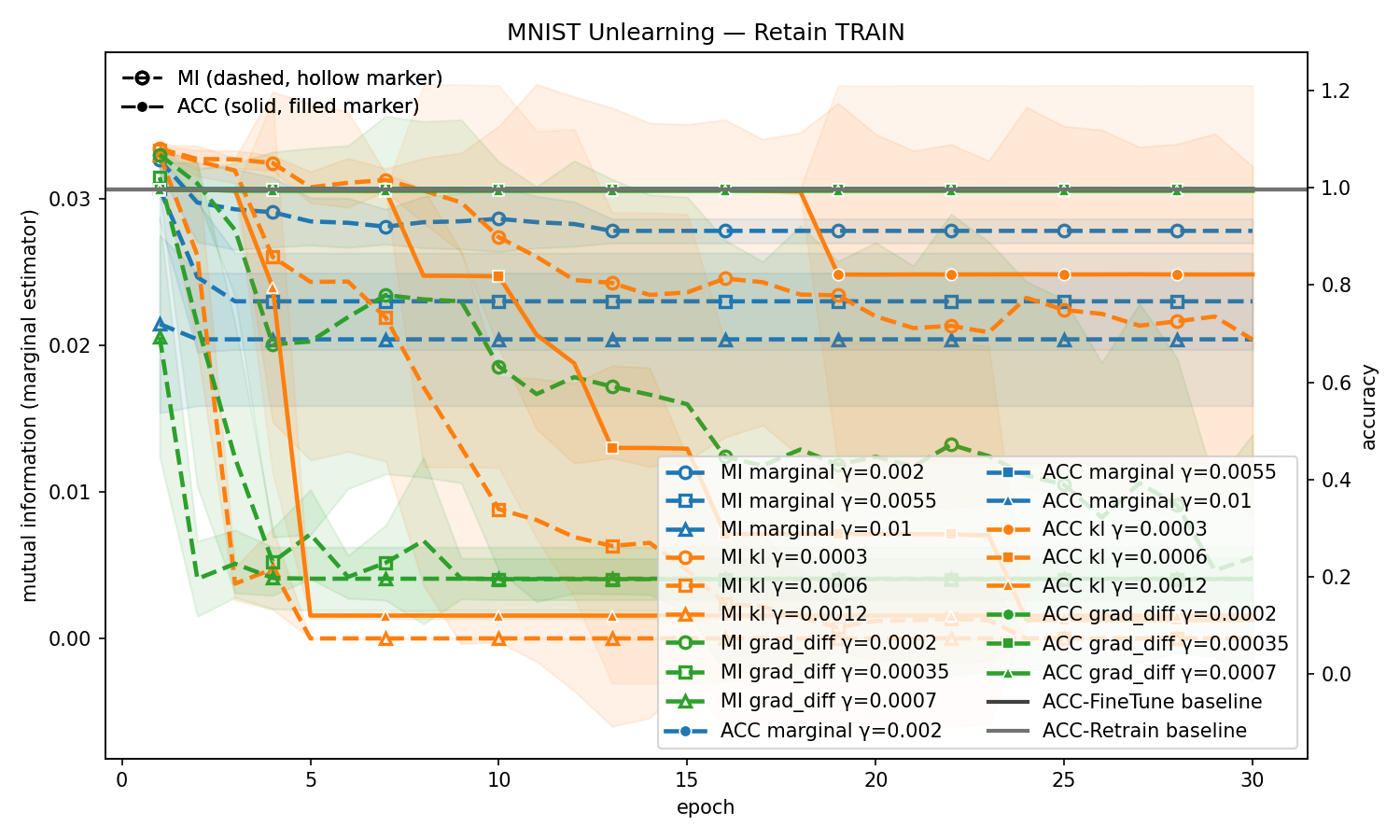}
    \caption{Retain set performance trajectories}
    \label{fig:MNIST_train}
  \end{subfigure}
  \caption{MNIST marginal data-unlearning trajectories on the training folds. The top panel reports unlearn-train accuracy together with the softmax-marginal MI proxy, and the bottom panel reports retain-train accuracy together with the same MI proxy. Solid filled markers show accuracy, while dashed hollow markers show MI. Lines denote means and shaded regions denote \(\pm1\) standard deviation over five folds. Horizontal bands show the FineTune-All and Retrain-on-Retain accuracy references.}
  \label{fig:MNIST_unlearning_compare}
\end{figure}

Figure~\ref{fig:MNIST_unlearning_compare} records the training-fold accuracy
on \(R_i^{\mathrm{tr}}\) and \(U_i^{\mathrm{tr}}\), together with the
softmax-marginal MI proxy, over the unlearning trajectory. Curves are
aggregated over five folds using mean \(\pm\) standard deviation. The
horizontal bands indicate the FineTune-All and Retrain-on-Retain accuracy
references.

The experiment shows two main trends:
\begin{itemize}[leftmargin=*]
    \item \textit{The softmax-marginal MI proxy tracks the unlearning
    trajectory.}
    For Marginal-MI, decreasing the MI proxy moves the unlearn-set behavior
    away from the FineTune-All baseline and toward the Retrain-on-Retain
    reference. This is consistent with the theory: reducing
    \(I(S_{\mathrm{margin}};Z)\) reduces the distinguishability between the
    retain-only output law and the retain-plus-unlearn output law. Since digit
    \(3\) is almost entirely removed from the retain set, the retrain-on-retain
    model provides the natural reference behavior for the unlearn subset.

    \item \textit{Marginal-MI most stably approaches the retrain-on-retain
reference in this experiment.}
Compared with the gradient-ascent-based baselines, Marginal-MI stays closest
to the Retrain-on-Retain reference across the tested trade-off parameters and
training trajectories. This stability can be seen from four perspectives:
(1) robustness to the trade-off parameter, (2) robustness across folds,
(3) stability of the training trajectory, and (4) simultaneous retain- and
unlearn-set behavior.

From the \emph{parameter-robustness} perspective, different choices of
\(\gamma\) for Marginal-MI lead to only moderate changes in the unlearn-set
trajectory, and the resulting curves remain near the Retrain-on-Retain
reference. In contrast, the Grad-Diff and KL+CE baselines are more sensitive to
\(\gamma\): depending on the value of the trade-off parameter, they may either
remain too close to the FineTune-All model or move too aggressively toward
chance-level behavior.

From the \emph{cross-validation robustness} perspective, Marginal-MI exhibits
smaller fold-to-fold variation in the plotted trajectories. The shaded
regions in Figure~\ref{fig:MNIST_unlearning_compare} are visibly narrower for
Marginal-MI than for the baselines, indicating more stable behavior across the
five folds. If the corresponding table is included, this observation can be
quantified by reporting the standard deviations of retain and unlearn
accuracies at the stopping epoch.

From the \emph{training-trajectory} perspective, Marginal-MI produces a smooth
and stable movement of the unlearn-set behavior toward the Retrain-on-Retain
reference, while retaining high performance on the retain set. By contrast,
Grad-Diff and KL+CE display larger fluctuations and sharper drops for some
trade-off choices, reflecting the instability caused by directly maximizing an
unlearn-set loss.

From the \emph{retain/unlearn accuracy} perspective, Marginal-MI better
balances the two desiderata: it moves the unlearn-set behavior toward the
Retrain-on-Retain reference while preserving retain-set accuracy. The
gradient-ascent-based baselines are more vulnerable to under- or
over-unlearning: they can either fail to move far enough from the FineTune-All
behavior, or push the unlearn-set accuracy too far toward chance while also
causing larger changes in retain-set performance.

These four empirical observations are consistent with the objective-level
interpretation of marginal unlearning. Marginal-MI directly penalizes the
source-label distinguishability \(I(S_{\mathrm{margin}};Z)\), while its
cross-entropy term preserves retain-set utility. Thus, the two terms are
aligned with the marginal-unlearning goal: preserve the retain behavior while
making the retain-only and retain-plus-unlearn output laws difficult to
distinguish. In contrast, Grad-Diff and KL+CE use a gradient-ascent term on
the unlearn set. This term indirectly seeks forgetting by pushing against the
unlearn labels, but it does not specify the target behavior that should replace
the removed information. As a result, the utility-preservation term and the
unlearn-set gradient-ascent term can act like competing forces, producing
greater sensitivity to \(\gamma\), fold variation, and training dynamics.
\end{itemize}

\subsection{Optimal Feature Unlearning via Analytic Solution: Algorithm \ref{alg:bary}}\label{s6:analytic_feature_experiment}

\paragraph{Dataset.}
We use the CelebFaces Attributes dataset (CelebA)~\citep{liu2015faceattributes}, a large-scale collection of 200K celebrity face images annotated with $40$ binary attributes (e.g., \texttt{Smiling}, \texttt{Female}). We focus on two binary features, \emph{Smile} and \emph{Gender}, and consider the standard image resolution pre-processing used in face-generation benchmarks.

\paragraph{Problem setup.}
Let $Z\in\{0,1\}$ denote the attribute to unlearn (e.g., $Z{=}1$ for ``smiling'', $Z{=}0$ for ``non-smiling''). Denote by $\mu_0$ and $\mu_1$ the empirical image distributions conditioned on $Z{=}0$ and $Z{=}1$, respectively. Our goal is to transform each image $(x,z)$ to an \emph{attribute-neutral} counterpart $\bar{x}$ whose distribution $\bar{\mu}$ is (i) minimally distorted from the originals and (ii) uninformative about $Z$.

For a complete feature-unlearning map, one applies slice-wise maps from each
conditional law to the common barycenter:
\[
R_0(x)=\frac12 x+\frac12 T_{0\to1}(x),\qquad
R_1(y)=\frac12 y+\frac12 T_{1\to0}(y),
\]
so that ideally \(R_{0\#}\mu_0=R_{1\#}\mu_1=\bar\mu\). In the visualization
below, we display one learned transport direction and its midpoint
interpolation.

\paragraph{Barycentric characterization (theory).}
Theorem~\ref{th:optimal_solution} shows that the $2$-Wasserstein barycenter of $\mu_0$ and $\mu_1$ with equal weights characterizes the optimal $Z$-neutral distribution:
\[
\bar{\mu}
\;\in\;
\argmin_{\mu}\;\tfrac{1}{2}W_2^2(\mu,\mu_0)+\tfrac{1}{2}W_2^2(\mu,\mu_1).
\]
For two measures with equal weights, the unique minimizer equals the midpoint of the $W_2$ geodesic connecting $\mu_0$ and $\mu_1$~\citep{mccann1997convexity}. Let $T_{0\to1}$ be the Brenier map pushing $\mu_0$ to $\mu_1$ (i.e., $T_{0\to1}{=}\nabla\psi$ for a convex potential $\psi$). The McCann \emph{displacement interpolation} induces a geodesic
\[
\mu_t \;=\; \big((1{-}t)\,{\rm Id} + t\,T_{0\to1}\big)_{\#}\mu_0,\qquad t\in[0,1],
\]
and the equal-weight barycenter is precisely the midpoint $\bar{\mu}=\mu_{1/2}$. Intuitively, projecting each sample to the geodesic midpoint makes the two attribute-conditional populations equidistant in $W_2$, attenuating attribute information while minimizing transport cost.

\paragraph{Neural OT implementation (practice).}
Direct computation of $T_{0\to1}$ in image space is intractable. We therefore use a neural OT approach in the spirit of~\citet{korotin2022neural}: a \emph{critic} network approximates the Kantorovich potential (dual), and a \emph{generator} parameterizes the transport map. Concretely, we learn a forward map $T_{0\to1}$ using adversarial objectives that enforce the dual constraints and pushforward consistency. Given a source image $x$ with attribute $z$, we produce its barycentric counterpart by a midpoint displacement step along the learned map for its group:
\[
\bar{x} \;= \frac{1}{2}\,x + \frac{1}{2}\,T_{0\to1}(x)
\]
which realizes the McCann interpolation at $t{=}0.5$.

\paragraph{Results.}
Figure~\ref{plot:unlearn_smile} illustrates the unlearning outcomes for two
target attributes: smile and gender. In each panel, the first row \(X\) shows
source images from one attribute-conditioned group. The last row \(T(X)\)
shows the corresponding push-forward images under the learned OT map. The
middle row,
\[
    \mathrm{Bary}(X):=\frac12 X+\frac12 T(X),
\]
shows the midpoint displacement interpolation, which corresponds to the
two-marginal barycenter in the ideal population setting.

In theory, \(\mathrm{Bary}(X)\) approximates the optimal feature-unlearned
representation. The target attribute is visually attenuated, while other image
features, such as lighting, pose, and idiosyncratic facial traits, are largely
preserved.

\begin{figure}[H]
    \centering
    \vspace{-3mm}  % Reduce space above the figure
    \includegraphics[width=\columnwidth]{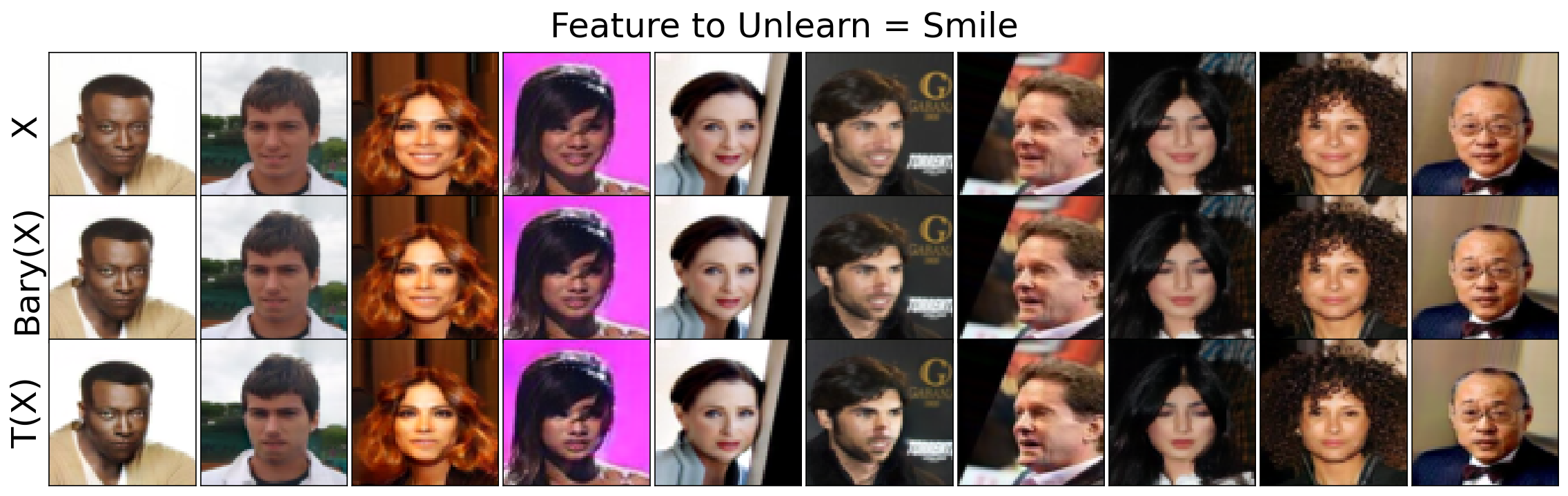}
    \includegraphics[width=\columnwidth]{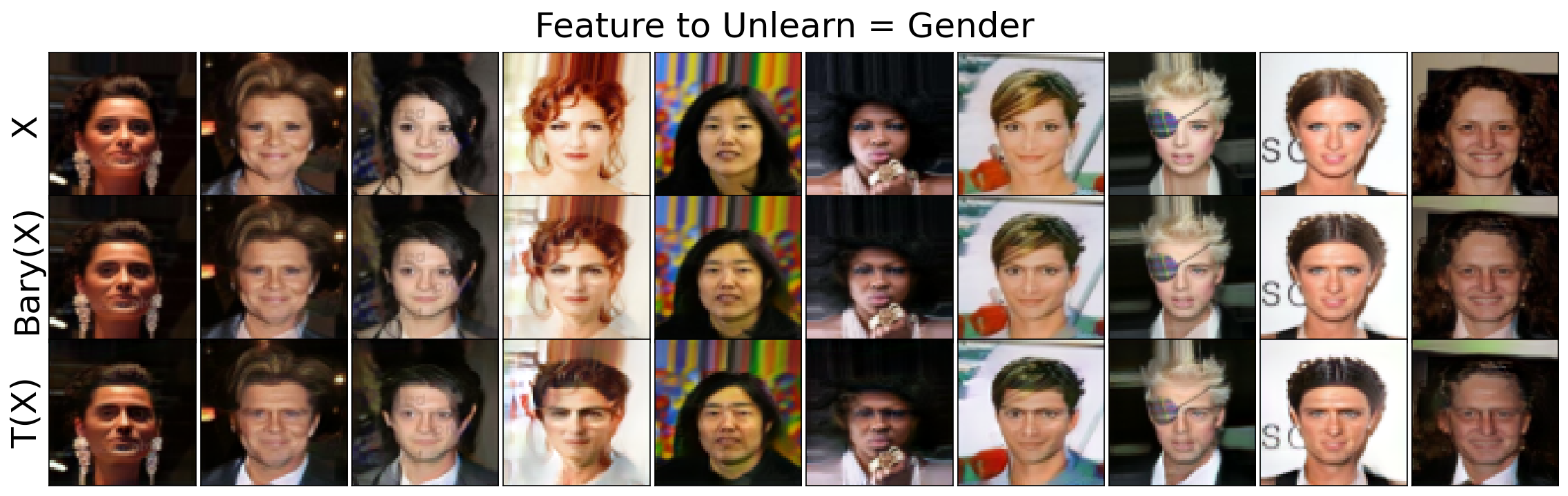}
    \vspace{-5mm}  % Reduce space above the figure
    \caption{illustrates the unlearning outcomes for both tasks: The above penal shows the unlearn results for smile feature and the bottom shows the results for gender feature. In each of the panel, the first row, denoted by $X$, represents samples from the original data set with the chosen feature value (smile and female, respectively). The last row, denoted by $T(X)$ represents the push-forward image of the corresponding sample in $X$ by the generated optimal transport map $T$. The middle row, denoted by $\text{Bary}(X) := [0.5Id + 0.5T](x)$, represents the corresponding sample generated by the McCann interpolation at $t = 0.5$, which coincides with the barycenter in the our two-marginal cases.}
    \label{plot:unlearn_smile}
    \vspace{-3mm}  % Reduce space above the figure
\end{figure}

\paragraph{Remarks and limitations.}
The $W_2$ barycenter in pixel space gives a principled, utility-aware neutralization target, but the pixel $\ell_2$ ground cost is in general not perfectly aligned with the perceptual face manifold. In practice, neural OT ameliorates this by learning structured maps, yet alternative ground costs or feature-space OT (e.g., in a perceptual embedding) may further improve realism. Extending to multi-class or continuous attributes is straightforward via multi-marginal barycenters with prior weights.

%% file: Conclusion.tex
This paper advances \emph{machine unlearning} by shifting the foundation from an \emph{anchor-based} retrain-on-retain paradigm, which is difficult to verify in practice, to an \emph{auditable, inference-based} paradigm grounded in the \textbf{Marginal Unlearning Principle}. Inspired by ``blur + reinforce'' mechanisms from cognitive neuroscience, our principle targets the \emph{marginal information} contributed by the unlearn set beyond the retain set and provides provable guarantees about what an observer can infer from model outputs. We further implement the principle through a unified information-theoretic framework (a rate–distortion formulation) that trades off utility against an information regularizer $I(S';Z)$. We also devise practical algorithms for both feature and data-point unlearning, and provide guarantees: (i) \emph{auditability} from directly from model inputs-outputs, (ii) \emph{sufficiency} for the prevailing (approximate) anchored guarantees under utility control, and (iii) a complementary \emph{necessity} statement for robust retrain-on-retain models. In the hard-independence limit ($S'\!\perp\! Z$), we identify a single analytic solution, the $\mathcal W_2$ barycenter, that solves feature unlearning under standard learning utility objectives and yields maximal utility among admissible outcomes. Experiments across tabular and image modalities corroborate the capability of the framework and illustrate its practicality. Collectively, these results pave a principled and testable road for replacing anchor-based requirements with neuroscience-inspired, information-theoretic \emph{marginal unlearning}.

\paragraph{Open problems.}
The results here open several directions of research for theory, algorithms, and practice:
\begin{enumerate}[leftmargin=*, itemsep=2pt, topsep=2pt]
\item \textbf{Sharper quantification of marginal information.}
Beyond mutual information, identify and analyze alternative leakage measures (e.g., general $f$-divergence, R\'enyi–MI, $\chi^2$-divergence, or other integral probability metrics (IPMs) over task-relevant test functions) that (a) admit tighter utility-unlearning tradeoffs, (b) are easier to estimate from finite samples, and (c) align with specific downstream utilities.

\item \textbf{From MI to auditable stopping rules.}
Empirically, unlearn accuracy tracks $I(S';Z)$ closely. Developing provable theoretical results on the connection between marginal information quantification (not necessarily the proposed mutual information) and unlearn accuracy may yield principled schedules for tuning $\lambda$ and \emph{stopping criteria} that certify target leakage levels.

\item \textbf{Information–transport synthesis in high dimensions.}
Our results demonstrate that information-theoretical regularization methods potentially can outperform brute force matching methods in estimating optimal transport maps, plans, and Wasserstein barycenter. Can this be extended beyond $\mathcal W_2$ to task-aligned costs and multi-marginal barycenters for complex, high-dimensional unlearning?

\item \textbf{Broader modalities and multi-modal systems.}
Instantiate marginal unlearning for LLMs, speech, graphs, time-series, and multi-modal models. This includes modality-specific $(S',Z)$ designs, direct output-level audit protocols for generative systems, and integration with continual learning/fine-tuning so that updates learn only the \emph{unique} signal of new data while preserving retained knowledge.

\end{enumerate}

%% file: Appendix_1.tex
\subsection{Feature Unlearning}\label{a:feature_unlearning_intro}

% Add in preamble:
% \usepackage{tabularx}
% \usepackage{array}

\begin{table}[t] % use table*, if you truly need full width in a 2-col template
\centering
\footnotesize
\setlength{\tabcolsep}{4pt}
\renewcommand{\arraystretch}{1.12}
\begin{tabularx}{\linewidth}{@{}%
>{\raggedright\arraybackslash}p{0.22\linewidth}%
>{\raggedright\arraybackslash}p{0.16\linewidth}%
>{\raggedright\arraybackslash}p{0.30\linewidth}%
>{\raggedright\arraybackslash}X@{}}
\toprule
\textbf{Definitions} & \textbf{Target} & \textbf{Enforcement mechanism} & \textbf{Notes/guarantees}\\
\midrule
\multicolumn{4}{@{}l}{\textbf{Fairness: outcome statistical parity}}\\
\qquad Statistical Parity \citep{dwork2012fairness, zemel2013learning} &
$\hat Y \perp Z$ &
Constraint at training or post-hoc calibration of selection rates &
Distributional parity; potential utility drop if base rates differ.\\
\midrule
\multicolumn{4}{@{}l}{\textbf{Fairness: representation statistical parity}}\\
\qquad Adversarial debiasing \citep{edwards2016censoring,madras2018laftr,zhang2018adversarial} &
$h(X)\perp Z$ &
Predictor vs.\ adversary predicting $Z$ from $h(X)$ or $\hat Y$ &
Empirical proxy to reduce $I(\cdot;Z)$; architecture-agnostic.\\
\qquad VFAE \citep{louizos2016vfae} &
$h(X)\perp Z$ &
VAE with MMD penalty to match $p(h\mid Z{=}z)$ across $z$ &
Produces reusable ``purged'' latents.\\
\qquad HSIC / dCov regs.\ \citep{li2019hsic,liu2022discov} &
$\operatorname{dCov}\!\big(h(X),Z\big)=0$ &
Add HSIC or distance-covariance penalties during training &
Nonparametric dependence control; task-agnostic.\\
\midrule
\multicolumn{4}{@{}l}{\textbf{Concept erasure: representation independence}}\\
\qquad INLP \citep{ravfogel2020inlp} &
$\Sigma_{hZ}=\mathbf{0}$ &
Iterate: train linear $Z$-probe, project $h(X)$ to probe nullspace &
No linear predictor of $Z$ from $h(X)$; nonlinear leakage may remain.\\
\qquad LACE / R--LACE \citep{ravfogel2022lace} &
$\Sigma_{hZ}=\mathbf{0}$ &
Linear minimax/convex program to erase concept subspace &
Improves optimality/control vs.\ heuristic projections.\\
\qquad LEACE \citep{belrose2023leace} &
$\Sigma_{hZ}=\mathbf{0}$ &
Closed-form least-squares projection &
Blocks all \emph{linear} $Z$-probes; nonlinear leakage may remain.\\
\midrule
\multicolumn{4}{@{}l}{\textbf{Concept erasure: generative independence}}\\
\qquad ESD \citep{gandikota2023erasing} &
$f_{\theta}(X) \perp C$ &
Fine-tune to a negative-guidance teacher (subtract conditional--unconditional score) &
Local removal of concept $c$ with utility preservation.\\
\qquad All-but-One \citep{hong2024allbutone} &
$f_{\theta}(X) \perp C$ &
Surgical update of guidance term with preservation constraints &
Improves fidelity vs.\ naive fine-tuning.\\
\qquad Receler \citep{huang2024receler} &
$f_{\theta}(X) \perp C$ &
Lightweight eraser adapters + regularization &
Robust/local erasure; preserves base weights.\\
\bottomrule
\end{tabularx}
\caption{Representative methods grouped by \emph{unlearning/removal definition} (mathematical form shown) and \emph{enforcement}. Here $h(X)$ denotes a learned representation. Linear erasure is expressed as cross-covariance independence $\Sigma_{hZ}=\mathbf{0}$ (with a centered encoding of $Z$); generative concept removal aims at independence in the model distribution $p_{\theta}(\cdot)$.}
\label{tab:unlearning-map}
\end{table}

\paragraph{Relation to other notions.}
Equalized Odds post-processing \citep{hardt2016equality} requires conditional independence $\hat Y \perp Z \mid Y$ (parity of errors). This may still permit $I(\hat Y;Z)>0$, so it is insufficient for unlearning at the point of exposure. Causal notions (e.g., counterfactual or path-specific fairness) formalize invariance under interventions on $Z$ within a structural causal model and are not equivalent to observational independence; they are typically enforced via functional restrictions or constraints on path-specific effects \citep{kusner2017counterfactual,nabi2018fair,chiappa2019paths}. For auditing, \emph{amnesic probing} assesses whether removing targeted information from representations changes behavior \citep{elazar2021amnesic}; recent diffusion benchmarks stress-test leakage and collateral damage after erasure \citep{amara2025erasebench}.

\subsection{Data Unlearning}\label{a:data_unlearning_intro}

\paragraph{Why anchor-based unlearning is unauditable, and why a verifiable definition is needed.}
Anchor-based frameworks define unlearning via equivalence to a theoretical retrain-on-retain \emph{anchor}. However, because this anchor is typically unobserved, strict auditing based solely on black-box behavior is theoretically impossible \citep{thudi2022auditable}. In practice, existing verification schemes have proven \emph{fragile}: \citet{zhang2024fragile} demonstrate that adversarial training procedures can satisfy both backdoor-based and reproducing verification metrics while retaining target information. While early frameworks establish feasibility, they rely heavily on \emph{auxiliary instrumentation}, such as seeded randomness, gradient traces, provenance metadata, or planted canaries, rather than observable properties of the released model \citep{sommer2022athena}. Similarly, cryptographic and TEE-based proofs target verifiability but necessitate trusted hardware and impose significant operational overhead \citep{weng2022poul}. Recent surveys underscore that robust verification remains a central open challenge \citep{xue2025verificationSurvey,xu2024machine}. Moreover, naive workflows (e.g., comparing pre- and post-unlearning models) can inadvertently leak membership information \citep{chen2021unlearning_leaks}. These systemic limitations motivate black-box, output-level definitions, such as our proposed marginal unlearning principle, that are directly auditable from observable model behavior.

\newpage

%% file: Appendix_2.tex
\subsection{Proof of Lemma \ref{lem:mi-tv}}\label{a:proof_lemma_mi_tv}

\begin{proof}
Since $\mathcal{L}(\widehat Y_{\mathrm{margin}})=M$, the chain rule for KL gives
\[
I(\widehat Y_{\mathrm{margin}};Z)
=\KL\!\big(\mathcal{L}(\widehat Y_{\mathrm{margin}},Z)\ \Vert\ M\otimes\mathcal{L}(Z)\big)
=\mathbb{E}_Z\Big[\KL\!\big(\mathcal{L}(\widehat Y_{\mathrm{margin}}\!\mid Z)\ \Vert\ M\big)\Big],
\]
which equals $(1-\pi)\KL(P_0\Vert M)+\pi\KL(P_1\Vert M)=\JS_\pi(P_0,P_1)$, proving the identity.

For TV, use triangle inequality and Pinsker on a general measurable space:
\[
\TV(P_0,P_1)\ \le\ \TV(P_0,M)+\TV(P_1,M)
\ \le\ \sqrt{\tfrac12\,\KL(P_0\Vert M)}+\sqrt{\tfrac12\,\KL(P_1\Vert M)}.
\]
Let $a=(1-\pi)\KL(P_0\Vert M)$ and $b=\pi\KL(P_1\Vert M)$ so that $a+b=I(\widehat Y_{\mathrm{margin}};Z)$. It then follows from Cauchy–Schwarz that
\begin{align*}
    \sqrt{\KL(P_0\Vert M)}+\sqrt{\KL(P_1\Vert M)}
    & =\frac{\sqrt{a}}{\sqrt{1-\pi}}+\frac{\sqrt{b}}{\sqrt{\pi}}\\
    & \le\ \sqrt{\Big(\frac{1}{1-\pi}+\frac{1}{\pi}\Big)\,(a+b)}\\
    & =\ \sqrt{\frac{I(\widehat Y_{\mathrm{margin}};Z)}{\pi(1-\pi)}}.
\end{align*}
Therefore, we obtain
\begin{equation*}
    \TV(P_0,P_1)\le \sqrt{\frac{I(\widehat Y_{\mathrm{margin}};Z)}{2\pi(1-\pi)}}.
\end{equation*}
\end{proof}

\subsection{Proof of Theorem \ref{thm:mu-to-anchor}}\label{a:proof_theorem_mu-to-anchor}

\begin{proof}
Write $\mu_{\Theta}^{\,p} := \mathcal{L}(\Theta(X))$ for $X \sim p$ independent of $\Theta$.
Let $Q_h(\cdot\mid x)$ denote the predictive distribution of model $h$ at input $x$, and let $\eta_x := \mathcal{L}(Y \mid X=x)$ be the true label distribution under the retain set.

\medskip
\textit{Step 1: Triangle decomposition at $p^d$.}
By the triangle inequality for total variation distance:
\[
\TV\!\big(\mu_{\Theta_u}^{\,p^d},\mu_{\Theta_r}^{\,p^d}\big)
\ \le\
\underbrace{\TV\!\big(\mu_{\Theta_u}^{\,p^d},\mu_{\Theta_u}^{\,p^r}\big)}_{(\mathrm{I})}
\;+\;
\underbrace{\TV\!\big(\mu_{\Theta_u}^{\,p^r},\mu_{\Theta_r}^{\,p^r}\big)}_{(\mathrm{II})}
\;+\;
\underbrace{\TV\!\big(\mu_{\Theta_r}^{\,p^r},\mu_{\Theta_r}^{\,p^d}\big)}_{(\mathrm{III})}.
\]

\medskip
\textit{Step 2: Utility alignment (Term II).}
We bound the distance between the unlearned model and the anchor on the retain distribution $p^r$. By the joint convexity of $\TV$ and the triangle inequality relative to the ground truth $\eta_x$:
\begin{align*}
\TV\!\big(\mu_{\Theta_u}^{\,p^r},\mu_{\Theta_r}^{\,p^r}\big)
&\le \mathbb{E}_{\Theta_u,\Theta_r} \int \TV\!\big(Q_{\Theta_u}(\cdot\mid x),\,Q_{\Theta_r}(\cdot\mid x)\big)\,p^r(dx)\\
&\le \mathbb{E}_{\Theta_u,\Theta_r} \int \Big[\TV\!\big(Q_{\Theta_u}(\cdot\mid x),\eta_x\big)
+\TV\!\big(Q_{\Theta_r}(\cdot\mid x),\eta_x\big)\Big]\,p^r(dx).
\end{align*}
Applying Pinsker’s inequality pointwise, $\TV(Q, \eta) \le \sqrt{\frac{1}{2}\KL(\eta \| Q)}$. Using Jensen's inequality (concavity of the square root) to move the expectation inside:
\begin{align*}
    \mathbb{E}_{\Theta_u}\int \TV\!\big(Q_{\Theta_u},\eta_x\big)\,p^r(dx) & \le\ \sqrt{\tfrac12\,\mathbb{E}_{\Theta_u}\int \KL\!\big(\eta_x\Vert Q_{\Theta_u}(\cdot\mid x)\big)\,p^r(dx)}\\
    & =\ \sqrt{\tfrac12\,\overline{\mathrm{reg}}_{\log}(\Theta_u)}
\ \le\ \sqrt{\tfrac{\delta}{2}}.
\end{align*}
A symmetric bound applies to $\Theta_r$ with $\delta_g$. Summing these yields:
\[
(\mathrm{II})\ \le\ \sqrt{\tfrac12}\,\big(\sqrt{\delta}+\sqrt{\delta_g}\big).
\]

\medskip
\textit{Step 3: Membership independence (Term I).}
Since
\[
    \mathcal L(X_{\mathrm{margin}}\mid Z=0)=p^d,
    \qquad
    \mathcal L(X_{\mathrm{margin}}\mid Z=1)=p^r,
\]
and $\widehat Y_{\mathrm{margin}}=f_{\Theta_u}(X_{\mathrm{margin}})$, we have
\[
    \mathcal L(\widehat Y_{\mathrm{margin}}\mid Z=0)
    =
    \mu_{\Theta_u}^{\,p^d},
    \qquad
    \mathcal L(\widehat Y_{\mathrm{margin}}\mid Z=1)
    =
    \mu_{\Theta_u}^{\,p^r}.
\]
These are mixture laws obtained after averaging over the internal randomness $\Theta_u$.
Applying Lemma~\ref{lem:mi-tv} directly to $\widehat Y_{\mathrm{margin}}$ gives
\[
(\mathrm{I})
=
\TV\!\big(\mu_{\Theta_u}^{\,p^d},\mu_{\Theta_u}^{\,p^r}\big)
\le
\sqrt{\frac{I(\widehat Y_{\mathrm{margin}};Z)}{2\pi(1-\pi)}}
\le
\sqrt{\frac{\varepsilon_u}{2\pi(1-\pi)}}.
\]

\medskip
\textit{Step 4: Anchor generalization (Term III).}
The term $(\mathrm{III})=\TV\!\big(\mu_{\Theta_r}^{\,p^r},\mu_{\Theta_r}^{\,p^d}\big)$ is the inherent distribution shift of the anchor model and is left explicit.
\medskip

Finally, combining Steps 2, 3, and 4 into Step 1 completes the proof.
\end{proof}

\subsection{Proof of Lemma \ref{l:DU_as_condition_for_good_model}}\label{a:proof_DU_as_good}

\begin{proof}
Assume that the first alternative fails. Then there exists
\(y^*\in\widehat{\mathcal Y}\) such that
\[
    \left|\log\frac{p_1(y^*)}{p_0(y^*)}\right|>\epsilon.
\]
In particular, the two densities differ at \(y^*\), so
\[
    |p_1(y^*)-p_0(y^*)|>0.
\]
Define
\[
    \gamma(y^*):=\frac12 |p_1(y^*)-p_0(y^*)|>0.
\]
Since \(\widehat{\mathcal Y}\) is open and \(y^*\in\widehat{\mathcal Y}\),
there exists \(r_0>0\) such that
\[
    B_{r_0}(y^*)\subset \widehat{\mathcal Y}.
\]
By continuity of \(p_0\) and \(p_1\), equivalently of \(p_1-p_0\), there exists
\(\rho_*(y^*)\in(0,r_0]\) such that the sign of \(p_1-p_0\) is constant on
\(B_{\rho_*(y^*)}(y^*)\), and
\[
    |p_1(y)-p_0(y)|
    \ge
    \gamma(y^*),
    \qquad y\in B_{\rho_*(y^*)}(y^*).
\]
For notational simplicity, write \(\rho_*:=\rho_*(y^*)\). Now define the \(1\)-Lipschitz test function
\[
    \varphi(y):=(\rho_*-\|y-y^*\|)_+ .
\]
By Kantorovich-Rubinstein duality,
\begin{align*}
    \mathcal{W}_{d_{\widehat{\mathcal Y}}}(f(X_0),f(X_1))
    &\ge
    \left|
    \int_{\widehat{\mathcal Y}}
    \varphi(y)(p_1(y)-p_0(y))\,dy
    \right| \\
    &\ge
    \gamma(y^*)
    \int_{B_{\rho_*}(y^*)}(\rho_*-\|y-y^*\|)\,dy \\
    &=
    \gamma(y^*)
    \frac{\omega_m\rho_*^{m+1}}{m+1}.
\end{align*}
On the other hand, since \(f\) is \(L\)-Lipschitz,
\[
    \mathcal{W}_{d_{\widehat{\mathcal Y}}}(f(X_0),f(X_1))
    \le
    L\,\mathcal{W}_{d_{\mathcal X}}(X_0,X_1).
\]
Combining the two estimates and using
\(\mathcal W_{d_{\mathcal X}}(X_1,X_0)>0\), we obtain
\[
    L
    \ge
    \frac{
        \gamma(y^*)\,\omega_m\,\rho_*^{m+1}
    }{
        (m+1)\,\mathcal{W}_{d_{\mathcal X}}(X_1,X_0)
    }
    =
    L^*(\epsilon,y^*).
\]
Therefore, if the likelihood-ratio bound fails, then the Lipschitz constant
must satisfy \(L\ge L^*(\epsilon,y^*)\).
\end{proof}

%% file: Appendix_3.tex
\subsection{Utility motivation}\label{a:utility_motivation}

Mutual information is a widely used quantification of the common information shared by two random variables, with useful coding interpretation in the discrete finite-entropy setting. In particular, given a data set $X$ with a goal to compress $X$ by an encoding $\hat{X}$, the volume of code needed to encode $X$ is $2^{H(X)}$ where $H(X)$ is the entropy of $X$. Furthermore, after observing \(\hat X\), the average volume of $X$ mapped to individual realizations of $\hat{X}$ is equal to $2^{H(X|\hat{X})}$. Here,
\[
    H(X\mid \hat X)
    :=
    -\sum_{\hat x}p(\hat x)\sum_x p(x\mid \hat x)\log p(x\mid \hat x).
\]
is the conditional entropy of $X$ on $\hat{X}$. Intuitively, a higher conditional entropy means more volume of $X$ is expected to be mapped to individual realizations of $\hat{X}$, which implies more randomness of $X$ remains given the observation of $\hat{X}$. In other words, less $X$ is explained by $\hat{X}$.

Since the volume of code for $X$ is $2^{H(X)}$ and the average volume of code mapped to each $\hat{X}$ is $2^{H(X|\hat{X})}$, the average cardinality of the partition generated by the values of $\hat{X}$ on the values of $X$ is the ratio:
\begin{equation}
    \frac{2^{H(X)}}{2^{H(X|\hat{X})}} = 2^{I(X;\hat{X})}.
\end{equation}
Here, $I(X;\hat{X}) = H(X) - H(X|\hat{X})$ is the mutual information between $X$ and $\hat{X}$. On the one hand, higher mutual information implies that  $\hat{X}$ generates a partition with higher cardinality (or usually finer partition) on $X$, which further implies more common information is shared between $X$ and $\hat{X}$. On the other hand, from a data compression perspective, lower mutual information means $\hat{X}$ generates a partition on $X$ with lower cardinality, which further implies a better data compression rate, because $\hat{X}$ can compress $X$ into a partition of smaller cardinality.

Based on the motivation provided by coding interpretation from the discrete finite-entropy setting, as discussed in Section~\ref{S:Definition}, we adopt mutual information to quantify the common information and compression rate between random variables. For  unlearning quality purposes, we hope to maintain as much information of $X$ as possible in generating $\hat{X}$. Therefore, to maximize utility, we should maximize mutual information $I(X;\hat{X})$ or, equivalently, minimize the information loss from from $X$ to $\hat{X}$.

\subsection{Admissibility}\label{a:admissibility}

Since we are unlearning the information of $Z$ by compressing dataset or variable pair $(X,Z)$, it is natural to require the resulting compressed data $\hat{X}$ to be measurable with respect to $(X,Z)$. Intuitively, the compression output $\hat{X}$ should have its ``root'' from $(X,Z)$ without introducing additional randomness by the compression map $f$ itself. Technically speaking, the ``root'' here means that for every event or observation $A$ of the compressed $\hat{X}$, the information or pre-image represented by the observation $\hat{X}^{-1}(A) := \{\omega: \hat{X}(\omega) \in A\}$ comes from the knowledge of $f^{-1}(A) := \{(x,z): f(x,z) \in A\}$ based on $(X,Z)$:
\begin{align}
\{\omega: \hat{X}(\omega) \in A\} 
&= \hat{X}^{-1}(A) \nonumber \\
&= (X,Z)^{-1}(f^{-1}(A)) \nonumber \\
&= \{\omega: (X(\omega),Z(\omega)) \in f^{-1}(A)\}.
\label{eq:long_equation}
\end{align}
Here, $\omega \in \Omega$ is the smallest unit of information we can have from the measure space $(\Omega, \mathcal{F},\mathbb{P})$. From a probability-theoretical perspective, since $\hat{X}$ is a deterministic compression of $(X,Z)$, it generates a coarser partition (or, more technically, sigma-algebra) than the original information $(X,Z)$ and we say $\hat{X}$ is measurable with respect to $(X,Z)$, denoted by
\begin{equation}
    \sigma(\hat{X}) \subset \sigma(X,Z).
\end{equation}
This is equivalent to the existence of a $\mathcal{B}_\mathcal{X} \otimes \mathcal{B}_\mathcal{Z}/\mathcal{B}_\mathcal{X}$-measurable map, denoted by $f$, such that $\hat{X} = f(X,Z)$. That is, our admissibility is equivalent to the assumption that the data compression process does not create information or randomness by itself. Therefore, we define the admissible unlearning outcome in our framework as follows:
\begin{equation}
\label{eq:admissible}
\mathcal{A}(X,Z) \;:=\; \Bigl\{\hat{X} = f(X,Z):
  f \text{ is }
 \mathcal{B}_\mathcal{X} \otimes \mathcal{B}_\mathcal{Z}/\mathcal{B}_\mathcal{X} 
  \text{-measurable} \Bigr\},
\end{equation}
and we use $\hat{X} = f(X,Z)$ and $\hat{X} \in \mathcal{A}(X,Z)$ interchangeably. This is the deterministic admissibility class. If randomized mechanisms are
allowed, the deterministic class can be enlarged to mechanisms of the form \(\hat X=f(X,Z,\Theta)\), where the internal seed \(\Theta\perp (X,Z)\) is not released, and all information quantities are computed after averaging over \(\Theta\). In this current work, we focus on the deterministic admissible compressions to better demonstrate the theoretical intuition behind the framework, and defer a detailed study of the randomized compression to future work.

\subsection{Details on considered utility quantifications}\label{a:utility_appendix}
In practical machine unlearning, the utility of unlearning may need to be evaluated with respect to a different target variable $Y$ rather than the original dataset $X$. Moreover, mutual information, while often a natural choice, is not the only metric for quantifying the relationship between the unlearning outcome $\hat{X}$ and the target variable $Y$. To accommodate diverse objectives, we extend our framework to the general formulation:
\begin{equation}\label{eq:general_formulation}
\inf_{\hat{X}=f(X,Z)}
\bigl\{
\mathcal C(Y;\hat X,Z):\hat X\perp Z
\bigr\}.
\end{equation}
where $\mathcal{C}(Y;\hat X,Z)$ represents a utility cost or risk quantification, and the constraint $\hat{X} \perp Z$ ensures that the unwanted information $Z$ is fully removed from $\hat{X}$. Below, we introduce several commonly used cost objectives and their corresponding constrained optimization problems, for which we provide a unified analytic feature unlearning solution in Section \ref{s4:feature_analytic_solution}.

\medskip
\noindent
{\bf Mutual information maximization:} The utility is defined as $\mathcal C(Y;\hat X,Z)=-I(Y;\hat X,Z)$, measuring the shared information between the target variable $Y$ and the unlearning outcome $\hat{X}$. This objective is widely applied in classification methods such as decision trees \cite{criminisi2012decision} and in deep learning techniques, including Deep InfoMax \cite{hjelm2018deepinfomax} and information bottleneck methods \cite{alemi2017vib}. The corresponding optimization problem, \textit{Mutual-Information-Maximized Feature Unlearning}, is given by $$\sup_{\hat X=f(X,Z)} \bigl\{ I(Y;\hat X,Z):\hat X\perp Z \bigr\}.$$ Since $I(Y;\hat X,Z)=H(Y)-H(Y\mid \hat X,Z)$, this problem is equivalent to minimizing the conditional entropy $H(Y|\hat{X},Z)$.

\medskip
\noindent
{\bf Posterior KL-divergence cost:} The cost is
\[
\mathcal{C}(Y;\hat{X},Z)
=
-\mathbb E\!\left[
D_{\mathrm{KL}}\!\big(
\mathbb{P}(Y\mid \hat{X},Z)\,\Vert\,\mathbb{P}(Y)
\big)
\right],
\]
where \(D_{\mathrm{KL}}\) measures the divergence between the posterior and
prior distributions of \(Y\). KL-divergence objectives are widely used in
information-theoretic learning and generative modeling, including variational
methods such as VAEs \cite{higgins2017betaVAE}. The corresponding optimization
problem, \textit{Posterior-KL-Minimized Feature Unlearning}, is formulated as
\[
\inf_{\hat{X}=f(X,Z)}
\left\{
-\mathbb E\!\left[
D_{\mathrm{KL}}\!\big(
\mathbb{P}(Y\mid \hat{X},Z)\,\Vert\,\mathbb{P}(Y)
\big)
\right]
:
\hat{X}\perp Z
\right\}.
\]
This problem seeks to make $\mathbb{P}(Y|\hat{X},Z)$ as deterministic as possible relative to the prior $\mathbb{P}(Y)$, thereby enhancing the predictive power of $(\hat{X},Z)$ for $Y$. Provided the quantities are well-defined, the expected posterior KL equals \(I(Y;\hat X,Z)\). Thus, minimizing this cost is equivalent to maximizing the informativeness of the representation \((\hat X,Z)\) about \(Y\).

\medskip
\noindent
{\bf Conditional probability energy maximization:} The utility is defined as the $L^2$-norm of the conditional probability $\mathbb{P}(Y \in A | \hat{X},Z)$ for classification or the negative mean squared error (MSE) for regression:
\begin{equation*}
    \mathcal{C}(Y;\hat{X},Z) =
    \begin{cases}
         -\big\|\mathbb P(Y\in A\mid \hat X,Z)\big\|_{L^2}^2 & \text{for classification-type events \(A\in\mathcal B_{\mathcal Y}\),}\\
        -\big\|\mathbb E(Y\mid \hat X,Z)\big\|_{L^2}^2 & \text{regression-type targets \(Y\in L^2(\Omega;\mathcal Y)\)}.
    \end{cases}
\end{equation*}
The corresponding optimization problem, \textit{Energy-Maximized Feature Unlearning}, is formulated as $\sup_{\hat{X} = f(X,Z)} \{||\mathbb P(Y\in A\mid \hat X,Z)||_{L^2}^2 : \hat{X} \perp Z\}$. A higher $L^2$-norm indicates a more precise prediction of the event $\{Y \in A\}$ based on $(\hat{X},Z)$, leading to reduced Bayes error and improved decision boundaries. Similarly, in the regression setting: $\sup_{\hat{X} = f(X,Z)} \{||\mathbb E(Y\mid \hat X,Z)||_{L^2}^2 : \hat{X} \perp Z\}$, a smaller conditional-mean-energy cost corresponds to larger explained variance and smaller mean-squared prediction error.

As we show in Section~\ref{s4:feature_analytic_solution}, when the above-listed objectives are applied, there exists a universal optimal feature unlearning solution to the general formulation: equation \eqref{eq:general_formulation} for arbitrary target variables $Y$.

\subsection{Formulation of constrained optimization problems}

Our goal here is to provide theoretical solutions to the following constrained optimization problems under mild assumptions, thereby developing a unified mathematical framework for machine unlearning of features and labels under various utility objectives:

\begin{prob}[Conditional-entropy-minimized feature unlearning]\label{prob:Conditional_Entropy_Minimized}
    \begin{equation*}
        \inf_{\hat X\in\mathcal A(X,Z)} \bigl\{H(Y\mid \hat X,Z):\hat X\perp Z\bigr\}.
    \end{equation*}
    \end{prob}
    In many cases, an unlearning output $\hat{X}$ may be used to generate inferences or predictions for some random variable $Y$. Thus, it is desirable to solve Problem \ref{prob:Conditional_Entropy_Minimized} for such a target variable $Y$. Notice that, due to $I(Y;\hat{X},Z) = H(Y) - H(Y|\hat{X},Z)$, the above problem shares the same solution as the maximization of mutual information between $Y$ and $(\hat{X},Z)$:

\begin{prob}[Mutual-information-maximized feature unlearning]\label{prob:Mutual_Information_Maximized}
    \begin{equation*}
        \sup_{\hat X\in\mathcal A(X,Z)} \bigl\{I(Y;\hat X,Z):\hat X\perp Z\bigr\}.
    \end{equation*}
    \end{prob}
    Notably, the optimal solution to Problems \ref{prob:Conditional_Entropy_Minimized} and \ref{prob:Mutual_Information_Maximized} does not depend on the specific choice of $Y$ due to the monotonicity of the functional $H(Y|\cdot)$ with respect to the sigma-algebra generated by $(\hat{X},Z)$. Thus, despite the explicit presence of $Y$ in Problem \ref{prob:Conditional_Entropy_Minimized}, it provides a generalized solution for any choice of $Y$.

\begin{prob}[Posterior-divergence cost for feature unlearning]\label{prob:KL-Divergence-Maximized}
    \begin{equation*}
    \inf_{\hat X\in\mathcal A(X,Z)}
    \left\{
    -\mathbb E\!\left[
    D_{\mathrm{KL}}\!\big(
    \mathbb P(Y\mid \hat X,Z)\Vert \mathbb P(Y)
    \big)
    \right]
    :
    \hat X\perp Z
    \right\}.
    \end{equation*}
    \end{prob}
    Given a variable of interest denoted by $Y$, a general downstream machine learning or AI task may aim to estimate the conditional probability using the unlearning outcome $(\hat{X},Z)$. Therefore, it is desirable to make $\mathbb{P}(Y|\hat{X},Z)$ as deterministic as possible relative to the original distribution of $Y$. To quantify this determinism, we use the KL-divergence of $\mathbb{P}(Y|\hat{X},Z)$ relative to $\mathbb{P}(Y)$, leading to the optimization problem above. Intuitively, a more accurate prediction of $\mathbb{P}(Y|\hat{X},Z)$ implies less randomness relative to $\mathbb{P}(Y)$, which increases the KL-divergence of $\mathbb{P}(Y|\hat{X},Z)$ relative to $\mathbb{P}(Y)$.

\begin{prob}[Energy-maximized feature unlearning]
\label{prob:Energy-Maximized}
For a measurable target \(Y:\Omega\to\mathcal Y\) and event
\(A\in\mathcal B_{\mathcal Y}\), consider
\[
    \sup_{\hat X\in\mathcal A(X,Z)}
    \left\{
    \big\|\mathbb P(Y\in A\mid \hat X,Z)\big\|_{L^2}^2
    :\hat X\perp Z
    \right\}.
\]
For regression-type targets \(Y\in L^2(\Omega;\mathcal Y)\), with
\(\mathcal Y\) a Hilbert space, one may instead consider
\[
    \sup_{\hat X\in\mathcal A(X,Z)}
    \left\{
    \big\|\mathbb E(Y\mid \hat X,Z)\big\|_{L^2}^2
    :\hat X\perp Z
    \right\}.
\]
\end{prob}
Finally, from the perspective of conditional probability estimation, for a given $Y$ and event $A\in\mathcal B_{\mathcal Y}$, it is natural to maximize the energy (or equivalently, the $L^2$ norm) of the conditional probability $\mathbb{P}(\{Y \in A\}|\hat{X},Z)$. Here, a larger $L^2$ norm of the conditional probability indicates a more precise prediction of the event $\{Y \in A\}$ based on the information provided by $\hat{X}$.

%sNote that the objectives listed above are widely used in scientific fields such as biology, chemistry, physics, medical science, social science, and AI. %In the next section, we demonstrate that focusing on the sigma-algebra generated by $\hat{X}$ allows us to simultaneously solve these four unlearning optimization problems, despite their seemingly different objective functions.

\newpage

%% file: Appendix_4.tex
\subsection{Wasserstein distance and barycenter}\label{a:OT_intro}

Given $\mu, \nu \in \mathcal{P}(\mathbb{R}^d)$ where $\mathcal{P}(\mathbb{R}^d)$ denotes the set of all the probability measures on $\mathbb{R}^d$, $$\mathcal{W}_2(\mu,\nu) := \left(\inf_{ \lambda \in \Pi(\mu,\nu) } \Big\{\int_{\mathbb{R}^d \times \mathbb{R}^d} ||x_1 - x_2||^2 d \lambda(x_1,x_2)\Big\}\right)^{\frac{1}{2}}.$$ Here, $\Pi(\mu,\nu) := \{\pi \in \mathcal{P}((\mathbb{R}^d)^2): \int_{\mathbb{R}^d} d\pi(\cdot,v) = \mu, \int_{\mathbb{R}^d} d\pi(u,\cdot) = \nu \}$.
$(\mathcal{P}_2(\mathbb{R}^d),\mathcal{W}_2)$ is called the Wasserstein space, where $\mathcal{P}_2(\mathbb{R}^d):= \Big\{\mu \in \mathcal{P}(\mathbb{R}^d): \int_{\mathbb{R}^d} ||x||^2 d\mu < \infty\Big\}$. Also, we use $\mathcal{P}_{2,ac}(\mathbb{R}^d)$ to denote the set of probability measures with finite second moments and are absolute continuous w.r.t. the Lebesgue measure. To simplify notation, we often denote $$\mathcal{W}_2(X_1,X_2) := \mathcal{W}_2(\mathcal{L}(X_1),\mathcal{L}(X_2)),$$ where $\mathcal{L}(X) := \mathbb{P} \circ X^{-1} \in \mathcal{P}(\mathbb{R}^d)$ is the law or distribution of $X$, $X: \Omega \rightarrow \mathcal{X} := \mathbb{R}^d$ is a random variable (or vector) with an underlying probability space $(\Omega, \mathcal{F}, \mathbb{P})$. Intuitively, one can consider the Wasserstein distance as $L^2$ distance after optimally coupling two random variables whose distributions are $\mu$ and $\nu$. That is, if the pair $(X_1,X_2)$ is an optimal coupling \cite{villani2008optimal}, then
$$\mathcal{W}_2(X_1,X_2) = ||X_1 - X_2||_{L^2} = \left(\int_{\Omega} \|X_1(\omega)-X_2(\omega)\|^2\,d\mathbb P(\omega) \right)^{1/2}.$$
Given \(\{\mu_z\}_{z\in\mathcal Z}\subset\mathcal P_2(\mathbb R^d)\) and
weights \(\lambda\in\mathcal P(\mathcal Z)\), their Wasserstein barycenter
\cite{agueh2011barycenters} is any
\[
\bar\mu\in
\argmin_{\mu\in\mathcal P_2(\mathbb R^d)}
\left\{
\int_{\mathcal Z}\mathcal W_2^2(\mu_z,\mu)\,d\lambda(z)
\right\}.
\]

If there is no danger of confusion, we will refer to the Wasserstein barycenter simply as barycenter. In the feature-unlearning setting, let
\[
    \mu_z:=\mathcal L(X\mid Z=z),
\]
and let \(\bar\mu\) be a barycenter of \(\{\mu_z\}_{z\in\mathcal Z}\). When
\(\mu_z\in\mathcal P_{2,ac}(\mathbb R^d)\), Brenier's theorem gives an optimal
transport map $T_z:\mu_z\to\bar\mu$. We define the barycentric representation by
\[
    \bar X:=T_Z(X).
\]
Then, conditionally on \(Z=z\),
\[
    \mathcal L(\bar X\mid Z=z)
    =
    (T_z)_\#\mu_z
    =
    \bar\mu.
\]
Hence \(\bar X\perp Z\). Moreover, under the absolute-continuity assumptions
used below, the Brenier maps are a.e. invertible on their supports, with inverse
maps given by the corresponding Brenier maps from \(\bar\mu\) back to \(\mu_z\).

\subsection{Proof of Lemma \ref{l:epsilon_DU_bound_mutual_info}}\label{a:proof_DU_bound_mutual_info}

\begin{proof}
    To start, let $\hat{X} := \hat{X}_{\margin}$ to simplify notation, and notice that it follows from the construction of dataset or variable pair $(X_{\margin},Z)$ that $\{Z = 0\} = \{X_{\margin} = X_0\}$ and $\{Z = 1\} = \{X_{\margin} = X_1\}$. Also, we have 
    \begin{align}
        |\mathbb{P}(Z = 0|\hat{X}) - \mathbb{P}(Z = 1|\hat{X})| & = |2 \mathbb{P}(Z = 0|\hat{X}) - 1|\\
    &  = 2 |\mathbb{P}(Z = 0|\hat{X}) - \frac{1}{2}|\\
    &  = 2 ||\mathbb{P}(Z|\hat{X}) - \mathbb{P}(Z)||_{TV},
    \end{align}
    where the third line follows from the definition of total variation distance and the prior information $\mathbb{P}(Z = 0) = \frac{1}{2}$. By taking the expectation over $\hat{X}$, we have
    \begin{align*}
        \mathbb{E}_{\hat{X}} ( |\mathbb{P}(Z = 0|\hat{X}) - \mathbb{P}(Z = 1|\hat{X})| ) & = 2\mathbb{E}_{\hat{X}} ( ||\mathbb{P}(Z|\hat{X}) - \mathbb{P}(Z)||_{TV} )\\
    & \leq 2\mathbb{E}_{\hat{X}} \left( \sqrt{\frac{1}{2}D_{\mathrm{KL}}(\mathbb{P}(Z|\hat{X})||\mathbb{P}(Z))} \right)\\
    & \leq 2 \sqrt{\frac{1}{2}\mathbb{E}_{\hat{X}} \left( D_{\mathrm{KL}}(\mathbb{P}(Z|\hat{X})||\mathbb{P}(Z)) \right)}\\
    & = 2\sqrt{\frac{1}{2}I(Z;\hat{X})}.
    \end{align*}
    Here, the second line follows from Pinsker's inequality, the third from Jensen's inequality, and the fourth from the definition of mutual information. Now, for any fixed $\epsilon \in (0,1)$, it follows from Markov's inequality that
    \begin{equation*}
        \mathbb{P}\left( \{ |\mathbb{P}(Z = 0|\hat{X}) - \mathbb{P}(Z = 1|\hat{X})| \leq \epsilon \} \right) \geq 1-\frac{2}{\epsilon}\left( \sqrt{\frac{1}{2} I(Z;\hat{X})} \right).
    \end{equation*}
    Finally, it follows from 
    $$|\mathbb{P}(Z = 0|\hat{X}) - \mathbb{P}(Z = 1|\hat{X})| \leq \epsilon \Longrightarrow \log\left(\frac{\mathbb{P}(Z = 0 \mid \hat{X})}{\mathbb{P}(Z = 1 \mid \hat{X})}\right) \leq \log\Big(\frac{1+\epsilon}{1-\epsilon}\Big)$$
    that 
    $$\big\{|\mathbb{P}(Z = 0|\hat{X}) - \mathbb{P}(Z = 1|\hat{X})| \leq \epsilon\big\} \subset \left\{ |\log\left(\frac{\mathbb{P}(Z = 0 \mid \hat{X})}{\mathbb{P}(Z = 1 \mid \hat{X})}\right)| \leq \log\Big(\frac{1+\epsilon}{1-\epsilon}\Big) \right\},$$ and
    \begin{align*}
        \mathbb{P}\left( \{ |\log\left(\frac{\mathbb{P}(Z = 0 \mid \hat{X})}{\mathbb{P}(Z = 1 \mid \hat{X})}\right)| \leq \log(\frac{1+\epsilon}{1-\epsilon}) \} \right) & \geq \mathbb{P}\left( \{ |\mathbb{P}(Z = 0|\hat{X}) - \mathbb{P}(Z = 1|\hat{X})|\leq \epsilon \} \right)\\
    & \geq 1-\frac{2}{\epsilon}\left( \sqrt{\frac{1}{2} I(Z;\hat{X})} \right).
    \end{align*}
    That completes the proof.
\end{proof}

\subsection{Proof of Lemma \ref{l:Finest_Sigma_Algebra}}\label{a:fine_algebra}

\begin{proof}
    We provide the argument for the lifted statement. By the absolute-continuity assumptions, the barycenter \(\bar P\) is absolutely continuous, and Brenier's theorem yields optimal maps \(T_z:P_z\to\bar P\). Moreover, the corresponding Brenier maps are a.e. invertible on their supports. Denote the a.e. inverse of \(T_z\) by $S_z:\bar P\to P_z$. Thus,
    \begin{equation*}
        S_z(T_z(x))=x \qquad P_z\text{-a.e.}
    \end{equation*}
    Since \(\bar X=T_Z(X)\), the pair \((\bar X,Z)\) is a measurable function of \((X,Z)\). Hence,
    \begin{equation*}
        \sigma(\bar X,Z)\subset \sigma(X,Z).
    \end{equation*}
    Conversely, by the a.e. inverse relation, we also have
    \begin{equation*}
        X=S_Z(\bar X), \qquad \mathbb P\text{-a.s.}.
    \end{equation*}
    Therefore \((X,Z)\) is a measurable function of \((\bar X,Z)\), up to null sets, and hence 
    \begin{equation*}
        \sigma(X,Z)\subset \sigma(\bar X,Z).
    \end{equation*}
    Combining the two inclusions gives
    \begin{equation*}
        \sigma(X,Z)=\sigma(\bar X,Z), \qquad \mathbb P\text{-a.s.}.
    \end{equation*}
    Finally, let \(\hat X=f(X,Z)\) be any admissible deterministic unlearning outcome. Then \((\hat X,Z)\) is a measurable function of \((X,Z)\), so that
    \begin{equation*}
        \sigma(\hat X,Z)\subset \sigma(X,Z).
    \end{equation*}
    Using the equality above, we obtain
    \begin{equation*}
        \sigma(\hat X,Z) \subset \sigma(X,Z) = \sigma(\bar X,Z).
    \end{equation*}
    That completes the proof.
\end{proof}

\subsection{Sigma-algebra and information}\label{a3:Sigma-Algebra and Information}

In probability theory, a probability space is often represented as a triple $(\Omega, \Sigma, \mathbb{P})$, where $\Omega$ is the sample space, $\Sigma$ is the sigma-algebra (a collection of subsets of \(\Omega\) closed under complements and countable unions), and $\mathbb{P}: \Sigma \rightarrow [0,1]$ is a probability measure that assigns probabilities to each event in $\Sigma$.

The same sample space can be associated with different sigma-algebras, resulting in different measurable spaces, and, after specifying a probability measure, different probability spaces. We say that a sigma-algebra $\Sigma_1$ is finer than $\Sigma_2$, denoted $\Sigma_2 \subset \Sigma_1$, if $\Sigma_1$ contains all events in $\Sigma_2$. Conversely, we say $\Sigma_1$ is coarser than $\Sigma_2$ if $\Sigma_1$ contains fewer events than $\Sigma_2$.

A random variable or random vector $X$ is a measurable function from the probability space to $\mathbb{R}^d$ (or $\mathbb{C}^d$), $X: \Omega \rightarrow \mathbb{R}^d$. The sigma-algebra generated by $X$, denoted by $\sigma(X)$, comprises all possible events in $\Omega$ that could be defined based on the image of $X$ in $\mathbb{R}^d$ (or $\mathbb{C}^d$). Thus, if $X$ generates a finer sigma-algebra than another variable $X’$, denoted $\sigma(X’) \subset \sigma(X)$, then $X$ contains more events and, therefore, more information than $X’$.

In modern probability theory, sigma-algebras facilitate the construction of probability measures, especially in countably or uncountably infinite spaces (as the concept is trivial in finite spaces). They satisfy certain axioms, including countable additivity, that link the set algebra of events in the space to the algebra of their probabilities, particularly through continuity properties.

\subsection{Proof of Lemma \ref{l:monotonicity}}\label{a:proof_monotonicity_sigma}

\begin{proof}
Assume that
\[
    \sigma(X_1)\subset \sigma(X_2).
\]
Then \(X_1\) is a measurable post-processing of \(X_2\). Equivalently, under
the standard Borel assumptions, there exists a measurable map \(g\) such that
\[
    X_1=g(X_2)
    \qquad \mathbb P\text{-a.s.}
\]

For mutual information, this implies the Markov chain $Y \rightarrow X_2 \rightarrow X_1$. Hence, by the data-processing inequality,
\[
    I(Y;X_1)\le I(Y;X_2),
\]
for arbitrary measurable \(Y\), whenever the mutual informations are
well-defined, possibly taking the value \(+\infty\).

For the conditional entropy statement, the discrete finite-entropy case follows
from
\[
    I(Y;X_s)=H(Y)-H(Y\mid X_s),
    \qquad s=1,2,
\]
together with the mutual-information monotonicity above. In the continuous case, it follows from assumption that the conditional laws
\(\mathcal L(Y\mid X_s)\) admit densities \(p_{Y\mid X_s}\), and the
conditional differential entropies are finite. Since
\(\sigma(X_1)\subset\sigma(X_2)\), we have
\[
    I(Y;X_2\mid X_1)\ge 0.
\]
By the differential-entropy chain rule, we have
\[
    I(Y;X_2\mid X_1)
    =
    h(Y\mid X_1)-h(Y\mid X_2).
\]
Therefore,
\[
    h(Y\mid X_2)\le h(Y\mid X_1).
\]

Next, fix a measurable target \(Y:\Omega\to\mathcal Y\) and an event
\(A\in\mathcal B_{\mathcal Y}\). Let $U:=\mathbbm 1_{\{Y\in A\}}$. Since \(\sigma(X_1)\subset\sigma(X_2)\), the tower property gives
\[
    \mathbb E(U\mid X_1)
    =
    \mathbb E\!\left(
        \mathbb E(U\mid X_2)
        \,\middle|\, X_1
    \right).
\]
Using the \(L^2\)-contraction property of conditional expectation, we obtain
\[
    \big\|\mathbb E(U\mid X_1)\big\|_{L^2}^2
    \le
    \big\|\mathbb E(U\mid X_2)\big\|_{L^2}^2.
\]
Since
\[
    \mathbb E(U\mid X_s)
    =
    \mathbb P(Y\in A\mid X_s),
    \qquad s=1,2,
\]
it follows that
\[
    \big\|\mathbb P(Y\in A\mid X_1)\big\|_{L^2}^2
    \le
    \big\|\mathbb P(Y\in A\mid X_2)\big\|_{L^2}^2.
\]

Finally, let \(Y\in L^2(\Omega;\mathbb{R}^d)\). Again, by the tower property,
\[
    \mathbb E(Y\mid X_1)
    =
    \mathbb E\!\left(
        \mathbb E(Y\mid X_2)
        \,\middle|\, X_1
    \right).
\]
Applying the \(L^2\)-contraction property of conditional expectation gives
\[
    \big\|\mathbb E(Y\mid X_1)\big\|_{L^2}^2
    \le
    \big\|\mathbb E(Y\mid X_2)\big\|_{L^2}^2.
\]
That completes the proof.
\end{proof}

\subsection{Intuitive insights into Theorem \ref{th:optimal_solution}}

In feature unlearning, the released representation \(\hat X\) is required to
satisfy \(\hat X\perp Z\), so that \(\hat X\) itself does not expose the
sensitive or unlearned feature \(Z\). At the same time, the predictive
relationship between the downstream target \(Y\) and the representation may be
different across \(Z\)-slices. Therefore, measuring utility only through
\(I(Y;\hat X)\) can underestimate the task information preserved within each
slice. We instead evaluate utility through the lifted quantity
\(I(Y;\hat X,Z)\). Since
\[
    I(Y;\hat X,Z)=I(Y;Z)+I(Y;\hat X\mid Z),
\]
and \(I(Y;Z)\) is fixed by the data distribution, maximizing
\(I(Y;\hat X,Z)\) is equivalent to maximizing the within-slice task information
\(I(Y;\hat X\mid Z)\). Thus, the lifted formulation separates the two roles of
\(Z\): the released representation \(\hat X\) must not reveal \(Z\), while
utility is evaluated slice-wise through \((\hat X,Z)\), as discussed in Section~\ref{s2:feature_unlearning}.

Under the assumptions of Lemma~\ref{l:Finest_Sigma_Algebra}, the barycentric
representation \(\bar X=T_Z(X)\) satisfies
\[
    \sigma(\hat X,Z)\subset \sigma(\bar X,Z)
\]
for every admissible representation \(\hat X=f(X,Z)\). Thus, among all
admissible feature-unlearning representations, \((\bar X,Z)\) preserves the
finest lifted information available from \((X,Z)\), while \(\bar X\) itself is independent of \(Z\). This lifted maximality has several consequences.

\begin{itemize}[leftmargin=*]
    \item \emph{Utility is evaluated with the task context retained.}
    The constraint \(\hat X\perp Z\) ensures that the released representation
    does not reveal \(Z\). At the same time, the lifted pair \((\hat X,Z)\)
    retains the task context needed to evaluate utility. The theorem therefore
    identifies \(\bar X\) as a representation whose release removes marginal
    information about \(Z\), while whose lifted version \((\bar X,Z)\) preserves
    maximal task-relevant information.
    \item \emph{Conditional uncertainty about the target is minimized.}
    For finite-valued targets \(Y\), the sigma-algebra inclusion implies
    \[
        H(Y\mid \bar X,Z)\le H(Y\mid \hat X,Z)
    \]
    for every admissible \(\hat X=f(X,Z)\). Thus, after the slice variable \(Z\)
    is included for utility evaluation, \(\bar X\) leaves the least residual
    uncertainty about \(Y\).
    \item \emph{Lifted mutual information is maximized.}
    Since \((\bar X,Z)\) generates the finest lifted sigma-algebra among
    admissible representations, data processing gives
    \[
        I(Y;\hat X,Z)\le I(Y;\bar X,Z).
    \]
    Hence, \(\bar X\) preserves the largest amount of information about any
    downstream target \(Y\) when utility is evaluated in the lifted sense.
    \item \emph{Posterior concentration is maximized.}
    The expected posterior divergence
    \[
        \mathbb E\!\left[
        D_{\mathrm{KL}}\!\left(
        \mathbb P(Y\mid \hat X,Z)\,\Vert\,\mathbb P(Y)
        \right)
        \right]
    \]
    equals \(I(Y;\hat X,Z)\) whenever the quantities are finite. Therefore,
    the same lifted maximality means that the posterior law of \(Y\) given
    \((\bar X,Z)\) is, on average, the most concentrated relative to the prior
    law of \(Y\).
    \item \emph{Conditional probability and conditional mean energies are
    maximized.}
    For any event \(A\in\mathcal B_{\mathcal Y}\),
    \[
        \big\|\mathbb P(Y\in A\mid \hat X,Z)\big\|_{L^2}^2
        \le
        \big\|\mathbb P(Y\in A\mid \bar X,Z)\big\|_{L^2}^2.
    \]
    Thus events generated by \(Y\) are predicted at least as sharply from
    \((\bar X,Z)\) as from any other admissible lifted representation.
    Similarly, for \(Y\in L^2\),
    \[
        \big\|\mathbb E(Y\mid \hat X,Z)\big\|_{L^2}^2
        \le
        \big\|\mathbb E(Y\mid \bar X,Z)\big\|_{L^2}^2.
    \]
    Hence, \((\bar X,Z)\) also maximizes explained conditional-mean energy.
\end{itemize}

In summary, the theorem separates the two roles of \(Z\): the released
representation \(\hat X\) must remove marginal information about \(Z\), while
the lifted pair \((\hat X,Z)\) is used to measure how much task-relevant
information is preserved on each of the $Z$-slices. The barycentric representation \(\bar X\) is optimal
because it satisfies both requirements: \(\bar X\perp Z\), and
\((\bar X,Z)\) is maximal in the lifted sigma-algebra order.